\documentclass[final,hidelinks,onefignum,onetabnum]{siamart250211}
\usepackage{lipsum}
\usepackage{amsfonts}
\usepackage{graphicx}
\usepackage{epstopdf}
\usepackage{algorithmic}
\usepackage{amsopn}
\usepackage{enumerate}
\usepackage{dsfont}

\usepackage{amsmath}
\usepackage{amssymb}

\ifpdf
\hypersetup{
  pdftitle={Manifold learning in Wasserstein space},
  pdfauthor={K. Hamm, C. Moosm\"uller, B. Schmitzer, and M. Thorpe}
}
\fi

\ifpdf
  \DeclareGraphicsExtensions{.eps,.pdf,.png,.jpg}
\else
  \DeclareGraphicsExtensions{.eps}
\fi

\newcounter{PauseListCounter}
\newcommand{\mfold}{\mc{S}}
\newcommand{\emb}{E}

\newcommand{\LL}{\mathrm{L}}
\newcommand{\HH}{\mathrm{H}}
\newcommand{\CC}{\mathrm{C}}
\newcommand{\len}{\mathrm{len}}
\newcommand{\energy}{\mathrm{energy}}
\newcommand{\dd}{\mathrm{d}}

\newcommand{\range}{\mathrm{range}}

\newcommand{\cB}{\mathcal{B}}

\newcommand{\eps}{\varepsilon}

\newcommand{\mfoldprob}{\bbP_{\mfold}}
\newcommand{\mfoldprobN}{\bbP_{\mfold^N}}
\newcommand{\mfoldprobhatN}{\widetilde{\bbP}_{\mfold,N}}
\newcommand{\diffeo}{\Phi}
\def\one{\mathds{1}}
\def\lp{\left(}
\def\rp{\right)}
\def\lb{\left\{}

\def\rd{\right.}

\DeclareMathOperator*{\weakstarto}{\stackrel{*}{\rightharpoonup}}

\def\fro{\mathrm{F}}
\def\Lp#1{\mathrm{L}^{#1}}

\def\Ck#1{\mathrm{C}^{#1}}

\def\dis{\mathrm{dis}}
\def\W{\mathrm{W}}
\def\GH{\mathrm{GH}}
\def\GW{\mathrm{GW}}
\def\HS{\mathrm{HS}}
\def\fro{\mathrm{F}}
\def\bbN{\mathbb{N}}
\def\bbR{\mathbb{R}}
\def\bbP{\mathbb{P}}

\def\cP{\mathcal{P}}
\def\cR{\mathcal{R}}
\def\cX{\mathcal{X}}
\def\cY{\mathcal{Y}}
\def\Leb{\mathrm{ac}}
\def\probL{\prob_{\Leb}}
\def\domexp{\mathcal{D}}
\newcommand{\doubletilde}{\widetilde{\widetilde{\lambda}}{\vphantom{\widetilde{\lambda}}}_{\gamma_2(t)}}

\newcommand{\tn}[1]{\textnormal{#1}}
\newcommand{\mc}[1]{\mathcal{#1}}

\newcommand{\N}{\mathbb{N}}
\newcommand{\R}{\mathbb{R}}

\newcommand{\veps}{\varepsilon}

\newcommand{\RadNik}[2]{\tfrac{\diff #1}{\diff #2}}

\newcommand{\ol}[1]{\overline{#1}}

\newcommand{\diff}{\tn{d}}

\DeclareMathOperator{\id}{id}

\DeclareMathOperator{\TV}{TV}

\DeclareMathOperator{\Lip}{Lip}
\DeclareMathOperator{\spt}{spt}
\DeclareMathOperator{\diam}{diam}

\DeclareMathOperator{\vecspan}{span}

\newcommand{\CE}{\tn{CE}}
\newcommand{\assign}{:=}
\newcommand{\assignRe}{=:}
\newcommand{\restr}{{\mbox{\LARGE$\llcorner$}}}
\newcommand{\measp}{\mc{M}_+}
\newcommand{\meas}{\mc{M}}
\newcommand{\prob}{\mc{P}}

\newsiamremark{remark}{Remark}
\newsiamremark{hypothesis}{Hypothesis}
\crefname{hypothesis}{Hypothesis}{Hypotheses}
\newsiamthm{claim}{Claim}
\newsiamremark{fact}{Fact}
\crefname{fact}{Fact}{Facts}
\newsiamremark{assumption}{Assumption}
\crefname{assumption}{Assumption}{Assumptions}
\newsiamremark{example}{Example}
\crefname{example}{Example}{Examples}

\headers{Manifold learning in Wasserstein space}{K. Hamm, C. Moosm\"uller, B. Schmitzer, and M. Thorpe}

\title{Manifold learning in Wasserstein space\thanks{Submitted to the editors 16/11/2023.
\funding{This work was funded by the Army Research Office under contract no.~W911NF-23-1-0213, NSF awards DMS-2306064 and DMS-2410140, DFG projects~403056140 and~A03, H2020 grant 777826, Leverhulme Trust~RPG-2024-051 and EPSRC grant~EP/Y028783/1.}}}

\author{Keaton Hamm\thanks{Department of Mathematics, University of Texas at Arlington, Arlington, TX, USA.\\ Division of Data Science, The University of Texas at Arlington, Arlington, TX, USA.}
\and Caroline Moosm\"uller\thanks{Department of Mathematics, University of North Carolina at Chapel Hill, Chapel Hill, NC, USA.}
\and Bernhard Schmitzer\thanks{Faculty of Mathematics and Computer Science, G\"ottingen University, G\"ottingen, Germany}
\and Matthew Thorpe\thanks{Department of Statistics, University of Warwick, Coventry, UK (\email{matthew.thorpe@warwick.ac.uk}).}}

\begin{document}

\maketitle

\begin{abstract}
This paper aims at building the theoretical foundations for manifold learning algorithms in the space of absolutely continuous probability measures $\probL(\Omega)$ with $\Omega$ a compact and convex subset of $\mathbb{R}^d$, metrized with the Wasserstein-2 distance $\W$. We begin by introducing a construction of submanifolds $\Lambda$ in $\probL(\Omega)$ equipped with metric $\W_\Lambda$, the geodesic restriction of $\W$ to $\Lambda$. In contrast to other constructions, these submanifolds are not necessarily flat, but still allow for local linearizations in a similar fashion to Riemannian submanifolds of $\mathbb{R}^d$. We then show how the latent manifold structure of $(\Lambda,\W_{\Lambda})$ can be learned from samples $\{\lambda_i\}_{i=1}^N$ of $\Lambda$ and pairwise extrinsic Wasserstein distances $\W$ on $\probL(\Omega)$ only. In particular, we show that the metric space $(\Lambda,\W_{\Lambda})$ can be asymptotically recovered in the sense of Gromov--Wasserstein from a graph with nodes $\{\lambda_i\}_{i=1}^N$ and edge weights $W(\lambda_i,\lambda_j)$. In addition, we demonstrate how the tangent space at a sample $\lambda$ can be asymptotically recovered via spectral analysis of a suitable ``covariance operator'' using optimal transport maps from $\lambda$ to sufficiently close and diverse samples $\{\lambda_i\}_{i=1}^N$. The paper closes with some explicit constructions of submanifolds $\Lambda$ and numerical examples on the recovery of tangent spaces through spectral analysis.
\end{abstract}



\section{Introduction}

\subsection{Motivation}

\paragraph{The Riemannian structure of the Wasserstein-2 distance}
The Wasserstein-2 distance, induced by optimal transport, is a geometrically intuitive, robust distance between probability measures with applications in the analysis of PDEs, stochastics, and subsequently in data analysis and machine learning. We refer to \cite{santambrogio2015optimal,villani2009optimal} for recent monographs on the subject and to \cite{PeyreCuturiCompOT} for an overview on computational aspects.

Let $\Omega \subset \R^d$ be convex and compact, and denote by $\prob(\Omega)$ the probability measures on $\Omega$. Then the Wasserstein-2 distance between $\mu, \nu \in \prob(\Omega)$ is given by
\begin{equation}
\label{eq:IntroW}
\W(\mu,\nu) \assign \inf \left\{ \int_{\Omega \times \Omega} |x-y|^2 \dd \pi(x,y) \middle| \pi \in \Pi(\mu,\nu) \right\}^{1/2}
\end{equation}
where $\Pi(\mu,\nu)$ denotes the set of \emph{transport plans} between $\mu$ and $\nu$, given by the probability measures on $\Omega \times \Omega$ with $\mu$ and $\nu$ as first and second marginals respectively.

Let $\probL(\Omega)$ be the set of probability measures that are absolutely continuous with respect to the Lebesgue measure. By Brenier's polar factorization theorem \cite{MonotoneRerrangement-91}, if $\mu \in \probL(\Omega)$, then the minimizing $\pi$ in \eqref{eq:IntroW} is unique and of the form $\pi=(\id,T)_\# \mu$ for a map $T$ that is ($\mu$-almost everywhere) the gradient of a convex potential $\phi : \Omega \to \R$. Here $(\id,T)$ denotes the map $\Omega \ni x \mapsto (x,T(x))$ and $\#$ denotes the push-forward of a measure under a map. One usually refers to $T$ as the optimal transport map and in particular one has $T_\# \mu=\nu$.
Conversely, whenever $\phi : \Omega \to \R$ is convex, then $T \assign \nabla \phi$ is the optimal transport map between $\mu$ and $T_\# \mu$.

The space $(\prob(\Omega),\W)$ has a rich geometric structure. For example, $\W$ metrizes the weak* topology on $\prob(\Omega)$ in duality with continuous functions (on non-compact domains one needs some additional control on the measures' moments). There are geodesics, known as \emph{displacement interpolation} \cite{McCannConvexity1997}. For instance, for $\mu \in \probL(\Omega)$, $\nu \in \prob(\Omega)$ and $T$ being the optimal transport map, the shortest path from $\mu$ to $\nu$ is given by
\begin{equation}
\label{eq:IntroGeodesic}
[0,1] \ni t \mapsto (x \mapsto (1-t) \cdot x + t \cdot T(x))_\# \mu.
\end{equation}
That is, along the geodesic from $\mu$ to $\nu$, a mass particle initially located at $x$ travels with constant speed along the straight line from $x$ to $T(x)$.
Geodesics can also be characterized via the celebrated Benamou--Brenier formula through a least action principle \cite{BenamouBrenier2000}. In addition, there is a well-defined notion of barycenters \cite{WassersteinBarycenter}, and a notion of gradient flow induced by minimizing movements \cite{JKO1998,AGSGradientFlows2008}.

At an intuitive level, $(\probL(\Omega),\W)$ can be formally interpreted as an infinite-dimensional Riemannian manifold \cite{OttoRiemannianOT2001}. Tangent vectors at $\mu \in \probL(\Omega)$ can be thought of as being represented by (Eulerian) velocity fields in $\LL^2(\mu;\R^d)$ that indicate the infinitesimal direction of movement of mass particles. For instance, in the trajectory \eqref{eq:IntroGeodesic} each particle of $\mu$ travels with constant speed along a straight line from its initial position at $x$ to $T(x)$, and so at $t=0$ the velocity field of the particles is given by $v(x)=T(x)-x$.
This intuition comes with a corresponding Riemannian notion of a logarithmic map
$$\log_\mu : \prob(\Omega) \to \Lp{2}(\mu;\R^d)$$
which takes a measure $\nu$ to the tangent vector $v \in \Lp{2}(\mu;\R^d)$ that is the direction of the geodesic from $\mu$ to $\nu$ at $t=0$. In the above example one has $\log_\mu(\nu)=T-\id$.
The corresponding left-inverse
$$\exp_\mu : \Lp{2}(\mu;\R^d) \to \prob(\R^d), \qquad v \mapsto (\id+v)_\# \mu$$
is interpreted as the exponential map.
More details on this are given, for instance, in \cite{Ambrosio2013}.

Infinite dimensional Riemannian geometry is an active field of research \cite{MichorGlobalAnalysis}, driven, for instance, by applications in shape analysis \cite{Metamorphoses2005,YounesShape2010}.
However, for $\W$, the interpretation as a Riemannian manifold is purely intuitive and formal.
With sufficient regularity assumptions, one can define notions of the Levi-Civita connection, Christoffel symbols and parallel transport \cite{LottWassersteinRiemannian2008}. 
But the submanifold of smooth measures considered in \cite{LottWassersteinRiemannian2008} is not complete with respect to $\W$ and its completion is again the whole Wasserstein space, including singular measures, where more complex notions of tangent cones become necessary \cite{Ambrosio2013,LottWassersteinCones2016,Altekrueger2023neural}. We avoid this issue by considering finite-dimensional submanifolds of $\probL(\Omega)$ that are compact with respect to $\W$ and its geodesic restriction (\Cref{cor:WLambdaCompact}).

In \cite{Carlen2003DeepestDescent}, a submanifold $\mathcal{E}_{u,\theta}$ of $\mathcal{P}(\Omega)$ is defined by considering all densities with fixed mean $u$ and variance $\theta$. This submanifold arises in studying nonlinear kinetic Fokker-Planck equations, which is further discussed in the authors’ follow-up paper \cite{Carlen2004Boltzmann}.
The geometry of $\mathcal{E}_{u,\theta}$ is described in \cite{Carlen2003DeepestDescent}, including closed expressions for the metric induced by the 2-Wasserstein metric and a description of geodesics. This is different to our approach in the following ways: (1) our construction of submanifolds $\Lambda$ is broader in the sense that it does not depend on only two parameters; (2)  $\mathcal{E}_{u,\theta}$ is infinite dimensional, while our paper considers finite-dimensional constructions only.
A related topic is discussed in \cite{Gangbo2000ShapeRecognition}, where optimal transport problems are considered between two measures supported on boundaries of domains in $\mathbb{R}^d$. This is motivated by shape recognition problems. While \cite{Gangbo2000ShapeRecognition} provides conditions under which optimal transport solutions exists between such boundary curves, we are concerned with defining and relating the intrinsic geometry of subsets $\Lambda$ to that of $(\mathcal{P}(\Omega),W)$, see \Cref{subsec:outline}.

\paragraph{Linearized optimal transport}
Nevertheless, the intuition inspires powerful applications, such as the linearized optimal transport framework \cite{OptimalTransportTangent2012}, which can formally be interpreted as the local approximation of $(\prob(\Omega),\W)$ by embedding it via the logarithmic map to the tangent space $\Lp{2}(\mu;\R^d)$ at a reference measure $\mu \in \probL(\Omega)$.
For two measures $\lambda_1, \lambda_2 \in \prob(\Omega)$ one approximates
\begin{equation}
\label{eq:IntroApprox}
\W(\lambda_1,\lambda_2) \approx \|\log_\mu(\lambda_1)-\log_\mu(\lambda_2)\|_{\Lp{2}(\mu; \R^d)}.
\end{equation}
This is useful from a numerical perspective, since evaluating all pairwise approximate distances between $N$ samples via \eqref{eq:IntroApprox} merely requires solving $N$ optimal transport problems (to compute the logarithmic maps for all samples), instead of $O(N^2)$ problems for computing all pairwise exact Wasserstein distances.
It is furthermore useful for subsequent data analysis applications since many more methods for tasks such as clustering, classification, and regression are available for the linear space $\Lp{2}(\mu;\R^d)$ than for the general non-linear metric space $(\prob(\Omega),\W)$.
By now, a vast list of applications and extensions of this framework has appeared in the literature \cite{OptimalTransportTangent2012,basu2014detecting,kolouri2016continuous,park2018representing,crook2020linear,LinHK2021,bai2023linear,martin2023lcot,Sarrazin23}.

\paragraph{Tangent space approximation quality}
In finite-dimensional Riemannian geometry there are bounds of the type
\begin{align}
\label{eq:IntroRiemannBound}
\Big| d(\exp_x(t \cdot v),\exp_x(t \cdot w)) - t \cdot \|v-w\|_x \Big| \leq C \cdot t^2
\end{align}
where $x$ is a point on a smooth Riemannian  manifold $\mfold$, $\exp_x$ denotes the exponential map at $x$, and $v,w \in T_x \mfold$ are two tangent vectors (of bounded norm), $\|\cdot\|_x$ denotes the norm on the tangent space $T_x \mfold$, and $C$ is a constant depending on the curvature of the manifold (and  the norms of $v$ and $w$). This means that, locally around $x$, the exact distance and the linearization agree to leading order.

For $\W$, it is easy to see that the linearization \eqref{eq:IntroApprox} does indeed provide an upper bound, i.e.~$\W(\lambda_1,\lambda_2) \leq \|\log_\mu(\lambda_1)-\log_\mu(\lambda_2)\|_{\Lp{2}(\mu;\R^d)}$.
However, a control in the other direction is much more delicate.
For some specific subsets $\Lambda \subset \prob(\Omega)$ bounds in the spirit of \eqref{eq:IntroRiemannBound} are known. For instance, when $\Lambda$ is generated by translations or scalings of a template $\mu \in \probL(\Omega)$, then the local linearization is exact, i.e.~equality holds in \eqref{eq:IntroApprox} and $\Lambda$ is a called a flat subset of $\prob(\Omega)$.
The linearization of flat and approximately flat subsets of $(\prob(\Omega),\W)$ is studied in more detail in \cite{MooClo22,MoosmuellerLinOTSheared2022}. In \cite{cloninger2023linearized,Hamm2023Wassmap} an isometric embedding of $\Lambda$ into a Euclidean vector space is constructed, assuming it exists.

For other, more general subsets of $\probL(\Omega)$, the approximation error in \eqref{eq:IntroApprox} has been shown to be much worse than \eqref{eq:IntroRiemannBound}. A result on the H\"older continuity of $\log_\mu$ (between $(\prob(\Omega),\W)$ and $\Lp{2}(\mu;\R^d)$) was reported in \cite{gigli11}.
A refinement is given in \cite{delalande21} assuming that $\mu$ has convex, compact support and a density that is bounded away from zero and infinity.

\paragraph{Manifold Learning}
A fundamental observation about high-dimensional data is that, in many applications, data lies on or near a low-dimensional manifold embedded into a high-dimensional space. If this structure is unknown, but suspected, then the assumption that data lies on a manifold is termed the \textit{manifold hypothesis} \cite{FeffermanManifoldHypothesis}. This is the starting point of \textit{manifold learning}, in which the aim is to  construct a putative manifold which encompasses the data \cite{MaggioniLiao}. Alternatively, one may seek an embedding map $\Phi:\R^D\to\R^d$ mapping the data in a high-dimensional ambient space to a low-dimensional feature space ($d\ll D$) such that salient properties and structure (such as geodesic distances) are preserved \cite{tenenbaum2000global,coifman2006diffusion}. This is often called \textit{dimensionality reduction}. Most of our analysis pertains to the first aim; however, the results of Section \ref{sec:GW} are related to dimensionality reduction algorithms that involve approximating a data manifold via a graph over discrete samples.

\subsection{Outline and contribution}\label{subsec:outline}

In this article we introduce a class of subsets $\Lambda \subset \probL(\Omega)$ that are not flat, and establish corresponding bounds on the linearization of $\W$ in the spirit of \eqref{eq:IntroRiemannBound}. The motivation is to study an interesting setting between \cite{MooClo22,Hamm2023Wassmap} and \cite{delalande21}.
Our regularity assumptions do not aim to be minimal, but to obtain insightful results with a reasonable technical effort.
We hope that this construction can serve as a useful model to explain why linearized optimal transport works so well in many applications, as well as lay the foundations for manifold learning in Wasserstein space beyond the local tangent approximation, that we plan to pursue in future work.

The definition of $\Lambda$ is given in Section \ref{sec:SubManIntro}.
Concretely, we parametrize $\Lambda$ via a bi-Lipschitz embedding $\emb : \mfold \to \probL(\Omega)$ of some finite-dimensional compact smooth Riemannian manifold $\mfold$, together with a family of deformation velocity fields that indicate how $\lambda=\emb(\theta)$ changes when $\theta \in \mfold$ moves. The precise construction is given in \Cref{asp:Main}.
We then study the geodesic restriction $\W_\Lambda$ of $\W$ to $\Lambda$ and show that geodesics in $(\Lambda,\W_\Lambda)$ can be equivalently found by looking for shortest paths in $\mfold$ with respect to the formal pull-back of the Riemannian tensor of $\W$ (or $\W_\Lambda$).
Consequently, we interpret $(\Lambda,\W_\Lambda)$ as submanifold of $(\probL(\Omega),\W)$.

In Section \ref{sec:LocalLin} we then derive several results related to the embedding of $\Lambda$ into $\probL(\Omega)$ and on local linearization in the spirit of \eqref{eq:IntroRiemannBound}.
For instance, we show in \Cref{prop:WLambdaWComparison} that there is a global constant $C$ such that
$$0 \leq \frac{\W_\Lambda(\lambda_1,\lambda_2)-\W(\lambda_1,\lambda_2)}{\W(\lambda_1,\lambda_2)^2} \leq C$$
for all $\lambda_1,\lambda_2 \in \Lambda$.
Clearly one has $\W(\lambda_1,\lambda_2) \leq \W_{\Lambda}(\lambda_1,\lambda_2)$ since the shortest path between $\lambda_1,\lambda_2$ with respect to $\W$ will typically not remain in $\Lambda$ but move along a direct chord. The above estimate then intuitively implies that the chord remains close to the submanifold.

Moreover, let $\gamma_i : [0,1] \to \mfold$, $i\in \{1,2\}$, be two constant speed geodesics on $\mfold$, starting at a common point $\theta=\gamma_1(0)=\gamma_2(0)$, let $\lambda_{\gamma_i(t)}=\emb(\gamma_i(t))$ be their embeddings into $\probL(\Omega)$ via $\emb$ with $\lambda_\theta=\emb(\theta)$, and let $v_i$ be the deformation velocity fields associated with this embedding at $t=0$. Then \Cref{thm:LocalLinearization} states that there is a constant $C$ (not depending on $\gamma_i$) such that
\begin{equation*}
\Big|\W_\Lambda(\lambda_{\gamma_1(t)},\lambda_{\gamma_2(t)}) - t \cdot \|v_1-v_2\|_{\Lp{2}(\lambda;\R^d)}\Big| \leq t^2 \cdot C \cdot (\|\dot{\gamma}_1(0)\|_{\theta}+\|\dot{\gamma}_2(0)\|_{\theta})^2
\end{equation*}
when $0 \leq t \leq \min\{1,\tfrac{1}{C \|\dot{\gamma}_1(0)\|_{\theta}},\tfrac{1}{C \|\dot{\gamma}_2(0)\|_{\theta}}\}$.
The proof does not rely on the study of the regularity properties of the pull-back of the Riemannian tensor of $(\probL(\Omega),\W)$ to $\mfold$ (i.e.~an intrinsic approach), but on the regularity of the embedding $E$ and the embedding of $\Lambda$ into $\probL(\Omega)$ (i.e.~via extrinsic arguments).

This result may be somewhat abstract, since in applications the latent parametrization manifold $\mfold$ and deformation velocity fields are probably not known explicitly. We therefore give approximation results that do not require such latent knowledge but merely access to curves $\lambda_{\gamma_i}$ in $\Lambda$ and evaluation of pairwise distances $\W$ instead of $\W_\Lambda$. For instance, the deformation velocity field $v_1$ can be recovered by \Cref{prop:wvconvergence} as
$$v_1 = \lim_{t \to 0} (T_{1,t}-\id)/t$$
where $T_{i,t}$ is the optimal transport map from $\lambda_\theta$ to $\lambda_{\gamma_i(t)}$. The above results can then be combined to show
$$
\Big|\W(\lambda_{\gamma_1(t)},\lambda_{\gamma_2(t)}) - \|T_{1,t}-T_{2,t}\|_{\Lp{2}(\lambda)}\Big| \leq t^{3/2} \cdot C \cdot (\|\dot{\gamma}_1(0)\|_{\theta}+\|\dot{\gamma}_2(0)\|_{\theta})
$$
for sufficiently small $t$. This means, on $\Lambda$ the linearized optimal transport approximation \eqref{eq:IntroApprox} is correct to leading order, and thus works much better than in the more general cases studied in \cite{delalande21} where only H\"older regularity of the logarithmic holds.
Finally, these results can be adapted from curves $\lambda_{\gamma_i}$ that are induced by geodesics on $\mfold$ to general differentiable curves.

Sections \ref{sec:GW} and \ref{sec:TangentSpace} further explore the challenge of recovering $(\Lambda,\W_\Lambda)$ and its manifold structure while only having access to samples from $\Lambda$ and pairwise distances $\W$.
In Section \ref{sec:GW}, we show that the metric space $(\Lambda,\W_\Lambda)$ can be recovered asymptotically in the sense of Gromov--Wasserstein from samples of $\Lambda$ equipped with a local metric graph with edge lengths given by $\W$.
The main result of Section \ref{sec:TangentSpace} is that the tangent space of deformation fields $v$ at a measure $\lambda \in \Lambda$ can asymptotically be recovered from (sufficiently diverse) samples $\{\lambda_i\}_i$ from a small ball around $\lambda$ and their pairwise optimal transport maps via spectral analysis of a suitable auxiliary matrix. These results are similar in spirit to data analysis results involving local principal component analysis (PCA) or local singular value decomposition (SVD) in which one attempts to locally recover tangent spaces to a data manifold or to locally linearly approximate data by a subspace that minimizes data variance. See, for example \cite{little2017multiscale,MaggioniLiao}.

Together these statements imply that, in principle, if one has merely access to $\Lambda$ and $\W$, then the latent parametrization structure of Assumption \ref{asp:Main} can be recovered. That is, in this sense the submanifold $\Lambda$ exists in its own right, independently from the parametrization.

For concreteness, some examples for such manifolds are discussed in Sections \ref{sec:ExampleGeneral} and \ref{sec:ExampleAdditional}, and numerical illustrations are given in Section \ref{sec:Numerics}.

\subsection{Notation and setting}

\begin{itemize}
\item Throughout this article, let $\Omega$ be a compact, convex subset of $\R^d$.
\item For a metric space $X$, denote by $\CC(X)$ the space of continuous real-valued functions on $X$ equipped with the supremum norm. Other co-domains are denoted by $\CC(X;Y)$. When applicable, $\CC^k(X;Y)$ denotes the space of $k$-times continuously differentiable functions, equipped with the supremum norm on the function and its derivatives up to order $k$. In all cases, the norm will be denoted by $\|\cdot\|_\infty$.
The set of Lipschitz-continuous functions from $X$ to $Y$ is denoted by $\Lip(X;Y)$.
\item On a compact metric space $X$, denote by $\meas(X)$ the space of signed Radon measures on $X$, identified with the topological dual space of $\CC(X)$ and we mainly use the induced weak* topology on $\meas(X)$. Denote by $\measp(X)$ and $\prob(X)$ the subsets of non-negative and probability measures.
For $\mu \in \measp(X)$, denote by $\Lp{2}(X;Y)$ the space of functions from $X$ to $Y$ that are square-integrable  with respect to $\mu$. If clear from context, we may omit~$Y$.
\item For a metric space $(X,d)$, $x \in X$, $\delta>0$, denote by $B_d(x,\delta)\assign\{ y\in X : d(x,y) < \delta\}$ the open ball of radius $\delta$ around $x$.
\item For a measure $\mu\in\meas(X)$ and a $\mu$-measurable function $F:X\to Y$ we define the pushforward of $\mu$ by $F$ by $F_{\#}\mu(B)\assign\mu(F^{-1}(B))$ for all measurable $B\subset Y$, and where $F^{-1}(B) = \{x\in X\,:\, F(x)\in B\}$.
\item Given two probability measures $\mu\in\prob(X)$, $\nu\in\prob(Y)$ we denote the set of couplings by $\Pi(\mu,\nu)$. That is $\pi\in\Pi(\mu,\nu)$ if $\pi\in\prob(X\times Y)$ and $P_{X\#}\pi = \mu$ and $P_{Y\#}\pi = \nu$ where $P_X:X\times Y\ni(x,y) \mapsto x\in X$ and $P_Y:X\times Y\ni(x,y) \mapsto y\in Y$.
\end{itemize}

Finally, we use the celebrated Benamou--Brenier formula \cite{BenamouBrenier2000} to express $\W(\mu,\nu)$ as a dynamic optimization problem over curves in $\prob(\Omega)$ connecting $\mu$ and $\nu$. In the following we state it in a slightly unusual regularity setting which will be natural in the context of our article. A proof for this form is given in the appendix.
\begin{proposition}[Energy functional and Benamou--Brenier formulation]
\label{prop:EnergyBasic}
For a curve $\rho \in \Lip([0,1];(\prob(\Omega),\W))$ let
\begin{align*}
	\CE(\rho) & \assign \Bigg\{ v : [0,1] \to \LL^2(\rho(t);\R^d) \tn{ measurable } \Bigg|
		\int_0^1 \int_\Omega \partial_t \phi(t,\cdot) \, \diff \rho(t) \, \diff t
		\\
		& + \int_0^1 \int_\Omega \langle \nabla \phi(t,\cdot) , v(t)\rangle \, \diff \rho(t) \, \diff t
		= \int_\Omega \phi(1,\cdot) \, \diff \rho(1) - \int_\Omega \phi(0,\cdot) \, \diff \rho(0) \\
		& \tn{ for all } \phi \in \Ck{1}([0,1] \times \Omega)
		\Bigg\}
\end{align*}
be the corresponding set of velocity fields that satisfy the continuity equation with $\rho$ in a distributional sense and set
\begin{align}
\label{eq:Energy}
\energy(\rho) & \assign \inf\left\{ \int_0^1 \| v(t)\|^2_{\LL^2(\rho(t))} \, \dd t \ \middle| \ v \in \CE(\rho) \right\}.
\end{align}
The map $\Lip([0,1];(\prob(\Omega),\W)) \ni \rho \mapsto \energy(\rho)$ is lower-semicontinuous (with respect to uniform convergence of the curves $\rho$) and one has
\begin{equation} \label{eq:BB}
\W(\mu,\nu)=\inf \left \{\energy(\rho)^{1/2} \ \middle|\ \rho \in 
	\Lip([0,1];(\prob(\Omega),\W)),\,\rho(0)=\mu,\,\rho(1)=\nu \right\}.    
\end{equation}
Furthermore, minimizers in \eqref{eq:BB} exist.
\end{proposition}

\section{Submanifolds in Wasserstein space} \label{sec:SubMan}

\subsection{Parametrized submanifolds}
\label{sec:SubManIntro}

In this article we investigate submanifolds of $(\probL(\Omega),\W)$ modeled by embedding a latent parameter manifold into $\probL(\Omega)$, together with a family of deformation vector fields that encode how infinitesimal movements on $\mfold$ translate to movements in $\probL(\Omega)$. We now give precise construction and regularity assumptions.

\begin{assumption}[Main assumptions for manifold modeling]\hfill
\label{asp:Main}
\begin{enumerate}[(i)]
\item \textbf{Parameter manifold:} \label{item:mfold} Let $\mfold$ be a compact, connected, $m$-dimensional, smooth Riemannian manifold, with tangent bundle $T\mfold$, induced norm $\|\eta\|_\theta$ for $\eta \in T_\theta \mfold$, and induced distance $d_\mfold : \mfold \times \mfold \to \R_+$,
Let $\domexp \subset T\mfold$ be a subset of the tangent bundle on which the Riemannian exponential map is well-defined, with [$(\theta,\eta) \in \domexp$] $\Rightarrow$ [$(\theta,t \cdot \eta) \in \domexp$] for $t \in [0,1]$, and such that for any $(\theta,\eta) \in \domexp$ the curve $\gamma:[0,1] \ni t \mapsto \exp_{\theta}(t \cdot \eta) \in \mfold$ is a constant speed geodesic from $\theta$ to $\exp_\theta(\eta)$ with $\|\dot{\gamma}(t)\|_{\gamma(t)}=\|\eta\|_{\theta}=d_\mfold(\theta,\exp_\theta(\eta))$ for $t \in [0,1]$.
Clearly $\|\eta\|_\theta \leq \diam \mfold$ for any $(\theta,\eta) \in \domexp$.
For $\theta \in \mfold$ write $\domexp_\theta \assign \{ \eta \in T_\theta \mfold | (\theta,\eta) \in \domexp\}$. We assume that for any two points $\theta_0,\theta_1 \in \mfold$ there exists a $\eta \in \domexp_{\theta_0}$ such that $\theta_1=\exp_{\theta_0}(\eta)$.
The parameter manifold $\mfold$ may have a boundary, as long as it is a geodesically convex subset of a larger (not necessarily compact) manifold as specified above, and the construction of geodesics via $\domexp$ remains valid.
\item \textbf{Embedding:} \label{item:embed} Let $\emb : \mfold \to \prob(\Omega)$ be a bi-Lipschitz bijection between $\mfold$ and $\Lambda \assign \emb(\mfold)$, with respect to metrics $d_\mfold$ and $\W$. We will frequently write $\lambda_\theta=\emb(\theta)$ for $\theta \in \mfold$, and $\lambda_{\gamma(t)} = \emb(\gamma(t))$ for a path $\gamma: [0,1] \to \mfold$.
Clearly, $\Lambda$ is a compact subset of $(\prob(\Omega),\W)$.
We assume that all $\lambda \in \Lambda$ are absolutely continuous with respect to the Lebesgue measure, so that the optimal Wasserstein-2 transport between any two of them is given by a map according to Brenier's theorem \cite{MonotoneRerrangement-91}.
\setcounter{PauseListCounter}{\value{enumi}}
\end{enumerate}
Intuitively, we want $\emb$ to be a diffeomorphism from $\mfold$ to $\Lambda$. But since the Riemannian differential structure associated with $(\probL(\Omega),\W)$ is merely formal, we need to introduce a suitable notion of a differential of $\emb$, which takes $T_\theta \mfold$ to $T_{\lambda_\theta} \Lambda$ for $\theta \in \mfold$. The tangent space of $(\probL(\Omega),\W)$ is usually identified with gradient vector fields that describe infinitesimal mass displacement \cite[Section 3.3.2]{Ambrosio2013}. This motivates the following definition.
\begin{enumerate}[(i)]
\setcounter{enumi}{\value{PauseListCounter}}
\item \textbf{Gradient velocity fields:} \label{item:vel} For $\theta \in \mfold$, let $B_\theta : T_\theta \mfold \to \CC^1(\R^d;\R^d)$ be a linear map from the tangent space at $\theta$ to $\CC^1$ gradient vector fields on $\R^d$. We assume there is some $C \in (0,\infty)$ such that
\begin{align}
\label{eq:VelEquivalence}
\|B_\theta \eta\|_{\LL^2(\lambda_\theta)} & \in [1/C,C] \cdot \|\eta\|_\theta \\
\label{eq:VelDerivatives}
\|B_\theta \eta\|_{\CC^1(\R^d)} & \leq C \cdot \|\eta\|_\theta \\
\label{eq:VelTimeSmooth}
|\partial_t [B_{\gamma(t)} \dot{\gamma}(t)](x)| & \leq C \cdot \|\eta\|_{\theta}^2,
\end{align}
for all $(\theta,\eta) \in \domexp$, $x \in \R^d$, and $t \in [0,1]$ where we set $\gamma(t)\assign\exp_\theta(t \cdot \eta)$.
While we only work with measures on a compact domain $\Omega$, note that we assume that the velocities are defined on the whole space $\R^d$ to avoid technical issues at the boundary of~$\Omega$.
\setcounter{PauseListCounter}{\value{enumi}}
\end{enumerate}
We interpret $B_\theta$ as a map $T_\theta \mfold \to T_{\lambda_\theta} \Lambda$ and it will play the role of the differential of $\emb$ at $\theta$.
The vector field $B_\theta \eta$ indicates the velocity of mass particles of the measure $\lambda_\theta$ in $\Omega$ when moving in the direction $\eta$ on the parameter manifold $\mfold$.
The final part of the assumption ensures that the infinitesimal change of $\lambda_\theta$ described by the velocity field $B_\theta \eta$ is indeed consistent with the change of $\emb$ in direction $\eta$.
\begin{enumerate}[(i)]
\setcounter{enumi}{\value{PauseListCounter}}
\item \textbf{Consistency of velocity fields:} \label{item:consistent}
The embedding $\emb$ and the velocity field maps $B_\theta$ satisfy the following consistency condition:
For some $\gamma \in \Lip([0,1];\mfold)$, set $v(t,\cdot) \assign B_{\gamma(t)} \dot{\gamma}(t)$ (well-defined $t$-a.e.), and the flow $\varphi : [0,1] \times \Omega \to \R^d$ as the solution to the initial value problem
\begin{align}
\label{eq:Flow}
\varphi(0,\cdot) & =\id, &
\partial_t \varphi(t,\cdot) & = v(t,\varphi(t,\cdot))
\end{align}
where the second equality need only hold for a.e.~$t$.
See \Cref{lem:Flow} for well-posedness of this ODE.
Then we assume that $\lambda_{\gamma(t)}=\varphi(t,\cdot)_\# \lambda_{\gamma(0)}$ for all $t$.
\end{enumerate}
\end{assumption}
\begin{remark}[On the regularity assumptions]
It is not trivial to verify the regularity assumptions throughout Assumption \ref{asp:Main} for a given triple $(\mfold,\emb,B)$ or to construct non-trivial instances that satisfy them. In particular this concerns the condition that $B$ maps into gradient vector fields with sufficient regularity to satisfy \eqref{eq:VelEquivalence}-\eqref{eq:VelTimeSmooth}. More detailed discussions and examples are given in Section \ref{sec:ExampleGeneral} below, see \eqref{eq:EulerianVel}, and throughout Section \ref{sec:ExampleAdditional}.
\end{remark}

\begin{remark}[On the choice of the parameter manifold $\mfold$]
In this article we consider the geodesic restriction of $\W$ to $\Lambda=\emb(\mfold)$, denoted by $W_\Lambda$ (Section \ref{sec:Restriction}). $\W_\Lambda$ will not depend on the choice of the parametrization manifold $\mfold$ and should be seen as being intrinsic to $\Lambda$ instead.
Indeed, let $\mfold'$ be another Riemannian manifold satisfying Assumption~\ref{asp:Main}\eqref{item:mfold}, with a diffeomorphism $\diffeo : \mfold' \to \mfold$ (which need not be an isometry). Then setting $\emb' \assign \emb \circ \diffeo$, and $B'_\theta \assign B_{\diffeo(\theta)} \circ D_\theta \diffeo$, it is easy to see that the triple $(\mfold',\emb',B')$ also satisfies the rest of Assumption~\ref{asp:Main}, and that it yields the same metric $\W_\Lambda$.
As shown in Proposition~\ref{prop:wvconvergence} (see also Section \ref{sec:TangentSpace}), the differential $B_\theta$ (or $B'_\theta$) can also be reconstructed solely from knowledge about $\Lambda$.
The triple $(\mfold,\emb,B)$ should therefore be seen as an auxiliary scaffolding to parametrize $\Lambda$ and reasonably regular paths within $\Lambda$.
\end{remark}

\begin{remark}[Relation of \eqref{eq:VelTimeSmooth} to geodesic equation for $\W$]
Let
\[ \rho \in \Lip([0,1];(\prob(\Omega),\W)) \]
be a curve of measures with a Eulerian velocity field $v \in \CE(\rho)$.
Assume that $v$ is sufficiently regular such that its flow \eqref{eq:Flow} can be integrated and denote this flow by $\varphi$.
It is well-known that if $\rho$ is a geodesic in $(\prob(\Omega),\W)$, then particles travel with constant speed along straight lines, i.e.~the Lagrangian velocity of $\varphi$ is constant and therefore $\partial^2_t \varphi(t,\cdot)=0$ (at least $\rho(0)$-almost everywhere). Formally, at the level of the Eulerian velocity field $v$ this implies $\partial_t v + \nabla v \cdot v=0$ and when $v$ is written as gradient of some potential $\phi$, i.e.~$v=\nabla \phi$, this can formally be written as $\partial_t \phi + \tfrac12 \|\nabla \phi\|^2=0$ (taking the gradient of the former equation will yield the latter). This equation has formally been identified as a \emph{geodesic equation} on $(\prob(\Omega),\W)$, for instance in \cite{LottWassersteinRiemannian2008}.

Let now $\rho$ be given by $\rho(t)=\emb(\gamma(t))$ for some $\gamma \in \Lip([0,1],\mfold)$ and let $v(t,\cdot)\assign B_\gamma(t) \dot{\gamma}(t)$. Then \eqref{eq:VelDerivatives} and \eqref{eq:VelTimeSmooth} imply that
\begin{equation}
\label{eq:GeodesicDeviation}
|\partial_t v(t,x) + \nabla v(t,x) \cdot v(t,x)| \leq C \cdot \|\dot{\gamma}(t)\|_{\gamma(t)}^2
\end{equation}
for some constant $C$.
This means that curves $\emb(\gamma(\cdot))$ are not necessarily geodesics in $\W$, but their acceleration is bounded, which is a plausible assumption for the modeling of submanifolds. Conversely, assuming \eqref{eq:VelDerivatives} and \eqref{eq:GeodesicDeviation} one obtains that \eqref{eq:VelTimeSmooth} holds for a suitable constant $C$.

As a simple finite-dimensional analogy for \eqref{eq:VelTimeSmooth} let $E : S^1 \to \R^2$ be an embedding of the unit circle into $\R^2$, $\emb(\theta)=(\cos(\theta),\sin(\theta))^\top$, for which one has the differential $B_\theta \eta= (-\sin(\theta),\cos(\theta))^\top \eta$ for $\eta \in T_\theta S^1 \simeq \R$. Let $\gamma$ be a differentiable path in $S^1$ with $\gamma(0)=\theta$, $\dot{\gamma}(0)=\eta$, and $\ddot{\gamma}(0)=0$, and set $\rho \assign \emb \circ \gamma$. Then one finds that $\ddot{\rho}=-\eta^2\cdot \rho$, i.e.~the acceleration of the embedded curve is indeed proportional to $\eta^2$.
\end{remark}

\begin{lemma}[Well-posedness and regularity of flows]\hfill
\label{lem:Flow}
\begin{enumerate}[(i)]
\item For $\gamma \in \Lip([0,1];\mfold)$ there is a unique solution to the flow \eqref{eq:Flow}.
\label{item:FlowBasic}
\item If $\gamma$ is a constant speed geodesic, then $|\partial_t^2 \varphi(t,x)| \leq C \cdot \|\dot{\gamma}(t)\|_{\gamma(t)}^2$ for some $C$ (independent of $x$, $\gamma$, $t$).
\label{item:FlowCurv}
\end{enumerate}
\end{lemma}
\begin{proof}
Consider a fixed $\gamma \in \Lip([0,1];\mfold)$. The velocity field 
\[ v(t,x)\assign [B_{\gamma(t)} \dot{\gamma}(t)](x) \]
is well-define $t$-a.e., (we can extend it to all times by taking, for instance, the left limit at each non-differentiable time).
By~\eqref{eq:VelDerivatives} and~\eqref{eq:VelTimeSmooth}, $v(t,x)$ is continuous in $t$ and Lipschitz in $x$ (uniformly over $t$). By a standard extension of Picard's existence and uniqueness theorem (see for example~\cite[Theorem 10, Section 1.7]{adkins12}) there exists a unique solution to~\eqref{eq:Flow} on the interval $[0,1]$.

When $\gamma$ is a constant speed geodesic, one has that $v(t,x)$ is differentiable in time and space, with uniformly bounded derivatives as implied by \eqref{eq:VelDerivatives} and \eqref{eq:VelTimeSmooth}. Hence, existence and uniqueness of the flow $\varphi$ follows from the Picard--Lindelöf theorem. This $\varphi$ will be a classical solution, and for its second derivative in time we find with Assumptions \eqref{eq:VelDerivatives} and \eqref{eq:VelTimeSmooth} that
\begin{align*}
|\partial_t^2 \varphi(t,x)| & = |\partial_t [v(t,\varphi(t,x))] | \\
 & = |(\partial_t v)(t,\varphi(t,x)) + [\nabla v(t,\varphi(t,x))] v(t,\varphi(t,x)) | \\
 & \leq C\,\|\dot{\gamma}(t)\|^2_{\gamma(t)}. 
\end{align*}
\end{proof}

\subsection{Constructing submanifolds by diffeomorphic deformation of a template}
\label{sec:ExampleGeneral}

We now sketch a way to construct Wasserstein submanifolds in the above sense by deforming a template measure with a family of diffeomorphisms.
This description focuses on the intuition and will not treat regularity in full detail. More details will be included in concrete examples (Example \ref{ex:Translations} and Section \ref{sec:ExampleAdditional}).

Let $(\psi(\theta,\cdot))_{\theta \in \mfold}$ be a family of diffeomorphisms on $\R^d$, parametrized by $\theta \in \mfold$, and let $\lambda \in \probL(\R^d)$ be a template measure (with compact support).
Set $\emb(\theta) \assign \lambda_\theta \assign \psi(\theta,\cdot)_{\#} \lambda$, i.e.~the manifold consists of diffeomorphic deformations of the template.
Given sufficient regularity of the family $\psi$, the union of the supports of all $\lambda_\theta$ will be contained in some compact $\Omega \subset \R^d$.
For a mass particle in the template $\lambda$ at $y \in \R^d$, $x=\psi(\theta,y)$ gives the location of this particle in the deformed measure $\lambda_\theta$.

Now denote by $\nabla_\theta \psi$ the differential of $\psi$ with respect to the parameter $\theta \in \mfold$.
Note that for fixed $y \in \R^d$, $\nabla_\theta \psi(\theta,y)$ is a linear map $T_\theta \mfold \to \R^d$ and $\nabla_\theta \psi(\theta,\cdot)$ can be interpreted as a linear map from $T_\theta \mfold$ to vector fields $\R^d \to \R^d$.
Let $\gamma : [0,1] \to \mfold$ be a differentiable path of parameters and consider the induced path of measures $\lambda_{\gamma(t)}$.
Then $\partial_t \psi(\gamma(t),y) = \nabla_\theta \psi(\gamma(t),y) \dot{\gamma}(t)$ gives the Lagrangian velocity at time $t$ of the particle originally in the template at $y$.
The corresponding Eulerian velocity field is given by
\begin{equation}
\label{eq:EulerianVelConcrete}
\left.\partial_t \psi(\gamma(t),y) \right|_{y=\psi^{-1}(\gamma(t),x)} = \nabla_\theta \psi(\gamma(t),\psi^{-1}(\gamma(t),x)) \dot{\gamma}(t) \assignRe v(t,x)
\end{equation}
where $\psi^{-1}$ denotes the inverse of $\psi$ with respect to the second argument.

It seems therefore natural to consider the map
\begin{equation}
\label{eq:EulerianVel}
T_\theta \mfold \ni \eta \mapsto \nabla_\theta \psi(\theta,\psi^{-1}(\theta,\cdot)) \eta
\end{equation}
as a candidate for $B_\theta$.
For simplicity we will for now assume that this map takes values in \emph{gradient} vector fields as required by Assumption~\ref{asp:Main}\eqref{item:vel}. The case when this is not true is discussed in Section \ref{sec:ExampleNonGradient}.
Assumptions \eqref{eq:VelDerivatives} and \eqref{eq:VelTimeSmooth} will then follow from sufficient regularity of the family $\psi$. The bound in~\eqref{eq:VelDerivatives} then implies the upper bound in \eqref{eq:VelEquivalence}. The lower bound in \eqref{eq:VelEquivalence} states that infinitesimally a linear change in the parameter $\theta$ leads to a linear change in the measure $\lambda_\theta$. It is difficult to give general criteria for this, but it can be verified or tested numerically in concrete cases.
Further, given these assumptions, one then clearly has that $\varphi(t,\cdot) \assign \psi(\gamma(t),\psi^{-1}(\gamma(0),\cdot))$ solves the initial value problem
\begin{align*}
	\varphi(t,\cdot) & = \id, &
	\partial_t \varphi(t,\cdot) & = v(t,\varphi(t,\cdot))
\end{align*}
for $v$ given by \eqref{eq:EulerianVelConcrete}.
Therefore $\varphi(t,\cdot)_{\#} \emb(\gamma(0)) = \psi(\gamma(t),\psi^{-1}(\gamma(0),\cdot))_{\#} \psi(\gamma(0),\cdot)_{\#} \lambda = \emb(\gamma(t))$, i.e.~for this given choice of $B_\theta$, Assumption \ref{asp:Main} \eqref{item:consistent} is satisfied.

Simple concrete examples for the above construction are translations and dilations.
They have been used in some manifold learning contexts previously to illustrate the quality of dimensionality reduction embeddings, e.g., \cite{Hamm2023Wassmap, cloninger2023linearized,negrini2023applications}.
In these examples the map \eqref{eq:EulerianVel} indeed maps to gradient fields.
While they result in a flat submanifold for which the tangent space approximation is exact, they nevertheless help to build some intuition for the role of the map $B_\theta$. More complex examples are given in Section \ref{sec:ExampleAdditional}, in particular the case where \eqref{eq:EulerianVel} does not map to gradient velocity fields is discussed in Section \ref{sec:ExampleNonGradient}.

\begin{example}[Translations]
\label{ex:Translations}
Let $\mfold$ be a compact, convex subset of $\R^d$, which is a smooth, compact Riemannian manifold (with boundary) with $d_\mfold(\theta,\theta') = |\theta-\theta'|$ and we have
$$\domexp = \{(\theta,\eta) \,|\, \theta \in \mfold,\,\eta \in \R^d : \theta + \eta \in \mfold \}.$$
Set $\psi(\theta,x) \assign x+\theta$.
This means that $\psi(\theta,\cdot)$ implements a translation of the template $\lambda$ by $\theta \in \R^d$.
By boundedness of $\mfold$ we find that the union of the supports of all $\lambda_\theta$ is indeed bounded, i.e.~there exists a suitable compact set $\Omega$. Further, one has that $\nabla_\theta \psi(\theta,x) \eta = \eta$, i.e.~\eqref{eq:EulerianVel} corresponds to the choice $B_\theta \eta: x \mapsto \eta$ and one has
$$\|B_\theta \eta\|_{\LL^2(\lambda_\theta)}^2 = \int_\Omega |\eta|^2\,\diff \lambda_\theta=|\eta|^2.$$
As expected, this corresponds to an infinitesimal translation in direction $\eta$. This is a gradient vector field, it satisfies \eqref{eq:VelEquivalence} and \eqref{eq:VelDerivatives}, and by considering curves $\gamma(t)=\exp_\theta(t \cdot \eta)=\theta + t \cdot \eta$ we find that it satisfies \eqref{eq:VelTimeSmooth}.

In this case, since translations are optimal transport maps, we have
\[\W_\Lambda(\lambda_\theta,\lambda_{\theta'})=\W(\lambda_\theta,\lambda_{\theta'})=|\theta-\theta'|=d_\mfold(\theta,\theta'), \quad \theta,\theta'\in\mfold\]
and $|\theta'-\theta''|=\|B_\theta (\theta'-\theta'')\|_{\LL^2(\lambda_\theta)}$ for any $\theta,\theta',\theta'' \in \mfold$.
This means that the tangent space approximation is exact in this case.
\end{example}

Dilations are another simple example for a flat submanifold. In that case the velocity field $[B_\theta \eta](x)$ is not constant with respect to $\theta$ and $x$.
\begin{example}[Dilation]
Let $\mfold$ be a compact interval of the strictly positive numbers $\R_{++}$ and set $\psi(\theta,x)=\theta \cdot x$. Again there clearly exists a suitable compact set $\Omega$ on which all $\lambda_\theta$ are concentrated.
We find $\nabla_\theta \psi(\theta,x)=x$ and therefore \eqref{eq:EulerianVel} corresponds to the choice $B_\theta \eta: x \mapsto \tfrac{\eta}{\theta} x$. This is the gradient of $\phi_{\theta,\eta}(x)= \tfrac{\eta}{2\theta} |x|^2$ and one has
\begin{align*}
	\|B_\theta \eta\|_{\LL^2(\lambda_\theta)}^2 = \int_\Omega \left(\tfrac{\eta}{\theta}\right)^2 |x|^2\,\diff (\psi(\theta,\cdot)_{\#}\lambda)(x)=\eta^2 \cdot \int_{\R^d} |x|^2 \diff \lambda(x).
\end{align*}
It is easy to verify assumptions \eqref{eq:VelEquivalence} and \eqref{eq:VelTimeSmooth}. Although \eqref{eq:VelDerivatives} is not directly satisfied, since the velocity fields are unbounded with respect to $x$ on $\R^d$, this can be remedied by regularizing the potential $\phi_{\theta,\eta}$ outside of $\Omega$.
\end{example}

\subsection[Geodesic restriction of W to Lambda and pull-back of Riemannian tensor]{Geodesic restriction of $\W$ to $\Lambda$ and pull-back of Riemannian tensor}
\label{sec:Restriction}

While we could equip the set $\Lambda$ with the metric $\W$, for points that are far from each other, the distance and shortest paths according to $\W$ might not be very appropriate. This is similar to curved manifolds that are embedded into $\R^d$, where the Euclidean distance and straight lines might not be appropriate at large scales. Instead, we want to consider the geodesic restriction of $\W$ to $\Lambda$, i.e.~shortest paths between any two points must entirely lie within $\Lambda$. We now introduce this distance by adding the corresponding constraint to \Cref{prop:EnergyBasic}.
\begin{proposition}[Geodesic restriction of $\W$ to $\Lambda$]
\label{prop:SubMan:ParMan:WLambda}
For $\mu, \nu \in \Lambda$ set
\begin{align}
\W_\Lambda(\mu,\nu) & \assign \inf \left\{ \energy(\rho)^{1/2} \ \middle|\ \rho \in 
	\Lip([0,1];(\Lambda,\W)),\,\rho(0)=\mu,\,\rho(1)=\nu \right\}.
	\label{eq:WLambda}
\end{align}
Then $\W_\Lambda$ is a metric on $\Lambda$ with finite diameter and minimizers in \eqref{eq:WLambda} exist.
\end{proposition}

\begin{proof}
Finiteness of $\W_\Lambda$ follows by constructing explicit paths via the embedding of $\mfold$. For $\mu_0, \mu_1 \in \Lambda$, let $\gamma_0, \gamma_1 \in \mfold$ such that $\emb(\gamma_i)=\mu_i$, $i=0,1$. Let $\gamma : [0,1] \to \mfold$ be a constant speed geodesic from $\gamma_0$ to $\gamma_1$ (existence is provided by Assumption~\ref{asp:Main}\eqref{item:mfold}).
Let $\rho(t) \assign \emb(\gamma(t))=\lambda_{\gamma(t)}$.
Since $\gamma$ and $\emb$ are Lipschitz continuous, so is $\rho$ (with respect to $\W$).
Next, set $v(t) \assign B_{\gamma(t)}\dot{\gamma}(t)$ and let $\varphi$ be the flow induced by $v$, see Assumption \ref{asp:Main}\eqref{item:consistent}, which satisfies
\begin{align*}
\rho(t) & = \lambda_{\gamma(t)} = \varphi(t,\cdot)_{\#}\lambda_{\gamma(0)} = \varphi(t,\cdot)_{\#}\mu_0.
\end{align*}
We will show that $v \in \CE(\rho)$.
Indeed for any test function $\phi \in \CC^1([0,1] \times \Omega)$ one has
\begin{align}
& \int_0^1 \int_\Omega \partial_t \phi(t,\cdot) \, \diff \rho(t) \, \diff t +
	\int_0^1 \int_\Omega \langle \nabla \phi(t,\cdot) , v(t) \rangle \, \diff \rho(t) \, \diff t \nonumber \\
& = \int_0^1 \int_\Omega \partial_t\phi(t,\varphi(t,\cdot)) \, \diff \mu_0 \, \diff t +
	\int_0^1 \int_\Omega \langle \nabla \phi(t,\varphi(t,\cdot)) , \partial_t \varphi(t,\cdot) \rangle \, \diff \mu_0 \, \diff t \nonumber \\
& = \int_0^1 \int_\Omega \frac{\diff}{\diff t} \lp \phi(t,\varphi(t,\cdot)) \rp \, \dd \mu_0 \, \dd t \nonumber \\
& = \int_\Omega \phi(1,\varphi(1,\cdot)) - \phi(0,\varphi(0,\cdot)) \, \dd \mu_0 \nonumber\\
&	= \int_\Omega \phi(1,\cdot)\,\diff \rho(1)- \int_\Omega \phi(0,\cdot)\,\diff \rho(0).
	\label{eq:FlowIsCE}
\end{align}
Furthermore,
\begin{align}
\energy(\rho) & = \int_0^1 \|v(t)\|^2_{\LL^2(\rho(t))} \diff t
  = \int_0^1 \|B_{\gamma(t)} \dot{\gamma}(t)\|^2_{\LL^2(\rho(t))} \diff t \nonumber \\
  & \leq C^2 \cdot \int_0^1 \|\dot{\gamma}(t)\|^2_{\gamma(t)} \, \diff t
  = C^2 \cdot d_{\mfold}^2(\gamma_0,\gamma_1) < \infty, \label{eq:WLambdaEnergyBound}
\end{align}
where $C$ comes from~\eqref{eq:VelEquivalence} in Assumption~\ref{asp:Main}, and in particular does not depend on $\gamma$. This upper bound establishes the finite diameter.

To obtain the existence of minimizing curves $\rho$ let $\{\rho_n\}_{n\in\bbN}$ be a minimizing sequence of curves.
Since the $\{\rho_n\}_{n\in\bbN}$ are Lipschitz, they can be reparametrized to constant speed, i.e.~such that the $\LL^2(\rho_n(t))$-norm of the associated minimizing velocity fields $v_n$ is constant in time (see again \cite[Theorem 5.14]{santambrogio2015optimal} and context for more details) and equal to the Lipschitz constant.
Since the sequence is minimizing and due to the uniform bound \eqref{eq:WLambdaEnergyBound} the Lipschitz constants of the $\{\rho_n\}_{n\in\bbN}$ can therefore be assumed to be uniformly bounded, and hence $\{\rho_n\}_{n\in\bbN}$ are an equicontinuous subset of $\Ck{0}([0,1];(\Lambda,\W))$.
Since, in addition, $\Lambda$ is a compact subset of $(\prob(\Omega),\W)$, the Arzel\`a--Ascoli theorem then allows extraction (with respect to uniform converegence) of a cluster curve $\rho \in \Lip([0,1];(\Lambda,\W))$, which is a minimizing curve by the lower-semicontinuity of $\energy$ (see \Cref{prop:EnergyBasic}).

From the definition of $\W_\Lambda$ we see directly that $\W_\Lambda(\mu,\nu) = \W_\Lambda(\nu,\mu) \geq 0$.
We also see that $\W_\Lambda(\mu,\mu)=0$.
If $\W_{\Lambda}(\mu,\nu)=0$ then there exists $\bar{\rho}$ and $v\in\CE(\bar{\rho})$ with $\int_0^1 \|v(t)\|_{\Lp{2}(\rho(t))}^2 \, \dd t=0$.
It follows that $v(t)=0$ $t$-a.e.~and (choosing $\phi(t,x) = \phi(x)$ as the test function) $\int_\Omega \phi \, \dd \mu = \int_\Omega \phi \, \dd \rho(0) = \int_\Omega \phi \, \dd \rho(1) = \int_\Omega \phi \, \dd \nu$.
Hence $\mu=\nu$.

The triangle inequality then follows by standard arguments, see for instance \cite{DNSTransportDistances09} for details. Lipschitz paths can be reparametrized in time to have constant speed (i.e.~such that $\|v(t)\|^2_{\LL(\rho(t))}$ is constant in time almost everywhere). Subsequently, any two constant speed paths $\rho_{0,1}$ from $\mu_0$ to $\mu_1$ and $\rho_{1,2}$ from $\mu_1$ to $\mu_2$ can be concatenated to a joint path $\rho_{0,2}$ from $\mu_0$ to $\mu_2$ (allotting the appropriate amount of time to each segment) such that $\energy(\rho_{0,2})^{1/2} = \energy(\rho_{0,1})^{1/2}+\energy(\rho_{1,2})^{1/2}$.
\end{proof}

In the above proof we showed that $\W_\Lambda$ is finite by using a curve on $\mfold$ and its deformation field $B_{\gamma(t)} \dot{\gamma}(t)$, given by the differential, to bound the value of $\energy(\rho)$. The next result shows that this bound is actually optimal and that $\W_\Lambda$ can be computed equivalently directly on $\mfold$. 

\begin{proposition}
\label{prop:SubMan:Pullback:energy}
For $\gamma \in \Lip([0,1];(\mfold,d_\mfold))$ let
\begin{align*}
	\energy_\mfold(\gamma) & \assign \int_0^1 \|B_{\gamma(t)}\dot{\gamma}(t) \|^2_{\LL^2(\emb(\gamma(t)))} \, \dd t.
\end{align*}
Then $\energy_\mfold$ is equivalent to $\energy$ in the sense that
\begin{align*}
\energy(\emb \circ \gamma) &
	=\energy_\mfold(\gamma) & & \tn{for} \quad \gamma \in \Lip([0,1];(\mfold,d_\mfold)), \\
\energy(\rho) & 
	=\energy_\mfold(\emb^{-1} \circ \rho) & & \tn{for} \quad \rho \in \Lip([0,1];(\Lambda,\W)).
\end{align*}
In particular, for a curve $\gamma$, the velocity field $t \mapsto B_{\gamma(t)}\dot{\gamma}(t)=v(t)$ satisfies $\energy(\emb\circ \gamma) = \int_0^1 \|v(t)\|^2_{\Lp{2}(\emb(\gamma(t)))} \, \dd t$.
\end{proposition}
\begin{proof}
By Assumption \ref{asp:Main}\eqref{item:embed} the embedding $\emb : \mfold \to \Lambda$ is bi-Lipschitz w.r.t.~$d_\mfold$ and $\W$, hence the two classes of Lipschitz curves are equivalent.

Let $\gamma\in\Lip([0,1];\mfold)$, set $\rho = \emb \circ \gamma$ and $v(t)=B_{\gamma(t)}\dot{\gamma}(t)$.
As in the proof of \Cref{prop:SubMan:ParMan:WLambda}, equation \eqref{eq:FlowIsCE}, $v\in\CE(\rho)$ and therefore
$$\energy_\mfold(\gamma) = \int_0^1 \|v(t)\|^2_{\Lp{2}(\rho(t))} \geq \energy(\rho).$$
We will show that $v$ is indeed minimal. Let $w \in \CE(\rho)$ be a competitor. Plugging $v$ and $w$ into the distributional continuity equation for some test function $\psi \in \CC^1([0,1] \times \Omega)$ and then subtracting the two equations from each other one obtains
$$\int_0^1 \int_\Omega \langle \nabla \psi(t,\cdot) , w(t)-v(t) \rangle \,\diff \rho(t)\,\diff t=0.$$
A simple approximation argument then yields that for all $\psi \in \CC^1(\Omega)$ one has for a.e.~$t$ that
$$\int_\Omega \langle \nabla \psi , w(t)-v(t)\rangle\,\diff \rho(t)=0.$$
Using that $v(t)$ is a $\CC^1$ gradient field and is continuous in time (see Assumption \ref{asp:Main}\eqref{item:vel}), this implies that $v(t)$ and $w(t)-v(t)$ are orthogonal in $\LL^2(\rho(t))$ for a.e.~$t$.
This in turn implies that $\|w(t)\|^2_{\LL^2(\rho(t))} \geq \|v(t)\|^2_{\LL^2(\rho(t))}$ for a.e.~$t$ and therefore establishes minimality of $v$ and equality of $\energy_\mfold(\gamma)$ and $\energy(\rho)$

Choosing $\gamma=\emb^{-1}\circ \rho$ for given $\rho$ implies the second equivalence.
\end{proof}

\Cref{prop:SubMan:Pullback:energy} suggests interpreting the bi-linear form
\begin{align}
\label{eq:Pullback}
T_\theta \mfold \times T_\theta\mfold \ni (a,b) \mapsto \langle a,b \rangle_{\Lambda,\theta} \assign \langle B_\theta a, B_\theta b \rangle_{\LL^2(\emb(\theta))} \in \R
\end{align}
as a Riemannian tensor on $\mfold$ which is formally the pull-back of the Riemannian tensor of $\W$ via the embedding $\emb : \mfold \to \Lambda$. As the next statement shows, by construction this metric is equivalent to $d_\mfold$ on $\mfold$, which then implies the equivalence of $\W$ and $\W_\Lambda$ on $\Lambda$.

\begin{proposition}
\label{prop:MetricEquiv}
Let
\[d_{\Lambda}(\theta_0,\theta_1) \assign \W_\Lambda(\lambda_{\theta_0},\lambda_{\theta_1}).\]
Then
\begin{equation}
\label{eq:dLambdaViaEnergy}
d_{\Lambda}(\theta_0,\theta_1) = \inf \left\{ \energy_\mfold(\gamma)^{1/2} \middle| \gamma \in 
	\Lip([0,1];(\mfold,d_\mfold)),\,\gamma(0)=\theta_0,\,\gamma(1)=\theta_1 \right\}.
\end{equation}
Moreover, $d_\mfold$ and $d_\Lambda$ are equivalent metrics on $\mfold$, and $\W$ and $\W_\Lambda$ are equivalent metrics on $\Lambda$.
\end{proposition}
\begin{corollary}
\label{cor:WLambdaCompact}
$(\Lambda,\W_\Lambda)$ is a compact metric space.
\end{corollary}
\begin{proof}[Proof of \Cref{prop:MetricEquiv}]
The equality in \eqref{eq:dLambdaViaEnergy} follows directly from \Cref{prop:SubMan:Pullback:energy} since it means that the two minimization problems \eqref{eq:WLambda} and \eqref{eq:dLambdaViaEnergy} yield the same minimal value.

Let $\gamma \in \Lip([0,1];\mfold)$. Then by \eqref{eq:VelEquivalence}
\begin{align*}
& \int_0^1 \|\dot{\gamma}(t)\|^2_{\gamma(t)} \, \dd t \in [1/C^2,C^2] \cdot
  \int_0^1 \| B_{\gamma(t)} \dot{\gamma}(t) \|^2_{\Lp{2}(\emb(\gamma(t))} \, \dd t.
\end{align*}
Since $d_\mfold$ and $d_\Lambda$ are both defined via infimization over curves $\gamma$ of these respective functionals one has $d_\mfold(\theta_0,\theta_1) \in [1/C,C] \cdot d_\Lambda(\theta_0,\theta_1)$ and the two metrics are equivalent.
By assumption the embedding $\emb : \mfold \to \Lambda$ is bi-Lipschitz, i.e.~there is some constant $C'$ such that for any $\mu_0, \mu_1 \in \Lambda$ one has
\begin{align*}
	\W(\mu_0,\mu_1) & \in [1/C',C'] \cdot d_\mfold(\emb^{-1}(\mu_0),\emb^{-1}(\mu_1)) \\
	& \subseteq [1/(CC'),C C'] \cdot d_\Lambda(\emb^{-1}(\mu_0),\emb^{-1}(\mu_1)) \\
	& = [1/(CC'),C C'] \cdot \W_\Lambda(\mu_0,\mu_1).
\end{align*}
\end{proof}

\subsection{Local linearization}
\label{sec:LocalLin}

After having constructed the set $\Lambda$, the geodesic restriction $\W_\Lambda$ and the formal pull-back of the metric to $\mfold$, we now derive several linearization results that show that $\W$ and $\W_\Lambda$ can locally be well approximated in the respective tangent space norms. The first result states that when moving away from $\lambda_\theta$ in direction $\eta$ on $\mfold$, then the deformation velocity field $B_\theta \eta$ is to leading order the relative optimal transport map.
This can be interpreted as an extrinsic result on the embedding of $(\Lambda,\W_\Lambda)$ into $(\probL(\Omega),\W)$.
For this statement the assumption that $B_\theta \eta$ is a gradient field is crucial to allow for the application of Brenier's theorem.
\begin{proposition}
\label{prop:BasicDistBound}
There exists a constant $C \in (0,\infty)$ such that for any $(\theta,\eta) \in \domexp$ with $\|\eta\|_\theta < 1/C$, one has
$$\Bigg| \|B_{\theta} \eta\|_{\LL^2(\emb(\theta))}-\W(\lambda_\theta,\lambda_{\exp_\theta(\eta)}) \Bigg| \leq \|\eta\|_{\theta}^2 \cdot C.$$
\end{proposition}
\begin{proof}
If $\eta=0$ the statement is trivial, so assume $\eta \neq 0$ from now on.
In the following let $\gamma : [0,1] \ni t \mapsto \exp_\theta(t \cdot \eta)$
and recall
\begin{align*}
\lambda_{\gamma(t)} & = \emb(\gamma(t)), &
v_t & \assign B_{\gamma(t)} \dot{\gamma}(t)
\end{align*}
for $t \in [0,1]$.
Now recall that $v_0 = \nabla \phi$ for some $\phi \in \CC^2(\Omega)$ with $\|\nabla^2\phi\|_{\infty} \leq C \cdot \|\eta\|_\theta$ (see \Cref{asp:Main}\eqref{item:vel}).
Then one has for $\|\eta\|_\theta < 1/C$ that the map $x \mapsto \frac12 \|x\|^2 + \phi(x)$ is convex and therefore $\id + v_0: x \mapsto x+ v_0(x)$ is an optimal transport map between $\lambda_{\gamma(0)}$ and $\tilde{\lambda} \assign (\id + v_0)_\# \lambda_{\gamma(0)}$ by Brenier's theorem, i.e.~$W(\lambda_{\gamma(0)},\tilde{\lambda})=\|v_0\|_{\LL^2(\lambda_0)}$. (Note that $\tilde{\lambda}$ might be partially supported outside of $\Omega$, but due to the uniform pointwise bound on $v_0$, it will be concentrated in a compact superset of $\Omega$.) Consequently,
\begin{align}
	\|v_0\|^2_{\LL^2(\lambda_{\gamma(0)})}
	= \W(\lambda_{\gamma(0)},\tilde{\lambda})^2.
\end{align}

We now bound the distance $\W(\lambda_{\gamma(1)},\tilde{\lambda})$.
Let $\varphi_t \assign \varphi(t,\cdot)$ be the flow associated with the curve $\gamma$ in the sense of \Cref{lem:Flow}, i.e.~the solution to the initial value problem
\begin{align*}
\varphi_0 & = \id, & \partial_t \varphi_t = v_t \circ \varphi_t
\end{align*}
for $t \in [0,1]$. Then by \Cref{asp:Main}\eqref{item:consistent} (and recalling the definition of $\tilde{\lambda}$),
\begin{align*}
\lambda_{\gamma(t)} & = \varphi_{t\#} \lambda_{\gamma(0)}, &
\tilde{\lambda} & = (\varphi_0 + \dot{\varphi}_0)_\# \lambda_0
\end{align*}
and therefore $\W(\lambda_{\gamma(1)},\tilde{\lambda})^2\leq \int |(\varphi_1-\varphi_0)-\dot{\varphi}_0|^2\,\dd \lambda_{\gamma(0)}$.
For some $x \in \Omega$ one obtains by \Cref{lem:Flow}\eqref{item:FlowCurv} that
\begin{align*}
	|(\varphi(1,x)-\varphi(0,x))-\dot{\varphi}(0,x)|
	& = \left|\int_0^1 [\dot{\varphi}(t,x)-\dot{\varphi}(0,x)]\dd t\right| \\
	& =\left|\int_0^1 \int_0^t \ddot{\varphi}(s,x) \dd s \dd t\right| \leq \frac{\|\eta\|_\theta^2 \,C}{2},
\end{align*}
and consequently,
\begin{equation}
\label{eq:BasicDistTildeBound}
\W(\lambda_{\gamma(1)},\tilde{\lambda}) \leq \|\eta\|_\theta^2 \cdot C/2.
\end{equation}

We now combine the previous estimates and the triangle inequality to obtain
\begin{align*}
\|v_0\|_{\LL^2(\lambda_{\gamma(0)})} & = \W(\lambda_{\gamma(0)},\tilde{\lambda}) \\
 & \leq \W(\lambda_{\gamma(0)},\lambda_{\gamma(1)}) + \W(\lambda_{\gamma(1)},\tilde{\lambda}) \\
 & \leq \W(\lambda_{\gamma(0)},\lambda_{\gamma(1)}) + \|\eta\|_\theta^2\cdot C/2, \\
\|v_0\|_{\LL^2(\lambda_{\gamma(0)})} & = \W(\lambda_{\gamma(0)},\tilde{\lambda}) \\
 & \geq \W(\lambda_{\gamma(0)},\lambda_{\gamma(1)}) - \W(\lambda_{\gamma(1)},\tilde{\lambda}) \\
 & \geq \W(\lambda_{\gamma(0)},\lambda_{\gamma(1)}) - \|\eta\|_\theta^2 \cdot C/2.
\end{align*}
Relabeling the constant $C$ completes the proof.
\end{proof}

Having a local approximation of $\W$ in terms of the norm of the tangent vector $\|B_\theta \eta\|_{\Lp{2}(\lambda_\theta)}$ now allows to show that locally $\W$ and $\W_\Lambda$ agree to leading order.
This can be interpreted as implying that shortest paths between points in $\Lambda$ with respect to $\W$ remain close to $\Lambda$.

\begin{proposition}\label{prop:WLambdaWComparison}
There is a constant $C \in (0,\infty)$ such that
$$0 \leq \frac{W_{\Lambda}(\lambda_0,\lambda_1)-W(\lambda_0,\lambda_1)}{W(\lambda_0,\lambda_1)^2} \leq C$$
for all $\lambda_0,\lambda_1 \in \Lambda$, $\lambda_0 \neq \lambda_1$.
\end{proposition}
\begin{proof}
Let $\lambda_0, \lambda_1 \in \Lambda$, set $\theta_i \assign \emb^{-1}(\lambda_i)$ for $i \in \{0,1\}$,
and assume that $D \assign d_\mfold(\theta_0,\theta_1) >0$. By \Cref{asp:Main}\eqref{item:mfold} there is a tangent vector $\eta \in T_{\theta_0} \mfold$, $\|\eta\|_{\theta_0}=D$, such that $\gamma: [0,1] \ni t \mapsto \exp_{\theta_0}(t \cdot \eta)$ is a constant speed geodesic between $\theta_0$ and $\theta_1$ with $\|\dot{\gamma}(t)\|_{\gamma(t)}=D$ for $t \in [0,1]$.
As before, for $t \in [0,1]$, $\lambda_{\gamma(t)}=\emb(\gamma(t))$, $v_t = B_{\gamma(t)} \dot{\gamma}(t)$, and $\varphi_t$ is the flow induced by $v_t$.

First, we observe the trivial bound
\begin{equation*}
\W_\Lambda(\lambda_0,\lambda_1) \geq \W(\lambda_0,\lambda_1).
\end{equation*}
Then, using the definition of $\W_\Lambda$, \eqref{eq:WLambda} and the equivalence of \Cref{prop:SubMan:Pullback:energy} one has
\begin{align*}
\W_\Lambda(\lambda_0,\lambda_1) & \leq \lp \int_0^1 \|v_t\|_{\Lp{2}(\lambda_{\gamma(t)})}^2\,\diff t \rp^{\frac12}
= \lp \int_0^1 \|v_t \circ \varphi_t\|_{\Lp{2}(\lambda_0)}^2\,\diff t \rp^{\frac12} \nonumber \\
& \leq \left[
  \lp \int_0^1 \|v_0\|_{\Lp{2}(\lambda_0)}^2\,\diff t \rp^{\frac12}
  + \lp \int_0^1 \|v_t \circ \varphi_t - v_0\|_{\Lp{2}(\lambda_0)}^2\,\diff t \rp^{\frac12}
  \right] \nonumber \\
& \leq \|v_0\|_{\Lp{2}(\lambda_0)} + C \cdot D^2
\end{align*}
for some $C$ which does not depend on $\lambda_0,\lambda_1$.
In the first line we used that $\lambda_{\gamma(t)} = \varphi_{t\#} \lambda_0$, \Cref{asp:Main}\eqref{item:consistent}.
In the last inequality we used several regularity properties of $v_t$ and its induced flow $\varphi_t$, which are $\|v_t\|_{\CC^1(\R^d)} \leq C\cdot D$, \eqref{eq:VelDerivatives}, $\|\varphi_t-\id\|_{\CC(\R^d)} \leq C \cdot D \cdot t$, which together imply $\|v_t-v_t \circ \varphi_t\|_{\CC(\R^d)} \leq C^2\cdot D^2 \cdot t$, and finally $\|v_t-v_0\|_{\CC(\R^d)} \leq C \cdot D^2 \cdot t$, implied by \eqref{eq:VelTimeSmooth}.

From \Cref{prop:BasicDistBound} we have that if $D <1/C$, then
\begin{equation*}
\Big| \|v_0\|_{\LL^2(\lambda_0)}-\W(\lambda_0,\lambda_1) \Big| \leq D^2 \cdot C
\end{equation*}
where $C$ is the constant from \Cref{prop:BasicDistBound} and does not depend on $\lambda_0,\lambda_1$.
Plugging this into the previous equation and then combining it with the lower bound on $\W_\Lambda$, we obtain
\begin{equation*}
0 \leq \W_\Lambda(\lambda_0,\lambda_1) - \W(\lambda_0,\lambda_1) \leq D^2\cdot C
\end{equation*}
when $D<1/C$ for some suitable $C$ again not depending on $\theta_0,\theta_1$.

Using the equivalence of $\W$, $\W_\Lambda$ and $d_\mfold$, $d_{\Lambda}$ and the relationship between $\W_\Lambda$ and $d_{\Lambda}$ in \Cref{prop:MetricEquiv} there is another global constant $C'$ such that
$$\W(\lambda_0,\lambda_1)^2 \geq (C')^2\,d_\mfold(\theta_0,\theta_1)^2=D^2\cdot(C')^2.$$
So if $D=d_\mfold(\theta_0,\theta_1) <1/C$, by division we get
\begin{equation*}
0 \leq \frac{\W_\Lambda(\lambda_0,\lambda_1) - \W(\lambda_0,\lambda_1)}{\W(\lambda_0,\lambda_1)^2} \leq C
\end{equation*}
again for a global constant $C$. Using once more the metric equivalence, \Cref{prop:MetricEquiv}, there is a constant $\tau>0$ such that this holds whenever $\W_\Lambda(\lambda_0,\lambda_1) < \tau$.

Conversely, if $\W(\lambda_0,\lambda_1) \geq \tau$, using the finite diameter of $(\Lambda,\W_\Lambda)$ (see \Cref{prop:SubMan:ParMan:WLambda}) there is some other $C$ such that
\begin{equation*}
0 \leq \frac{\W_\Lambda(\lambda_0,\lambda_1) - \W(\lambda_0,\lambda_1)}{\W(\lambda_0,\lambda_1)^2} \leq C.
\end{equation*}
Taking $C$ as the maximum of these two bounds completes the proof.
\end{proof}

Finally, we can state a local linearization result in the spirit of \eqref{eq:IntroRiemannBound} relating to the question: how does the distance $W_{\Lambda}$ change locally around a fixed measure $\lambda_\theta$ for perturbations corresponding to movements in the tangent space of $\theta$ on $\mfold$.
Note that this is an intrinsic statement on the metric space $(\Lambda,\W_\Lambda)$ without reference to the ambient space $(\probL(\Omega),\W)$, but it is proved by extrinsic arguments (including the previous statement).

\begin{theorem}\label{thm:LocalLinearization}
There exists a constant $C \in (0,\infty)$, such that for any $\theta \in \mfold$, $\eta_1,\eta_2 \in \domexp_\theta$, and with $\gamma_i : [0,1] \ni t \mapsto \exp_\theta(t \cdot \eta_i)$ for $i=1,2$ one has
\begin{equation}\label{eq:LocalLinearizationTheorem} 
\left|\W_\Lambda(\lambda_{\gamma_1(t)},\lambda_{\gamma_2(t)})
	- t \cdot \| B_{\theta}(\eta_1-\eta_2)\|_{\LL^2(\lambda_\theta)}\right|
\leq t^2\cdot C \left( \|\eta_1\|_\theta+\|\eta_2\|_\theta \right)^2
\end{equation}
for $t\in[0,\min\{1,\frac{1}{C\|\eta_1\|_\theta},\frac{1}{C\|\eta_2\|_\theta}\})$.
Moreover, this implies
\[ \frac{\dd}{\dd t} \left. \W_\Lambda(\lambda_{\gamma_1(t)},\lambda_{\gamma_2(t)}) \right|_{t=0} = \| B_\theta(\eta_1-\eta_2) \|_{\LL^2(\emb(\theta))}. \]
\end{theorem}
\begin{proof}
If both $\eta_i$ are zero, the statement is trivial. If one is zero, the statement reduces to \Cref{prop:BasicDistBound}. Thus we assume that both $\eta_i$ are non-zero for the remainder of the proof.

We start by proving \eqref{eq:LocalLinearizationTheorem} with $\W$ on the left-hand side rather than $\W_\Lambda$. Let
\[\lambda_0 \assign \lambda_\theta=\emb(\gamma_i(0)),
	\quad \lambda_{\gamma_i(r)} =  \emb(\gamma_i(r)),
	\quad v_r^i \assign B_{\gamma_i(r)}\dot{\gamma}_i(r),
	\quad i=1,2, \, r\in[0,1].\]
As in the proof of \Cref{prop:BasicDistBound}, we obtain
\[v_0^i = \nabla\phi^i \qquad \tn{for some} \qquad \phi^i\in \CC^2(\Omega),\]
with $\|\nabla^2\phi^i\|_{\infty}\leq C\cdot \|\eta_i\|_{\theta}$ for $i=1,2$. For $t \in [0,1]$ set
\[ \widetilde\lambda_{\gamma_i(t)} \assign (\id+t\cdot v_0^i)_\#\lambda_0 = (\id+t\cdot\nabla\phi^i)_\#\lambda_0 = \left(\nabla\left(\frac12|x|^2+t\cdot\phi^i\right)\right)_\#\lambda_0.\]
Let $\tau \assign \min\{1,\frac{1}{C\|\eta_1\|_\theta},\frac{1}{C\|\eta_2\|_\theta}\}$. Then for $t\in[0,\tau)$, the functions $x\mapsto \frac12|x|^2+t\cdot \phi^i(x)$, $i=1,2$ are strictly convex and hence $\id+t\cdot v_0^i$ are optimal transport maps between $\lambda_0$ and $\widetilde\lambda_{\gamma_i(t)}$ by Brenier's theorem, and we have, analogous to \eqref{eq:BasicDistTildeBound},
\[\W(\lambda_0,\widetilde\lambda_{\gamma_i(t)}) = t\cdot\|v_0^i\|_{\LL^2(\lambda_0)},\qquad \W(\lambda_{\gamma_i(t)},\widetilde\lambda_{\gamma_i(t)})\leq t^2\cdot \|\eta_i\|_\theta^2 \cdot C.\]
Consequently, by the triangle inequality
\[|\W(\lambda_{\gamma_1(t)},\lambda_{\gamma_2(t)})-\W(\widetilde\lambda_{\gamma_1(t)},\widetilde\lambda_{\gamma_2(t)})|
\leq t^2\cdot C \cdot (\|\eta_1\|_\theta^2+\|\eta_2\|_\theta^2).\]  
Therefore, 
\begin{equation}
\label{eq:lambdatolambdatilde}
\begin{aligned}
& |\W(\lambda_{\gamma_1(t)},\lambda_{\gamma_2(t)})-t\|B_\theta(\eta_1-\eta_2)\|_{\LL^2(\lambda_0)}| \\
& \qquad \leq |\W(\widetilde\lambda_{\gamma_1(t)},\widetilde\lambda_{\gamma_2(t)})-t\cdot\|B_\theta(\eta_1-\eta_2)\|_{\LL^2(\lambda_0)}|
+ t^2\cdot C \cdot (\|\eta_1\|_\theta^2+\|\eta_2\|_\theta^2).
\end{aligned}
\end{equation}
Thus it suffices to control the discrepancy with the $\lambda_{\gamma_i(t)}$ replaced by $\widetilde\lambda_{\gamma_i(t)}$ on the left-hand side.

First, note that, since $\id+t\cdot \nabla\phi^i$ are optimal transport maps, we have
\begin{equation}\label{eq:LinearizationUpperBound}
\W(\widetilde\lambda_{\gamma_1(t)},\widetilde\lambda_{\gamma_2(t)}) \leq \|t\cdot \nabla\phi^1-t\cdot\nabla\phi^2\|_{\LL^2(\lambda_0)} = t\cdot\|B_\theta(\eta_1-\eta_2)\|_{\LL^2(\lambda_0)}.
\end{equation}
This yields the desired upper bound.

The lower bound is more delicate. We will introduce an auxiliary point $\doubletilde$ which is close to $\widetilde\lambda_{\gamma_2(t)}$ (hence also to $\lambda_{\gamma_2(t)}$) but for which we have an explicit optimal transport map from $\widetilde\lambda_{\gamma_1(t)}$ to $\doubletilde$. To do so, we start with the sequence of optimal transport maps $\widetilde\lambda_{\gamma_1(t)}\to\lambda_0\to\widetilde\lambda_{\gamma_2(t)}$. Let $F_t(x) = \frac12|x|^2+t\cdot\phi^1(x)$.
Note that $F_t$ is twice continuously differentiable, strictly convex for $0 \leq t <1/C \cdot \|\eta_1\|_\theta$, and has asymptotically quadratic growth as $|x| \to \infty$. Moreover, the eigenvalues of $\nabla^2 F_t(x)$ are contained in $[1-t\cdot C \cdot \|\eta_1\|_\theta,1+t\cdot C \cdot \|\eta_1\|_\theta]$.
This implies that the Fenchel--Legendre conjugate $F_t^*$ is strictly convex, continuously differentiable, $\nabla F_t^* = (\nabla F_t)^{-1}$ (see for instance \cite[Theorem 11.13]{rockafellar2009variational}), and by the inverse function theorem even twice continuously differentiable with eigenvalues of $\nabla^2 F_t^*$ contained in $[(1+t\cdot C \cdot \|\eta_1\|_\theta)^{-1},(1-t\cdot C \cdot \|\eta_1\|_\theta)^{-1}]$.
This means that $F_t^*$ can be written as
\begin{equation*}
F_t^*(x)=\tfrac12 |x|^2+t\cdot\widetilde{\phi}_t^1(x)
\end{equation*}
with $\|\nabla^2 \widetilde{\phi}_t^1\|_{\Lp{\infty}}$ uniformly bounded in $t$, for $t$ bounded away from $1/(C \cdot \|\eta_1\|_\theta)$, in particular, as $t \to 0$.
Indeed, since $\nabla^2 \tilde{\phi}_t^1=\tfrac1t (\nabla^2 F_t^* - \id)$, the eigenvalues of $\nabla^2 \tilde{\phi}_t^1$ are bounded from above by
$$\tfrac1t (1-t\cdot C \cdot \|\eta_1\|_\theta-1)=\tfrac{C \cdot \|\eta_1\|_\theta}{1- t \cdot C \cdot \|\eta_1\|_\theta}$$
and a similar argument can be used for a lower bound.
The same regularity argument has been used (on the torus) in \cite[Lemma 2.3]{berman2020sinkhorn}.
Finally, $f_1 \assign \nabla F_t^*$ is the optimal transport map from $\widetilde\lambda_{\gamma_1(t)}$ back to $\lambda_0$.

We denote the optimal transport map from $\lambda_0\to\widetilde\lambda_{\gamma_2(t)}$ via
\[f_2 = \id+t\cdot\nabla\phi^2.\]
Note that $\widetilde\lambda_{\gamma_2(t)} = (f_2\circ f_1)_{\#}\widetilde\lambda_{\gamma_1(t)}$, however $f_2\circ f_1$ is in general not the optimal transport map since it may not be the gradient of a convex function as the following shows:
\begin{align*} f_2\circ f_1(x) & = (\id+t\cdot\nabla \phi^2)\circ(\id+t\cdot\nabla\widetilde\phi_t^1)(x)\\
& = x+t\cdot\nabla\widetilde\phi_t^1(x) + t\cdot\nabla\phi^2(x+t\cdot\nabla\widetilde\phi_t^1(x)).
\end{align*}
Instead we consider the approximation
\[g \assign \id + t\cdot\nabla\left(\widetilde\phi_t^1+\phi^2\right),
\quad \tn{and set}\quad \doubletilde \assign g_\# \lambda_{\gamma_1(t)}.\]
The function $g$ is evidently the gradient of a convex function (for sufficiently small $t$), and thus is a transport map from $\lambda_0\to\doubletilde$. 
Notice that
\[g(x) = f_2\circ f_1(x)+ t\cdot\left(\nabla\phi^2(x) - \nabla\phi^2(x+t\cdot\nabla\widetilde\phi_t^1(x))\right) =: f_2\circ f_1(x)+E_t(x).\]
Since $g$ is an optimal transport map, we have \[\W(\widetilde\lambda_{\gamma_1(t)},\doubletilde) = \|g-\id\|_{\LL^2(\widetilde\lambda_{\gamma_1(t)})} = \|f_2\circ f_1 + E_t-\id\|_{\LL^2(\widetilde\lambda_{\gamma_1(t)})}.\]
Next, notice that
\begin{align}\label{eq:WlambdatildeLowerBound}
\W(\widetilde\lambda_{\gamma_1(t)},\widetilde\lambda_{\gamma_2(t)}) & \geq \W(\widetilde\lambda_{\gamma_1(t)},\doubletilde) - \W(\doubletilde,\widetilde\lambda_{\gamma_2(t)}) \nonumber\\
& = \|f_2\circ f_1 + E_t - \id\|_{\LL^2(\widetilde\lambda_{\gamma_1(t)})} - \W(\doubletilde,\widetilde\lambda_{\gamma_2(t)}) \nonumber\\
& \geq \|f_2\circ f_1 - \id\|_{\LL^2(\widetilde\lambda_{\gamma_1(t)})} - \|E_t\|_{\LL^2(\widetilde\lambda_{\gamma_1(t)})}- \W(\doubletilde,\widetilde\lambda_{\gamma_2(t)}) \nonumber\\
& \geq \|f_2-f_1^{-1}\|_{\LL^2(\lambda_0)}-\|E_t\|_{\LL^2(\widetilde\lambda_{\gamma_1(t)})}- \W(\doubletilde,\widetilde\lambda_{\gamma_2(t)}) \nonumber\\
& = t\|B_\theta(\eta_1-\eta_2)\|_{\LL^2(\lambda_0)}-\|E_t\|_{\LL^2(\widetilde\lambda_{\gamma_1(t)})}- \W(\doubletilde,\widetilde\lambda_{\gamma_2(t)}).
\end{align}
It remains to estimate $\|E_t\|_{\LL^2(\widetilde\lambda_{\gamma_1(t)})}$ and $\W(\doubletilde,\widetilde\lambda_{\gamma_2(t)})$.

Note that $g\circ (f_2\circ f_1)^{-1}$ is a (not necessarily optimal) map that pushes $\widetilde\lambda_{\gamma_2(t)}$ to $\doubletilde$, hence we get
\begin{equation}
\label{eq:WTripleTildeBound}
\begin{aligned}
    \W(\doubletilde,\widetilde\lambda_{\gamma_2(t)}) & \leq \|g\circ (f_2\circ f_1)^{-1}- \id\|_{\LL^2(\widetilde\lambda_{\gamma_2(t)})} \\
    & = \|g- f_2\circ f_1\|_{\LL^2(\widetilde\lambda_{\gamma_1(t)})} \\
    & = \|E_t\|_{\LL^2(\widetilde\lambda_{\gamma_1(t)})}.
\end{aligned}
\end{equation}
Therefore, it remains to estimate the norm of $E_t$.
First, note that \eqref{eq:VelDerivatives} implies that $v_0^2$ is Lipschitz, with
\[|v_0^2(x)-v_0^2(y)| \leq C \cdot\|\eta_2\|_{\theta}|x-y|.\]
Hence, recalling that $v_0^2 = \nabla\phi^2$,
\begin{align*}
    \|E_t\|_{\LL^2(\widetilde\lambda_{\gamma_1(t)})}^2 & = t^2\cdot\int|v_0^2(x)-v_0^2(x+t\cdot\nabla\widetilde\phi_t^1(x))|^2 \, \dd\widetilde\lambda_{\gamma_1(t)}(x)\\
    & \leq t^4\cdot C \cdot \|\eta_2\|_{\theta}^2\int|\nabla\widetilde\phi_t^1(x)|^2 \, \dd\widetilde\lambda_{\gamma_1(t)}(x)\\
    & = t^4\cdot C\cdot  \|\eta_2\|_{\theta}^2 \, \|\nabla \widetilde\phi_t^1\|_{\LL^2(\widetilde\lambda_{\gamma_1(t)})}^2.
\end{align*}
Recall that $\nabla F_t = \id+t\cdot\nabla \phi^1$ and $\nabla F^*_t = (\nabla F_t)^{-1} = \id + t\cdot\nabla \widetilde\phi_t^1$. Therefore,
\begin{align*}
    \|\nabla\widetilde\phi_t^1\|_{\LL^2(\widetilde\lambda_{\gamma_1(t)})} & = \left\|\frac{\nabla F^*_t-\id}{t}\right\|_{\LL^2(\widetilde\lambda_{\gamma_1(t)})}
    = \left\|\frac{\id-\nabla F_t}{t}\right\|_{\LL^2(\lambda_0)} \\
    & = \|\nabla\phi^1\|_{\LL^2(\lambda_0)}
    = \|v_0^1\|_{\LL^2(\lambda_0)}.
\end{align*}
Consequently, using \eqref{eq:VelEquivalence},
\begin{equation}\label{eq:EtBound}
\|E_t\|_{\LL^2(\widetilde\lambda_{\gamma_1(t)})} \leq t^2\cdot C^2\cdot  \|\eta_2\|_{\theta} \|\eta_1\|_{\theta}.
\end{equation}

Finally, combining \eqref{eq:lambdatolambdatilde}, \eqref{eq:LinearizationUpperBound}, \eqref{eq:WlambdatildeLowerBound}, \eqref{eq:WTripleTildeBound}, and \eqref{eq:EtBound} yields, for a suitable constant $C$, not depending on $\theta,\eta_1,\eta_2$ or $t$, that whenever $t \in [0,1/C)$,
\begin{equation}
\label{eq:WminusTangentBound}
\left|\W(\lambda_{\gamma_1(t)},\lambda_{\gamma_2(t)}) - t\cdot\|B_\theta(\eta_1-\eta_2)\|_{\LL^2(\lambda_0)}\right|
\leq t^2\cdot C\cdot(\|\eta_1\|_\theta^2+\|\eta_2\|_\theta^2 + \|\eta_1\|_\theta\, \|\eta_1\|_\theta).
\end{equation}
Hence, \Cref{prop:WLambdaWComparison} implies that (for possibly a new $C$),
\begin{align*}
\left|\W_\Lambda(\lambda_{\gamma_1(t)},\lambda_{\gamma_2(t)}) - t\cdot\|B_\theta(\eta_1-\eta_2)\|_{\LL^2(\lambda_0)}\right|
\leq t^2\cdot C\cdot (\|\eta_1\|_\theta^2+\|\eta_2\|_\theta^2 + \|\eta_1\|_\theta\, \|\eta_1\|_\theta) \\ + C \cdot \W(\lambda_{\gamma_1(t)},\lambda_{\gamma_2(t)})^2,
\end{align*}
which upon usage of \eqref{eq:WminusTangentBound} again, \eqref{eq:VelEquivalence}, and controlling the terms of higher order in $t$ and $\|\eta_i\|_\theta$ by the second order terms (once more, with potentially a new $C$, and using that $\|\eta_i\|_\theta$ in $\domexp$ are bounded) implies the desired bound. 
\end{proof}

This means that the local linearization of the metric $\W_\Lambda$ works similar as on a finite-dimen\-sional Riemannian manifold and much better than in the `full manifold' $\probL(\Omega)$, cf.~\cite{delalande21}. The local deformations described by the velocity fields $B_\theta \eta_i$ may be more general than translations, scaling and shearings, which are considered, for instance, in \cite{MooClo22}.

\subsection[Local linearization in the ambient space W]{Local linearization in the ambient space $\W$}

From a practical perspective, the latent manifold and the deformation velocity fields may not be explicitly known. More realistically, we may merely have access to samples or curves $\lambda$ in $\Lambda$ and can only evaluate the ambient distance $\W$, not $\W_\Lambda$.
In this section we show first how the deformation velocity fields can be recovered from observing curves and distances in $\W$. Afterwards we can state a linearization result directly for $\W$ and optimal transport maps, giving a quantitative version of \eqref{eq:IntroApprox}.

\begin{proposition}\label{prop:wvconvergence}
Let $(\theta,\eta) \in \domexp$.
For $t \in (0,1]$, let $T_t$ be the optimal transport map from $\lambda_\theta$ to $\lambda_{\exp_\theta(t\cdot\eta)}$ and set
\begin{equation}
w_t \assign (T_t-\id)/t.
\end{equation}
Then $\lim_{t \to 0} w_t = v$ strongly in $\LL^2(\lambda_\theta)$ where $v = B_{\theta} \eta$.
More quantitatively, there is a constant $C$ that does not depend on $(\theta,\eta)$, such that for $t \in (0,1/C)$ one has
\begin{equation}
\label{eq:wvconvergencemodulus}
\|v-w_t\|_{\LL^2(\lambda_\theta)} \leq C \sqrt{t} \|\eta\|_\theta.
\end{equation}
\end{proposition}
\begin{remark}
For a submanifold of $\R^k$, equipped with the geodesic restriction of the ambient Euclidean distance, the equivalent of $v=\eta$ in the above statement would be a tangent vector of the submanifold and $w_t$ would correspond to a `finite difference' secant vector in the ambient space. If the submanifold's curvature is bounded, one obtains a bound linear in $t$ for the equivalent of \eqref{eq:wvconvergencemodulus}.
Note that the proof does not assume or exploit whether the transport maps $T_t$ are continuous or even differentiable.
It is an interesting open question if and under what conditions the above rate can be improved for Wasserstein submanifolds.
An example with linear rate is discussed in \Cref{rem:W1uv}.
\end{remark}

\begin{proof}
The statement is clearly true for $\eta=0$, so assume $\eta\neq 0$ in the following.
Let $\gamma : [0,1] \ni t \mapsto \exp_\theta(t\cdot\eta)$, and for simplicity let $\lambda_0=\lambda_{\gamma(0)}$.
By \Cref{asp:Main}\eqref{item:mfold} $\gamma$ is well-defined and one has $d_\mfold(\theta,\gamma(t)) = t\|\eta\|_\theta$ and $\|\dot{\gamma}(t)\|_{\gamma(t)} = \|\eta\|_\theta$ for $t \in [0,1]$.

For $\phi\in\Ck{2}(\Omega)$ one has
\[ \int \lp \phi\circ T_t - \phi \rp \, \dd \lambda_0 \leq \int \langle \nabla \phi, T_t-\id \rangle \, \dd \lambda_0 + \frac{\|\nabla^2\phi\|_{\CC(\Omega)}}{2} \int \|T_t - \id\|^2 \, \dd\lambda_0, \]
which implies
\begin{equation}
\label{eq:wvGradPhi}
\int \langle \nabla \phi, w_t \rangle \, \dd \lambda_0 \geq \frac{1}{t} \int \lp \phi\circ T_t - \phi\rp \, \dd \lambda_0 - \frac{t\|\nabla^2\phi\|_{\CC(\Omega)}\|w_t\|_{\Lp{2}(\lambda_0)}^2}{2}.
\end{equation}

As earlier, let $\varphi_t \assign \varphi(t,\cdot)$ be the flow associated with the path $\gamma$, as given by \Cref{lem:Flow}.
We observe, from \Cref{lem:Flow}\eqref{item:FlowCurv},
\[ \|\varphi_t - \id - t \cdot v\|_{\CC(\R^d)} \leq t^2\cdot \|\eta\|_\theta^2 \cdot C\]
for some $C$ independent of $t,\theta,\eta$.
Using that $\lambda_t=\varphi_{t\#}\lambda_0=T_{t\#} \lambda_0$, the same arguments as above, and using the control on $\varphi$, we obtain
\begin{align}
& \frac{1}{t} \int \lp \phi\circ T_t - \phi\rp \, \dd \lambda_0
  = \frac{1}{t} \int \lp \phi\circ \varphi_t - \phi \rp \, \dd \lambda_0 
  \nonumber \\
\geq {}& \frac{1}{t} \int \langle \nabla \phi,\varphi_t-\id \rangle \, \dd \lambda_0 - \frac{\|\nabla^2 \phi\|_{\CC(\Omega)} \cdot \|\varphi_t-\id\|_{\Lp{2}(\lambda_0)}^2}{2t} \nonumber \\
\geq {} & \int \langle \nabla \phi,v \rangle \, \dd \lambda_0
  - t\,\|\eta\|_\theta^2\,C \|\nabla\phi\|_{\Lp{2}(\lambda_0)}
  - t\frac{\|\nabla^2 \phi\|_{\CC(\Omega)}}{2}
  \lp \|v\|_{\Lp{2}(\lambda_0)} + t\,\|\eta\|_\theta^2\,C \rp^2.
  \label{eq:wvDeltaPhi}
\end{align}

By \Cref{prop:BasicDistBound}, one has for $t\,\|\eta\|_\theta \in (0,\tfrac{1}{C})$,
\begin{equation}
\label{eq:wvw2}
\|w_t\|_{\LL^2(\lambda_0)}= \W(\lambda_0,\lambda_{\gamma(t)})/t \leq \|v\|_{\LL^2(\lambda_0)} + t\,\|\eta\|_{\theta}^2\cdot C.
\end{equation}

By \Cref{asp:Main}\eqref{item:vel} there is a twice continuously differentiable $\phi : \R^d \to \R$ such that $v=B_\theta \eta=\nabla \phi$. This will now be our choice for $\phi$ in the above estimates. Assumptions \eqref{eq:VelEquivalence} and \eqref{eq:VelDerivatives} then imply the bounds
\begin{align*}
	\|v\|_{\LL^2(\lambda_0)} & = \|\nabla \phi\|_{\LL^2(\lambda_0)} \leq C \cdot \|\eta\|_\theta, &
	\|\nabla^2 \phi\|_{\CC(\R^d)} & \leq C \cdot \|\eta\|_\theta.
\end{align*}
With \eqref{eq:wvw2} this implies the bound
$$\|w_t\|_{\LL^2(\lambda_0)} \leq C \cdot (\|\eta\|_\theta +t \|\eta\|_\theta^2).$$

Now we apply the choice of $\phi$ such that $v=\nabla \phi$ in \eqref{eq:wvGradPhi}, then use \eqref{eq:wvDeltaPhi}, and finally the above norm controls. In addition, for simplicity, we will also use the bound $0<t \leq 1$ to bound terms of order $>1$ in $t$ by first order terms. In combination one then obtains
\begin{align}
\label{eq:wvInnerProduct}
\int_\Omega \langle v,w_t \rangle \dd \lambda_0 & \geq \|v\|_{\Lp{2}(\lambda_0)}^2 - t \cdot C \cdot \left(
  \|\eta\|_\theta^3+\|\eta\|_\theta^4+\|\eta\|_\theta^5
  \right)
\end{align}
when $t<\tfrac{1}{C\|\eta\|_\theta}$ for a suitable constant $C$, independent of $\theta,\eta$.

Combining now \eqref{eq:wvw2} and \eqref{eq:wvInnerProduct} yields
\begin{align*}
\|v-w_t\|_{\Lp{2}(\lambda_0)}^2 & = \|v\|_{\Lp{2}(\lambda_0)}^2 + \|w_t\|_{\Lp{2}(\lambda_0)}^2 - 2\langle v,w_t\rangle_{\Lp{2}(\lambda_0)} \\
  & \leq t \cdot C \cdot \left(
  \|\eta\|_\theta^2+\|\eta\|_\theta^3+\|\eta\|_\theta^4+\|\eta\|_\theta^5
  \right).
\end{align*}
Using that $(\theta,\eta) \in \domexp$ implies that $\|\eta\|_\theta \leq \diam \mfold$, and thus (by adjusting the constant $C$ if necessary) all higher-order terms in $\|\eta\|_\theta$ can be bounded by the second-order term and we thus arrive at a bound which has the form made in the claim.
\end{proof}

\begin{corollary}
\label{cor:LocalLinearizationW}
Consider the setting of \Cref{thm:LocalLinearization}, let $T_t^i$ be the optimal transport map from $\lambda_\theta$ to $\lambda_{\exp_\theta(t\cdot\eta_i)}$ and $w_t^i\assign(T_t^i-\id)/t$ (with well defined limit as $t \to 0$, see \Cref{prop:wvconvergence}) for $i=1,2$.
Then there exists a constant $C \in (0,\infty)$ (not depending on $\theta$, $\eta_1$, $\eta_2$ and $t$), such that for $t\in\left[0,\min\left\{1,\frac{1}{C\|\eta_1\|_\theta},\frac{1}{C\|\eta_2\|_\theta}\right\}\right)$,
\[ \left|\W(\lambda_{\gamma_1(t)},\lambda_{\gamma_2(t)}) - \|T_t^1-T_t^2\|_{\LL^2(\emb(\theta))}\right|
	\hspace{-1.5pt} \leq \hspace{-1.5pt} t^2\cdot C \cdot \left(\|\eta_1\|_\theta\hspace{-1pt}+\hspace{-1pt}\|\eta_2\|_\theta\right)^2 
	+ t^{\frac32} \cdot C \cdot \lp \|\eta_1\|_\theta \hspace{-1pt}+\hspace{-1pt} \|\eta_2\|_\theta \rp, \]
and in particular
\begin{align*}
\frac{\diff}{\diff t}\left. \W(\lambda_{\gamma_1(t)},\lambda_{\gamma_2(t)}) \right|_{t=0} = \| B_\theta(\eta_1-\eta_2) \|_{\LL^2(\emb(\theta))}.
\end{align*}
\end{corollary}
\begin{proof}
Let
\begin{align*}
\lambda_\theta & \assign \emb(\theta), & v^i & \assign B_\theta\eta_i, & w_t^i & \assign \frac{T_t^i-\id}{t}.
\end{align*}
Then for $t$ in the given range, 
\begin{align*}
& \left|\W(\lambda_{\gamma_1(t)},\lambda_{\gamma_2(t)}) - \|T_t^1-T_t^2\|_{\LL^2(\lambda_\theta)}\right| \\
& \quad = \left|\W(\lambda_{\gamma_1(t)},\lambda_{\gamma_2(t)})
	- t\cdot\|w_t^1 - w_t^2\|_{\LL^2(\lambda_\theta)}\right| \\
& \quad \leq \left|\W(\lambda_{\gamma_1(t)},\lambda_{\gamma_2(t)})
	- t\cdot\|v^1 - v^2\|_{\LL^2(\lambda_\theta)}\right|
	+ t\cdot \left| \|v^1-v^2\|_{\LL^2(\lambda_\theta)}
	- \|w_t^1-w_t^2\|_{\LL^2(\lambda_\theta)}\right|\\
& \quad \leq t^2\cdot C \cdot (\|\eta_1\|_\theta^2+\|\eta_2\|_\theta^2 +\|\eta_1\|_{\theta}\,\|\eta_2\|_{\theta})
	+ t^{\frac{3}{2}} \cdot C \cdot \lp \|\eta_1\|_\theta + \|\eta_2\|_\theta \rp,
\end{align*}
where the first bound in the final inequality comes from \eqref{eq:WminusTangentBound} and the second comes from the reverse triangle inequality and Proposition \ref{prop:wvconvergence}.
\end{proof}

For convenience, in the previous results we have used geodesic curves on $\mfold$ to model perturbations of measures. The following corollary illustrates that this can be extended to general differentiable curves in a straightforward way.

\begin{corollary}[Extension to general differentiable curves]
Let $\gamma_i : [0,\veps) \to \mfold$ for some $\veps>0$ be differentiable curves with $\gamma_i(0)=\theta \in \mfold$ and $\dot{\gamma}_i(0)=\eta_i \in \domexp_\theta$ such that $d_\mfold(\gamma_i(t),\exp_\theta(t \cdot \eta_i)) = o(t)$ as $t \to 0$.
Then, with the notation of \Cref{thm:LocalLinearization} and \Cref{cor:LocalLinearizationW} one has as $t \to 0$ that
\begin{align*}
\left|\W_\Lambda(\lambda_{\gamma_1(t)},\lambda_{\gamma_2(t)})
	- t \cdot \| B_{\theta}(\eta_1-\eta_2)\|_{\LL^2(\lambda_\theta)}\right|=o(t), \\
\left|\W(\lambda_{\gamma_1(t)},\lambda_{\gamma_2(t)}) - \|T_t^1-T_t^2\|_{\LL^2(\lambda_\theta)}\right|=o(t),
\end{align*}
and
\[
\frac{\dd}{\dd t} \left. \W_\Lambda(\lambda_{\gamma_1(t)},\lambda_{\gamma_2(t)}) \right|_{t=0} =
\frac{\dd}{\dd t} \left. \W(\lambda_{\gamma_1(t)},\lambda_{\gamma_2(t)}) \right|_{t=0}  =
\| B_\theta(\eta_1-\eta_2) \|_{\LL^2(\lambda_\theta)}.
\]
\end{corollary}
This follows directly by plugging the bound $d_\mfold(\gamma_i(t),\exp_\theta(t \cdot \eta_i)) = o(t)$ into the estimates of \Cref{thm:LocalLinearization} and \Cref{cor:LocalLinearizationW} and the new bound can be improved if a better bound on $d_\mfold(\gamma_i(t),\exp_\theta(t \cdot \eta_i))$ is assumed.

\section{Sampling and graph approximation}
\label{sec:GW}

\subsection{Definition and main result}

Many powerful manifold learning techniques utilize a graph formed on the data, and proceed to find an embedding into low dimensions based on this. For example, Isomap \cite{tenenbaum2000global} utilizes the graph shortest paths to approximate geodesics on the unknown manifold, while Diffusion Maps \cite{coifman2006diffusion} uses the graph Laplacian and diffusion coordinates to embed. For Isomap, \cite{bernstein2000graph} show that as the sampling density of the manifold data increases, the graph shortest paths of a random geometric graph, also known as a $\eps$-neighborhood graph, converge to the manifold geodesics under smoothness assumptions, which produces an approximation of the idealized embedding if one knew the manifold geodesics \textit{a priori}.  Coifman and Lafon show that the graph Laplacian converges to the Laplace--Beltrami operator on the data manifold as the sampling density increases.

The analogue of Isomap for data measures in Wasserstein space, Wassmap, was studied by Hamm et al.~and Cloninger et al.~\cite{Hamm2023Wassmap,cloninger2023linearized}, (see also \cite{wang2010optimal} which uses Multidimensional Scaling (MDS) on Wasserstein distances, and related work of Negrini and Nurbekyan on no-collision distances in MDS embeddings \cite{negrini2023applications}). However, these works did not consider consistency results of the kind known for Isomap or Diffusion Maps.  This section is devoted to proving consistency of the graph estimation of the $\W_\Lambda$ distance for $N$-term samples of $W_\Lambda$, where local distances are measured in $\W$ instead of $\W_\Lambda$.
The main result of this section shows that under the assumptions on the parameter manifold made above, if parameters are sampled i.i.d.~from the parameter manifold (given a suitable measure on $\mfold$), then the corresponding graph in $\Lambda$ converges as a metric space to $(\Lambda,\W_\Lambda)$ in the $\infty$-Gromov--Wasserstein sense (and therefore also in the Gromov--Hausdorff sense); see Theorem \ref{thm:Graph:GH}. This theorem is in a similar spirit to graph consistency results in the Euclidean setting. 
Indeed, we will assume that the parameter space is a manifold in a Euclidean space and hence many well-understood properties in the Euclidean case, e.g. graph connectivity, carry through to the Wasserstein submanifold.
We refer to~\cite{penrose2003random} for a thorough treatment of random geometric graphs in Euclidean spaces. 

\begin{definition}[Construction of graph metric measure space]
Let $\mfold$ be the parameter manifold satisfying Assumptions \ref{asp:Main}, and for simplicity, additionally assume that it is without boundary. Let $\mfoldprob$ be a probability measure on $\mfold$ that has a continuous density (with respect to the volume measure on $\mfold$) that is bounded away from infinity and zero.
Let $\mfold^N \assign \{\theta_i\}_{i=1}^N$ be a set of $N$ iid samples from $\mfoldprob$, let $\mfoldprobN \assign \frac{1}{N}\sum_{i=1}^N \delta_{\theta_i}$ be the corresponding empirical measure.
Further, let $\Lambda^N \assign \{\lambda_i\}_{i=1}^N$ be the embedded measures in Wasserstein space, where for convenience we shorten $\lambda_i \assign \lambda_{\theta_i} \assign \emb(\theta_i)$. 
Finally, let $\bbP_{\Lambda} \assign \emb_\# \mfoldprob$ and $\bbP_{\Lambda^N} \assign \emb_\# \mfoldprobN=\tfrac1N \sum_{i=1}^N \delta_{\lambda_i}$.

For a length scale parameter $\veps_N>0$ we then construct a graph with nodes $\Lambda^N$ and edges with lengths
\begin{equation}
\label{eq:EdgeLengths}
\ell_{ij} \assign \left\{ \begin{array}{ll} \W(\lambda_i,\lambda_j) & \text{if } \W(\lambda_i,\lambda_j)\leq \eps_N \\ + \infty & \text{else} \end{array} \right.
\end{equation}
between $\lambda_i$ and $\lambda_j$ for all $i,j\in \{1,\dots,N\}$ (here infinite edge lengths can be interpreted as there being no edge between the nodes). Note that this construction of the $\eps_N$-neighborhood graph is more common in manifold learning, while in other applications such as clustering, one often chooses edge weights via some decreasing function representing similarity, (e.g., Gaussian similarity; see for example \cite{von2007tutorial}).

This induces a discrete (shortest path) metric on the graph defined by
\begin{equation}
\label{eq:DiscreteShortestPath}
\W_{\Lambda^N}(\lambda_i,\lambda_j) \assign \inf\left\{ \sum_{k=0}^{K-1} \ell_{i_k,i_{k+1}} \middle| K\in\N, i_0=i, i_K=j, i_k\in\{1,\dots,N\} \right\}.
\end{equation}
The shortest path between any two nodes on the graph can be computed, for example, by Dijkstra's algorithm \cite{dijkstra1959note}.
\end{definition}

For this discrete construction we can then obtain the following approximation result.
As $N\to\infty$ we show (under some technical assumptions) that the discrete graph approximation $(\Lambda^N,\W_{\Lambda^N},\bbP_{\Lambda^N})$ recovers the continuum space $(\Lambda,\W_{\Lambda},\bbP_{\Lambda})$ in the sense of $\infty$-Gromov--Wasserstein convergence, hence Gromov--Haussdorff convergence (see Section~\ref{subsec:GH:Def} for the definition).
The proof of the theorem is given in Section~\ref{subsec:GH:Proof}.

\begin{theorem}
\label{thm:Graph:GH}
Let $m = \dim(\mfold) \geq 2$, $(\eps_N)_{N \in \N}$ be a positive sequence satisfying
\begin{align}
\label{eq:GHAspEps}
\eps_N  & \to 0, & \frac{(\eps_N)^m \cdot N}{\log N} & \to \infty
\end{align}
and consider the sequence of discrete metric measure spaces $(\Lambda^N,\W_{\Lambda^N},\bbP_{\Lambda^N})$ as defined above.
Then, with probability one there exists $N_0<+\infty$ such that for all $N\geq N_0$
\[ d_{\GW}((\Lambda^N,\W_{\Lambda^N},\bbP_{\Lambda^N}),(\Lambda,\W_{\Lambda},\bbP_\Lambda)) \hspace{-1pt} \leq \hspace{-2pt} \lb \hspace{-7pt} \begin{array}{ll} C \lp \hspace{-1pt} \lp\frac{\log N}{N}\rp^{\frac{1}{m}} \hspace{-2pt} \frac{1}{\eps_N} \hspace{-1pt} + \hspace{-1pt} \eps_N \rp & \hspace{-6pt} \text{for } m\geq 3 \\ C \lp \hspace{-1pt} \lp\frac{\log N}{N}\rp^{\frac{1}{2}} \hspace{-2pt} \frac{1}{\eps_N} \hspace{-1pt} + \hspace{-1pt} \frac{(\log N)^{\frac34}}{N^\frac12} \hspace{-1pt} + \hspace{-1pt} \eps_N \rp & \hspace{-6pt} \text{for } m=2 \end{array} \rd \]
where $d_{\GW}$ is the $\infty$-Gromov--Wasserstein distance (see Definition~\ref{def:GH:Def:GWDef}).
\end{theorem}

\begin{remark}
Including the case $m=1$ is possible but with possibly a (slightly) more restrictive lower bound on $\eps_N$ than the one in~\eqref{eq:GHAspEps}.
For $m\geq 3$ we match between the discrete space $\Lambda^N$ and the continuum space $\Lambda$ using transport maps which we can bound in the worst case using~\cite{trillos2015rate}.
More precisely, in our analysis we use the $\infty$-Wasserstein distance between the empirical measure, $\bbP_{\Lambda^N}$, on the parameter manifold and the data generating measure $\bbP_{\Lambda}$. 
When $m\geq 3$ we use the almost sure bounds in~\cite{trillos2015rate} to infer that $\W_\infty(\bbP_{\Lambda^N},\bbP_{\Lambda}) \lesssim \left(\log(N)/N\right)^{\frac{1}{m}}$.
When $m=2$ the rate obtained in~\cite{trillos2015rate} includes an extra logarithmic factor, i.e. (almost surely) $\W_\infty(\bbP_{\Lambda^N},\bbP_{\Lambda}) \lesssim \left(\log(N)/N\right)^{\frac{1}{2}} \times (\log(N))^{\frac{1}{4}}$.
This extra logarithmic factor can be removed using the techniques in~\cite{calder22,caroccia2020mumford}.
The case $m=1$ is not covered in any of these papers, however the one-dimensional case is simpler and one can derive $\W_\infty$ convergence rates directly from the law of the iterated logarithm to give a (almost sure) rate $\W_\infty(\bbP_{\Lambda^N},\bbP_{\Lambda}) \lesssim \left(\log(\log(N))/N\right)^{\frac{1}{2}}$.
This would mean we would, instead of~\eqref{eq:GHAspEps}, require $\frac{\eps^2 N}{\log \log N}\to \infty$ for the theorem to hold for $m=1$.
We do not know whether the techniques in~\cite{calder22,caroccia2020mumford} can be extended to reduce this bound.
\end{remark}

\subsection{Gromov--Hausdorff and Gromov--Wasserstein distances} \label{subsec:GH:Def}

In this section we briefly state the definitions for the Gromov--Hausdorff and Gromov--Wasser\-stein distances used in \Cref{thm:Graph:GH}. For an in-depth introduction we refer to \cite{memoli-gromov-shape-11,SturmMMGeometryI,SturmMMGeometryII}.

\begin{definition}[Correspondence between metric spaces]
Let
\[ (\cX,d_{\cX}) \quad \text{and} \quad (\cY,d_{\cY}) \]
be two metric spaces.
A correspondence $\cR$ between $\cX$ and $\cY$ is a subset of $\cX\times\cY$ satisfying (1) for all $x\in\cX$ there exists $y\in\cY$ such that $(x,y)\in\cR$ and (2) for all $y\in\cY$ there exists $x\in\cX$ such that $(x,y)\in\cR$.
\end{definition}

\begin{definition}[Distortion of a correspondence]
Let $(\cX,d_{\cX})$ and $(\cY,d_{\cY})$ be two metric spaces, and $\cR$ a correspondence between $\cX$ and $\cY$.
Then the distortion of $\cR$ is defined by
\[ \dis \cR = \sup_{\substack{(x,y)\in\cR\\(x^\prime,y^\prime)\in\cR}} \left| d_{\cX}(x,x^\prime) - d_{\cY}(y,y^\prime) \right|. \]
\end{definition}

\begin{definition}[Gromov--Hausdorff distance]
The $\infty$-Gromov--Hausdorff distance is defined as
\[ d_{\GH}((\cX,d_{\cX}),(\cY,d_{\cY})) = \frac12 \inf_{\cR} \dis \cR \]
where the infimum is taken over all correspondences $\cR$ between $\cX$ and $\cY$.
\end{definition}

The $\infty$-Gromov--Wasserstein distance is essentially the Gromov--Haussdorff distance with the restriction that the correspondences should be mass preserving.

\begin{definition}[$\infty$-Gromov--Wasserstein distance] \label{def:GH:Def:GWDef}
Let
\[ (\cX,d_{\cX},\bbP_{\cX}) \quad \text{and} \quad (\cY,d_{\cY},\bbP_{\cY}) \]
be two metric probability spaces.
Then the $\infty$-Gromov--Wasserstein distance is defined as
\[ d_{\GW}((\cX,d_{\cX},\bbP_{\cY}),(\cY,d_{\cY},\bbP_{\cY})) = \inf_{\pi\in\Pi(\bbP_{\cX},\bbP_{\cY})} \sup_{\substack{(x,y)\in \spt(\pi) \\ (x^\prime,y^\prime)\in\spt(\pi)}} \left|d_{\cX}(x,x^\prime) - d_{\cY}(y,y^\prime)\right|.\]
\end{definition}

If $\bbP_{\cX}$ and $\bbP_{\cY}$ both have full support, then for any $\pi\in\Pi(\bbP_{\cX},\bbP_{\cY})$ the set $\cR=\spt(\pi)$ defines a correspondence.
Hence, the $\infty$-Gromov--Wasserstein distance provides an upper bound on the Gromov--Hausdorff distance.
In particular,
\[ d_{\GH}((\cX,d_{\cX}),(\cY,d_{\cY})) \leq \frac12 d_{\GW}((\cX,d_{\cX},\bbP_{\cX}),(\cY,d_{\cY},\bbP_{\cY})). \]
More generally, one may also introduce $p$-Gromov--Wasserstein distances, see the references above for details.

\subsection{Proof of Theorem~\ref{thm:Graph:GH}} \label{subsec:GH:Proof}

Denote the $\infty$-Wasserstein distance on $\prob(\mfold)$ by $\W_\infty$. For $\mu,\nu \in \prob(\mfold)$ it is given by
$$\W_\infty(\mu,\nu) \assign \inf_{\pi \in \Pi(\mu,\nu)} \sup_{(x,y) \in \spt(\pi)} d_\mfold(x,y).$$
Using the assumption that $\mfoldprob$ has a density and full support, by~\cite[Theorem 2]{trillos2015rate}, there exists a minimizing transport plan of the form $\pi=(\id,S_N)_\# \mfoldprob$ for a map $S_N:\mfold\to\mfold^N$ that satisfies $S_{N\#}\mfoldprob = \mfoldprobN$, $\W_\infty(\mfoldprob,\mfoldprobN) = \sup_{\theta\in\mfold} d_{\mfold}(S_N(\theta),\theta)$, and for any $\beta>1$ there exists $C_1,C_2$ such that
\begin{equation} \label{eq:Graph:GH:Winfty}
\W_\infty(\mfoldprob,\mfoldprobN) \leq \left\{ \begin{array}{ll} C_1 \left(\frac{\log N}{N}\right)^{\frac{1}{m}} & \text{if } m\geq 3 \\ C_1\frac{(\log N)^{\frac34}}{N^{\frac12}} & \text{if } m=2 \end{array} \right. 
\end{equation}
with probability at least $1-C_2 N^{-\beta}$ and where $C_1$ depends only on $\beta$, the manifold $\mfold$ and the density of $\mfoldprob$, and $C_2$ depends only on the manifold $\mfold$.
We start by considering the case when $m\geq 3$ and we later return to the case where $m=2$.

If we fix $\beta>1$ (and therefore fix $C_1, C_2$ in~\eqref{eq:Graph:GH:Winfty}) then $\sum_{N=1}^\infty N^{-\beta}<+\infty$ and so by the Borel-Cantelli lemma the event that the bound in~\eqref{eq:Graph:GH:Winfty} does not hold happens only finitely many times with probability one.
In particular, there exists $N_0<\infty$ such that for all $N\geq N_0$ the bound~\eqref{eq:Graph:GH:Winfty} holds (where $N_0<+\infty$ with probability one).
In the sequel we assume $N\geq N_0$ (and the value of $N_0$ may increase).
Let $\delta_N = \frac{4L \W_\infty(\mfoldprob,\mfoldprobN)}{\eps_N} (\to 0)$.

Continuous parameters $\theta$ are mapped by $S_N$ to a nearby discrete sample $\theta_i$. In the same vein, the map $\hat{S}_N \assign \emb \circ S_N \circ \emb^{-1}$ take measures $\lambda \in \Lambda$ to the corresponding nearby $\lambda_i \in \Lambda^N$.
We define the measure $\pi_N\in\cP(\Lambda\times\Lambda^N)$ by
\[ \pi_N = (\emb, \emb)_{\#} (\id,S_N)_{\#}\mfoldprob = (\id,\hat{S}_N)_\# \emb_{\#} \mfoldprob = (\id,\hat{S}_N)_\# \bbP_{\Lambda}\]
which implies $\pi_N\in\Pi(\bbP_{\Lambda},\bbP_{\Lambda^N})$ since $\hat{S}_{N\#} \bbP_{\Lambda} = \emb_\# S_{N\#} \mfoldprob=\emb_\# \mfoldprobN =\bbP_{\Lambda^N}$.

Note that we have $\lambda'=\hat{S}_N(\lambda)$ for $\pi_N$-almost all $(\lambda,\lambda')$, and $\lambda' \neq \hat{S}_N(\lambda)$ outside of $\spt \pi_N$. Therefore,
\begin{equation}
\label{eq:SptPiN}
\spt \pi_N = \ol{\left\{ (\lambda,\hat{S}_N(\lambda)) | \lambda \in \Lambda \right\}}
\end{equation}
and thus we obtain the bound
\begin{align}
d_{\GW}((\Lambda,\W_{\Lambda},\bbP_{\Lambda}),(\Lambda^N,\W_{\Lambda^N},\bbP_{\Lambda^N})) & \leq
\sup_{\substack{(\lambda_0,\lambda_0') \in \spt \pi_N,\\(\lambda_1,\lambda_1') \in \spt \pi_N}}
|\W_{\Lambda}(\lambda_0,\lambda_1)-\W_{\Lambda^N}(\lambda_0',\lambda_1')| \nonumber \\
& = \hspace{-3.5pt} \sup_{(\lambda_0,\lambda_1) \in \Lambda} |\W_{\Lambda}(\lambda_0,\lambda_1)-\W_{\Lambda^N}(\hat{S}_N(\lambda_0),\hat{S}_N(\lambda_1))| \label{eq:SimpleGWBound}
\end{align}
where the second line is due to \eqref{eq:SptPiN} and continuity of $\W_\Lambda$ and $\W_{\Lambda^N}$.
We must therefore show that the map $\hat{S}_N$ is an approximate isometric embedding of $(\Lambda,\W_{\Lambda})$ into $(\Lambda^N,\W_{\Lambda^N})$.
Let now $\mu_0,\mu_1 \in \Lambda$, set $\nu_0 \assign \hat{S}_N(\mu_0)$, $\nu_1 \assign \hat{S}_N(\mu_1)$.
We need to control $|\W_\Lambda(\mu_0,\mu_1)-\W_{\Lambda^N}(\nu_0,\nu_1)|$.

\paragraph{Upper bound on $\W_{\Lambda^N}$ by $\W_\Lambda$}
Let now  $\mu \in \Lip([0,1];(\Lambda,\W_\Lambda))$ be a constant speed geodesic between $\mu(0)=\mu_0$ and $\mu(1)=\mu_1$, existence guaranteed by \Cref{prop:SubMan:ParMan:WLambda}.
Then
\[ \len(\mu) \assign \sqrt{\energy(\mu)} = \W_\Lambda(\mu_0,\mu_1) \qquad \text{and} \qquad \W_{\Lambda}(\mu(s),\mu(t)) = |t-s| \W_{\Lambda}(\mu_0,\mu_1) \]
for $s,t \in [0,1]$.
Let $K_N \in \N$ such that $K_N \in \left[\frac{C\diam(\W_\Lambda)^2}{\eps_N},\frac{C\diam(\W_\Lambda)^2}{\eps_N}+1\right]$ where $\diam(\W_\Lambda)<\infty$ by \Cref{prop:SubMan:ParMan:WLambda}.
Let $T_N \assign \{ t_i \assign i/K_N | i=0,\ldots,K_N\}$ be a set of equidistant discrete times and we introduce the time-discrete curve
\begin{align*}
	\nu & : T_N \to \Lambda^N, & t_i & \mapsto \hat{S}_N(\mu(t_i)).
\end{align*}
Note that $\nu(0)=\nu_0$ and $\nu(1)=\nu_1$.
One can think of $\nu$ as a discrete approximation on $\Lambda^N$ of the geodesic $\mu$ on $\Lambda$ (see Figure \ref{fig:GH:Proof:GeoConstruction} for illustration).
Our task is now to bound the length of $\nu$ by that of $\mu$. 

\begin{figure}[ht]
\centering
\includegraphics[width=.9\textwidth]{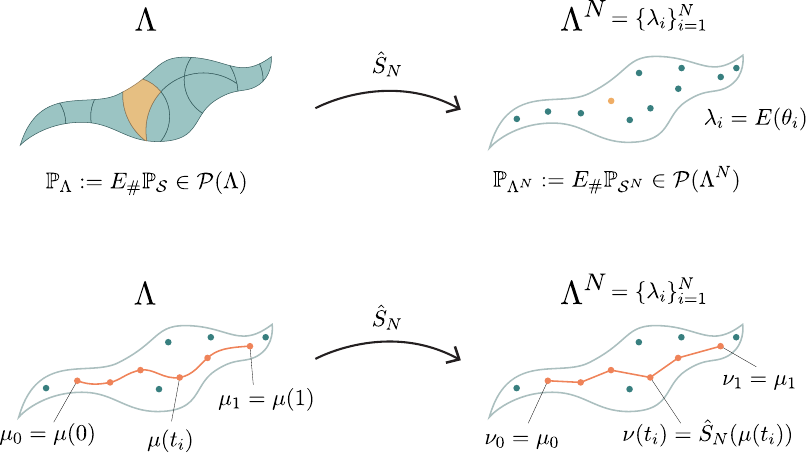}
\caption{Description of the pieces of the proof of Gromov--Wasserstein convergence. (Top left) Partitions of $\Lambda$ that get mapped to discrete points in $\Lambda^N$ via $\hat S_N$. (Top Right) Samples $\lambda_i$ of $\Lambda^N$.  (Bottom left) Continuous path from $\mu_0$ to $\mu_1$ in $\Lambda$. (Bottom right) Piecewise chordal path in $\Lambda^N$ from $\nu_0$ to $\nu_1$ passing through the samples $\lambda_i$. }
\label{fig:GH:Proof:GeoConstruction}
\end{figure}

Assuming that $C\diam(\W_\Lambda)\geq 2$ (after increasing $C$ if necessary) we note that
\begin{equation}
\label{eq:WStepBound}
\W(\mu(t_i),\mu(t_{i+1})) \leq \W_\Lambda(\mu(t_i),\mu(t_{i+1})) = \frac{\W_\Lambda(\mu_0,\mu_1)}{K_N} \leq \frac{\eps_N}{2}.
\end{equation}
With
\begin{align}
\W(\nu(t_k),\mu(t_k)) & \leq L \cdot d_{\mfold}(\emb^{-1}(\nu(t_k)),\emb^{-1}(\mu(t_k))) && \text{by Lipschitz continuity of } \emb \notag \\
 & = L \cdot d_{\mfold}(S_N(\emb^{-1}(\mu(t_k))),\emb^{-1}(\mu(t_k)) && \text{by construction} \notag \\
 & \leq L \sup_{\eta\in\mfold} d_{\mfold}(S_N(\eta),\eta) && \text{bounding uniformly} \notag \\
 & = \frac{\eps_N\delta_N}{4}, && \label{eq:Graph:PfGH:muzetaBound}
\end{align}
and \eqref{eq:WStepBound} we get
\[ \W(\nu(t_k),\nu(t_{k+1})) \leq \W(\nu(t_k),\mu(t_k)) + \W(\mu(t_k),\mu(t_{k+1})) + \W(\mu(t_{k+1}),\nu(t_{k+1})) \leq \eps_N. \]
Therefore, the edge lengths \eqref{eq:EdgeLengths} between any two subsequent $(\nu(t_k),\nu(t_{k+1}))$ in $\Lambda^N$ are finite and so we have
\begin{align*}
\W_{\Lambda^N}(\nu_0,\nu_1) & \leq \sum_{k=0}^{K_{N}-1} \W(\nu(t_k),\nu(t_{k+1})) \\
 & \leq \sum_{k=0}^{K_{N}-1} \Bigg( \W(\nu(t_k),\mu(t_k)) + \W(\mu(t_k),\mu(t_{k+1})) + \W(\mu(t_{k+1}),\nu(t_{k+1})) \Bigg).
\end{align*}
Combining this with \eqref{eq:Graph:PfGH:muzetaBound} we get the upper bound
\begin{equation} \label{eq:Graph:PfGH:Bound1}
\begin{aligned}
\W_{\Lambda^N}(\nu_0,\nu_1) & \leq \frac{K_{N}\eps_N\delta}{2} + \sum_{k=0}^{K_{N}-1} \W(\mu(t_k),\mu(t_{k+1})) \\
 & \leq \frac{(C \diam(\W_\Lambda)^2+\eps_N)\delta_N}{2} + \W_{\Lambda}(\mu_0,\mu_1).
\end{aligned}
\end{equation}
In particular, this estimate implies that there is some constant $C<\infty$ such that for $N \geq N_0$, with probability one, one has
\begin{equation}
\label{eq:UniformWLambdaNBound}
\diam (\W_{\Lambda^N})< C.
\end{equation}
\paragraph{Lower bound on $\W_{\Lambda^N}$ by $\W_\Lambda$}
Now consider the shortest discrete path between $\nu_0$ and $\nu_1$, as defined in \eqref{eq:DiscreteShortestPath}.
Let $K_N \in \N$, $T_N \assign \{ t_i \assign i/K_N | i=0,\ldots,K_N\}$ and $\nu : T_N \mapsto \Lambda^N$ be such that $\nu$ describes a shortest path in $\W_{\Lambda^N}$ from $\nu_0$ to $\nu_1$. That is,
$\nu(0) = \nu_0$, $\nu(1)=\nu_1$, $\W(\nu(t_k),\nu(t_{k+1}))\leq \eps_N$ for all $k=0,\dots, K_N-1$,  and
\[ \W_{\Lambda^N}(\nu_0,\nu_1) = \sum_{k=0}^{K_N-1} \W(\nu(t_k),\nu(t_{k+1})). \]
Such a path with finite length exists due to \eqref{eq:UniformWLambdaNBound}.

We now need to `interpolate' this to a time-continuous path $\mu \in \CC([0,1],\Lambda)$ between $\mu_0$ and $\mu_1$ in $\Lambda$. For this, we first fix the discrete time points
\begin{align*}
	\mu(0) & \assign \mu_0, &
	\mu(1) & \assign \mu_1, &
	\mu(t_k) & \assign \nu(t_k) \tn{ for } k \in \{1,\ldots,K_N-1\}.
\end{align*}
Then, for any $i \in \{0,\ldots,K_N-1\}$ we fill in the gap by setting the segment $\mu \restr_{[t_i,t_{i+1}]}$ to be a constant speed shortest path in $(\Lambda,\W_{\Lambda})$ from $\mu(t_i)$ to $\mu(t_{i+1})$.

Then by construction, for $k=1,\dots, K_N-2$,
\[ \W(\mu(t_k),\mu(t_{k+1})) = \W(\nu(t_k),\nu(t_{k+1})), \]
for $k=0$, using arguments similar to \eqref{eq:Graph:PfGH:muzetaBound},
\begin{align*}
\W(\mu(t_0),\mu(t_1)) & = \W(\mu_0,\nu(t_1)) \\
 & \leq \W(\mu_0,\nu_0) + \W(\nu_0,\nu(t_1)) \\
 & \leq L \cdot d_{\mfold}(\emb^{-1}(\mu_0),S_N(\emb^{-1}(\mu_0))) + \W(\nu(t_0),\nu(t_1)) \\
 & \leq \eps_N + \W(\nu(t_0),\nu(t_1)),
\end{align*}
and, similarly for $k=K_N-1$,
\[ \W(\mu(t_{K_N-1}),\mu(t_{K_N})) \leq \eps_N + \W(\nu(t_{K_N-1}),\nu(t_{K_N})). \]
Combining these, and using Proposition~\ref{prop:WLambdaWComparison},
\begin{align*}
\W_{\Lambda}(\mu_0,\mu_1) & \leq \sum_{k=0}^{K_N-1} \W_\Lambda(\mu(t_k),\mu(t_{k+1})) \\
 & \leq \sum_{k=0}^{K_N-1} \lp \W(\mu(t_k),\mu(t_{k+1})) + C \cdot \W(\mu(t_k),\mu(t_{k+1}))^2 \rp \\
 & \leq \lp \sum_{k=0}^{K_N-1} \W(\nu(t_k),\nu(t_{k+1})) \rp + 2\eps_N \\
 & \qquad \qquad
   + C \cdot \sum_{k=0}^{K_N-1} \big(
   \underbrace{\W(\nu(t_k),\nu(t_{k+1})}_{\leq \eps_N} + \eps_N\one_{k=0} +\eps_N\one_{k=K_N} \big)^2  \\
 & \leq \W_{\Lambda^N}(\nu_0,\nu_1) + 2\eps_N \\
 & \qquad \qquad + 3 \eps_N C \sum_{k=0}^{K_N-1} \lp \W(\nu(t_k),\nu(t_{k+1}) + \eps_N\one_{k=0} +\eps_N\one_{k=K_N} \rp \\
 & \leq \W_{\Lambda^N}(\nu_0,\nu_1) + 2\eps_N + 3 \eps_N C \cdot \lp \W_{\Lambda^N}(\nu_0,\nu_1) + 2\eps_N \rp.
\end{align*}
By \eqref{eq:UniformWLambdaNBound}, there is a constant $C$ such that
\begin{equation} \label{eq:Graph:PfGH:Bound2}
\W_\Lambda(\mu_0,\mu_1)  \leq \W_{\Lambda^N}(\nu_0,\nu_1) + \eps_N \cdot C.
\end{equation}

\paragraph{Combining upper and lower bounds}
Combining~\eqref{eq:Graph:PfGH:Bound1} and~\eqref{eq:Graph:PfGH:Bound2} we have with probability 1, for $N$ sufficiently large,
\[ -C\eps_N \leq \W_{\Lambda^N}(\nu_0,\nu_1) - \W_{\Lambda}(\mu_0,\mu_1) \leq C\delta_N. \]
And so by \eqref{eq:SimpleGWBound},
\[ d_{\GW}((\Lambda,\W_{\Lambda},\bbP_{\Lambda}),(\Lambda^N,\W_{\Lambda^N},\bbP_{\Lambda^N})) \leq C (\delta_N + \eps_N) \]
and by~\eqref{eq:Graph:GH:Winfty}
\[ d_{\GW}((\Lambda,\W_{\Lambda},\bbP_{\Lambda}),(\Lambda^N,\W_{\Lambda^N},\bbP_{\Lambda^N})) \leq C \lp \lp\frac{\log N}{N}\rp^{\frac{1}{m}} \frac{1}{\eps_N} + \eps_N\rp. \]

\paragraph{Case $m=2$}
Finally we are left to deal with the case when $m=2$.
By~\cite[Proposition 2.11]{calder22} there exists $\mfoldprobhatN\in\cP(\mfold)$, with a continuous density $\widetilde{p}_N$ with respect to the volume measure on $\mfold$, and positive constants $C$, $c$, $\eps_0$ and $\omega_0$ such that if $N^{-\frac{1}{m}}\leq \eps \leq \eps_0$ and $\omega\leq\omega_0$ then
\[ \W_\infty(\mfoldprobN,\mfoldprobhatN) \leq \eps
\qquad \text{and} \qquad
\| p - \widetilde{p}_N \|_{\infty} \leq C(\omega+\eps), \]
with probability at least $1-2Ne^{-cN\omega^2\eps^m}$. Here $p$ is the continuous density of $\mfoldprob$ with respect to the volume measure on $\mfold$.

For $\beta>2$ we choose $\omega_N = \sqrt{\frac{(1+\beta)\log N}{cN\eps_N^m}} (\to 0)$, so the above probability bound simplifies to $1-2Ne^{-cN\omega_N^2\eps_N^m} = 1-2N^{-\beta}$.
By the Borel--Cantelli lemma, with probability one, there exists $N_0<\infty$ such that
\[ \W_\infty(\mfoldprobN,\mfoldprobhatN) \leq \eps_N \qquad \text{and} \qquad \| p - \widetilde{p}_N \|_{\infty} \leq C(\omega_N+\eps_N)\to 0. \]

The proof now proceeds as in the case $m\geq 3$, with the modification that the transport map $S_N$ between $\mfoldprob$ and $\mfoldprobN$ is replaced by the transport map $\widetilde{S}_N$ between $\mfoldprobhatN$ and $\mfoldprobN$, i.e. $\widetilde{S}_{N\#}\mfoldprobhatN = \mfoldprobN$ and $\W_\infty(\mfoldprobhatN,\mfoldprobN) = \sup_{\theta\in\mfold} d_{\mfold}(\widetilde{S}_N(\theta),\theta)$, to show that
\[ d_{\GW}((\Lambda,\W_\Lambda,\widetilde{\bbP}_{\Lambda,N}),(\Lambda^N,\W_{\Lambda^N},\bbP_{\Lambda^N})) \leq C\lp \lp \frac{\log N}{N}\rp^{\frac{1}{m}} \frac{1}{\eps_N} + \eps_N \rp \]
where $\widetilde{\bbP}_{\Lambda,N} = \emb_{\#} \mfoldprobhatN$.
We are left to show that
\[ d_{\GW}((\Lambda,\W_{\Lambda},\bbP_{\Lambda}),(\Lambda,\W_{\Lambda},\widetilde{\bbP}_{\Lambda,N})) \leq C\lp \eps_N + \frac{(\log N)^{\frac34}}{N^{\frac12}} \rp. \]

Since
\[ \W_\infty(\mfoldprobhatN,\mfoldprob) \leq \W_\infty(\mfoldprobhatN,\mfoldprobN) + \W_\infty(\mfoldprobN,\mfoldprob) \leq \eps_N + C_1 \frac{(\log N)^{\frac{3}{4}}}{N^{\frac12}} \]
(using~\eqref{eq:Graph:GH:Winfty}) then we can find a transport map $R_N:\Lambda\to\Lambda$ such that $R_{N\#}\mfoldprob = \mfoldprobhatN$ and $\sup_{\theta\in\mfold} d_{\mfold}(R_N(\theta),\theta) = \W_\infty(\mfoldprobhatN,\mfoldprob)$.
We define the transport plan $\widetilde{\pi}_N = (\emb,\emb)_{\#}(\id, R_N)_{\#}\mfoldprob$ with $\widetilde{\pi}_N\in\Pi(\bbP_\Lambda,\widetilde{\bbP}_{\Lambda,N})$ and $\widetilde{\cR}_N = \spt(\widetilde{\pi}_N)$ defines a correspondence following the same argument as above.
We again use the trivial bound
\[ d_{\GW}((\Lambda,\W_\Lambda,\bbP_\Lambda),(\Lambda,\W_\Lambda,\widetilde{\bbP}_{\Lambda,N})) \leq \dis \widetilde{\cR}_N \]
and show $\dis\widetilde{\cR}_N\to 0$.

Let $(\mu_i,\widetilde{\mu}_i)\in\widetilde{\cR}_N$, $i=1,2$.
Then there exists $\eta_i,\widetilde{\eta}_i$ such that $\mu_i = \emb(\eta_i)$, $\widetilde{\mu}_i= \emb(\widetilde{\eta}_i)$ and $\widetilde{\eta}_i = R_N(\eta_i)$, $i=1,2$.
Hence,
\begin{align*}
\W_\Lambda(\widetilde{\mu}_1,\widetilde{\mu}_2) & = \W_\Lambda(\emb(R_N(\eta_1)),\emb(R_N(\eta_2))) \\
 & \leq \W_\Lambda(\emb(R_N(\eta_1)),\emb(\eta_1)) + \W_\Lambda(\emb(\eta_1),\emb(\eta_2)) + \W_\Lambda(\emb(\eta_2),\emb(R_N(\eta_2))) \\
 & \leq L d_{\mfold}(R_N(\eta_1),\eta_1) + \W_\Lambda(\mu_1,\mu_2) + L d_{\mfold}(\eta_2,R_N(\eta_2)) \\
 & \leq 2L\W_\infty(\mfoldprobhatN,\mfoldprob) + \W_\Lambda(\mu_1,\mu_2).
\end{align*}
So, $\W_\Lambda(\widetilde{\mu}_1,\widetilde{\mu}_2) - \W_\Lambda(\mu_1,\mu_2) \leq 2L\W_\infty(\mfoldprobhatN,\mfoldprob)$.
Similarly,
\[ \W_\Lambda(\widetilde{\mu}_1,\widetilde{\mu}_2) - \W_\Lambda(\mu_1,\mu_2) \geq -2L\W_\infty(\mfoldprobhatN,\mfoldprob). \]
Putting the bounds together we have
\begin{align*}
\dis\widetilde{\cR}_N & = \sup_{(\mu_i,\widetilde{\mu}_i)\in\widetilde{\cR}_N} |\W_\Lambda(\mu_1,\mu_2) - \W_\Lambda(\widetilde{\mu}_1,\widetilde{\mu}_2)| \\
 & \leq 2L \W_\infty(\mfoldprobhatN,\mfoldprob) \\
 & \leq C \lp \eps_N + \frac{(\log N)^{\frac34}}{N^{\frac12}} \rp.
\end{align*}
This completes the proof.

\section{Recovering the tangent space from samples}
\label{sec:TangentSpace}

In \Cref{prop:wvconvergence} we have shown that tangent vectors of curves in $(\Lambda,\W_\Lambda)$ can be recovered asymptotically by the logarithmic map of $(\probL(\Omega),\W)$ as the distance between the samples decreases. This implies, for instance, that the velocity fields that are assumed to exist in \Cref{asp:Main} can in principle be recovered from samples of the manifold $\Lambda$ and optimal transport maps with respect to $\W$ alone.
In this section we go one step further and show that the whole tangent space of $\Lambda$, including its dimensionality can be recovered from sufficiently close and diverse samples and their $\W$-transport maps.
The main result of this section is \Cref{thm:RecoverTanSpace}. We give relevant definitions and auxiliary results first.

Throughout this section let $\mc{H}$ be an (infinite-dimensional) Hilbert space.
In this section we use some results from finite-dimensional linear algebra on finite-rank operators on $\mc{H}$. This is possible by the following lemma, which we will state without proof.
\begin{lemma}[Basic properties of finite-rank operators]
\label{lem:BasicFiniteRank}
Let $A : \mc{H} \to \mc{H}$ be a finite-rank linear operator. Let $U \subset \mc{H}$ be a finite-dimensional subspace of $\mc{H}$ with $\range(A),\range(A^*) \subset U$ and denote by $P_U : \mc{H} \to U$ the orthogonal projection from $H$ onto $U$ and its adjoint by $P_U^*$. Then the following hold:
\begin{enumerate}[(i)]
\item $A=P_U^* \hat{A} P_U$ where $\hat{A}=P_U A P_U^*$ is a linear operator on $U$. We call $\hat{A}$ the restriction of $A$ to $U$ and $\hat{A}$ can be represented by a finite-dimensional (Hermitian) matrix. \label{item:HUrestrict}
\item $\|A\|_{\HS(\mc{H})}=\|\hat{A}\|_{\HS(U)}$. Here $\|\cdot\|_{\HS(\mc{H})}$ denotes the Hilbert--Schmidt norm of a linear operator $\mc{H} \to \mc{H}$. On finite-dimensional spaces this is equal to the Frobenius norm of the matrix representation in any orthonormal basis of $U$. \label{item:HUnorms}
\item $A$ and $\hat{A}$ have the same non-zero eigenvalues (including their multiplicity). In particular, $A$ has at most $\dim U$ non-zero eigenvalues. \label{item:HUeigvals}
\item $v \in \mc{H}$ is an eigenvector of $A$ for an eigenvalue $\mu\neq 0$ if and only if there exists some $\hat{v} \in U$ that is an eigenvector of $\hat{A}$ with eigenvalue $\mu$ and $v=P_U^* \hat{v}$. In particular, all eigenvectors for non-zero eigenvalues of $A$ lie in $U$. \label{item:HUeigvecs}
\end{enumerate}
\end{lemma}

\begin{theorem}[Weyl's theorem for self-adjoint finite-rank operators]\label{thm:Weyl}
Let $A, B: \mc{H} \to \mc{H}$ be self-adjoint positive-semidefinite finite-rank operators.
Let $(\mu_i)_{i=1}^{\dim \mc{H}}$ and $(\tilde{\mu}_i)_{i=1}^{\dim \mc{H}}$ be the (non-negative) eigenvalues of $A$ and $B$ respectively, sorted in decreasing order. Then
\begin{equation*}
|\mu_i-\tilde{\mu_i}| \leq \|A-B\|_{\HS}, \quad \tn{for all } i \in \{1,\ldots,\dim \mc{H}\}.
\end{equation*}
\end{theorem}
\begin{remark}
Weyl's theorem applies to Hermitian matrices that are not necessarily positive-semidefinite. However, this restriction is sufficient for the applications in this article and simplifies the identification of the corresponding eigenvalue pairs (or at least the notation).
\end{remark}
\begin{proof}[Proof of Theorem~\ref{thm:Weyl}]
Let $U$ be a finite-dimensional subspace of $\mc{H}$ such that $\range A, \range B \subset U$ and let $\hat{A}$ and $\hat{B}$ be the restrictions of $A$ and $B$ to $U$ (see \Cref{lem:BasicFiniteRank}\eqref{item:HUrestrict}). Then by \Cref{lem:BasicFiniteRank}\eqref{item:HUeigvals} the (finite) lists of sorted eigenvalues of $\hat{A}$ and $\hat{B}$ agree with the first $\dim(U)$ entries of the lists $(\mu_i)_{i=1}^{\dim \mc{H}}$ and $(\tilde{\mu}_i)_{i=1}^{\dim \mc{H}}$. In particular, these sub-lists contain all non-zero eigenvalues. For these eigenvalues the bound then follows by Weyl's theorem \cite{weyl1912eigenvalues} and $\|A-B\|_{\HS(\mc{H})}=\|\hat{A}-\hat{B}\|_{\HS(U)}$, see \Cref{lem:BasicFiniteRank}\eqref{item:HUnorms}.
For all $i>\dim(U)$ where both eigenvalues are zero the bound holds trivially.
\end{proof}

\begin{definition}[Principal angles between subspaces and induced distance]
\label{def:DSubspace}
Let $V,W$ be finite-dimensional subspaces of $\mc{H}$.
\begin{enumerate}[(i)]
\item The first principal angle $\theta_1$ between $V$ and $W$ is defined by
\begin{align*}
    \theta_1 \assign \cos^{-1}\left(\max_{v\in V, \|v\|=1}\max_{w\in W, \|w\|=1} |\left<v,w\right>|\right),
\end{align*}
and let $(v_1,w_1)$ denote a maximizer. Note that a maximum exists since a continuous function is maximized over a compact set (since $V$ and $W$ are finite-dimensional).
\item For $k>1$, inductively define $V_k \assign \{v\in V: \|v\|=1, \left<v,v_{\ell}\right> = 0, \ell =1,\ldots, k-1\}$ (and define $W_k$ similarly). Then
\begin{equation*}
        \theta_k \assign \cos^{-1}\left(\max_{v\in V_k}\max_{w\in W_k} |\left<v,w\right>|\right),
\end{equation*}
and let $(v_k,w_k)$ denote a maximizer. 
\item When $V,W$ are both $k$-dimensional subspaces, then we define a distance between $V$ and $W$ as follows: Let $\Theta$ be the $(k\times k)$ diagonal matrix containing the $k$ principal angles between $V$ and $W$. Then
\begin{equation}
    d(V,W) \assign \| \sin \Theta\|_\fro = \sqrt{\sum_{i=1}^k \sin(\theta_i)^2},
\end{equation}
where $\sin$ is to be understood elementwise. This defines a distance, $d$, between the subspaces $V,W$, and is one of the classical distances on the Grassmanian $\operatorname{Gr}(k,d)$ \cite{Ye2016schubert,Hamm2008grassmann}.
\end{enumerate}
\end{definition}

\begin{theorem}[Davis-Kahan for self-adjoint finite-rank operators]\label{thm:DK-finite-rank}
Let $A, B: \mc{H} \to \mc{H}$ be self-adjoint finite-rank operators.
Let $\{v_i\}_{i=1}^m$ be the eigenvectors of $A$ for non-zero eigenvalues, ordered according to the corresponding non-zero eigenvalues $\mu_1\geq \ldots \geq \mu_m$ of $A$, and define $\{w_i\}_{i=1}^n$ in the same way for $B$.
For $k \leq \min\{m,n\}$, let $V_k=\vecspan \{v_i\}_{i=1}^k$ and $W_k=\vecspan \{w_i\}_{i=1}^k$.
Then
\begin{equation*}
    d(V_k,W_k) \leq 2\frac{\|A-B\|_{\HS}}{\mu_k-\mu_{k+1}}.
\end{equation*}
where $\mu_k$ is the $k$-th eigenvalue of $A$ (and $\mu_{k+1}$ is potentially zero).
\end{theorem}

\begin{proof}
Pick any finite-dimensional subspace $U \subset \mc{H}$ such that $\range(A)\subset U$, $\range(B) \subset U$, which implies $\range(A-B) \subset U$. In particular all $v_i$ and $w_i$ will lie in $U$ by \Cref{lem:BasicFiniteRank}\eqref{item:HUeigvecs} and are precisely the eigenvectors (for non-zero eigenvalues) of the restriction $P_U (A-B) P_U^*$. Also, the $\HS(U)$-norm of the restriction equals $\|A-B\|_{\HS(\mc{H})}$ by \Cref{lem:BasicFiniteRank}\eqref{item:HUnorms}. We then invoke \cite[Theorem 2]{yu2014daviskahan}. 
\end{proof}

The following operators are closely related to empirical covariance operators. However, we do not center the vectors here, as it will not be appropriate in our intended application.
\begin{definition}[Empirical span operators]
\label{def:SpanOperators}
Let $V=\{v_i\}_{i=1}^N \subset \mc{H}$ and $W=\{w_i\}_{i=1}^N \subset \mc{H}$.
\begin{enumerate}[(i)]
\item We define the operator $F(V,W): \mc{H} \to \mc{H}$ as
\begin{equation}
	\label{eq:EmpF}
    F(V,W) \assign \frac1N\sum_{i=1}^N w_i \left<v_i,\; \cdot\; \right>.
\end{equation}
Note that $F(V,W)$ is a finite-rank operator with $\range(F(V,W))\subseteq \vecspan W$. The operator $F$ may not be self-adjoint, but
$$S(V,W) \assign F(V,W)+F(W,V)$$
is self-adjoint (and symmetric in its arguments).
\item We also define the symmetric operator
$$F(V) \assign F(V,V),$$
which is finite-rank with $\range(F(V))\subseteq \vecspan V$ and self-adjoint. It is furthermore non-negative in the sense that $\left<F(V)z,z\right> \geq 0$ for $z\in \mc{H}$, implying that all of its eigenvalues are non-negative.
\end{enumerate}
\end{definition}

\begin{lemma}[Basic properties of empirical span operators]
\label{lem:SpanBasic}
Let $V=\{v_i\}_{i=1}^N \subset \mc{H}$, $W=\{w_i\}_{i=1}^N \subset \mc{H}$ and let $U$ be a finite-dimensional subspace of $\mc{H}$ with $V,W \subset U$ and $P_U : \mc{H} \to U$ being the orthogonal projection onto $U$ as in \Cref{lem:BasicFiniteRank}.
\begin{enumerate}[(i)]
    \item Then $P_U F(V,W) P_U^* = F(P_U V,P_U W)$ where $P_U V$ and  $P_U W$ are the projections of the sets $V, W$ to $U$, and $F(P_U V,P_U W)$ denotes the finite-dimensional operator \eqref{eq:EmpF} on $U$.
	This implies $\|F(V,W)\|_{\HS(\mc{H})} = \|F(P_U V,P_U W)\|_{\HS(U)}$. \label{item:EmpRestriction}
\item And $ \displaystyle\|F(V,W)\|_{\HS(\mc{H})} \leq \max_{i=1,\ldots, N} \|v_i\| \max_{i=1,\ldots, N} \|w_i\|$. \label{item:EmpNormBound}
\item Let $m=\dim \vecspan V$. Then $F(V)$ has exactly $m$ non-zero eigenvalues and there exists an orthonormal basis of $\vecspan V$ consisting of eigenvectors $F(V)$.
\label{item:EmpEigenbasis}
\end{enumerate}
\end{lemma}
\begin{proof}
Part \eqref{item:EmpRestriction}:
The equality 
\[ P_U F(V,W) P_U^* = F(P_U V,P_U W) \]
follows directly from the definition in \eqref{eq:EmpF}.
The equality concerning norms follows from $\|P_U F(V,W) P_U^*\|_{\HS(U)}=\|F(V,W)\|_{\HS(\mc{H})}$, as implied by \Cref{lem:BasicFiniteRank}\eqref{item:HUnorms}. Here we use $\vecspan V, \vecspan W \subset U$ to ensure that the range of $F(V,W)$ and its adjoint are contained in $U$.

Part \eqref{item:EmpNormBound}:
Note that the operator $F(P_UV,P_U W)$ is a sum of rank-$1$ operators,
\begin{equation*}
  F(P_UV,P_U W) = \frac1N\sum_{i=1}^N P_U w_i \left< P_U v_i,\;\cdot\;\right>.
\end{equation*}
Therefore,
\begin{align*}
  \|F(V,W)\|_{\HS(\mc{H})} & = \|F(P_UV,P_U W)\|_{\HS(U)} \leq \frac1N \sum_{i=1}^N \|P_U w_i \|_U \|P_U v_i\|_U \\
  & \leq \max_{i=1,\ldots,N}\|P_U w_i\|_U \max_{i=1,\ldots,N}\|P_U v_i\|_U,
 \end{align*}
Since $\|P_U z\|_U \leq \|z\|_{\mc{H}}$ for $z \in \mc{H}$, the bound follows (in fact, the latter is an equality when $z \in U$).

Part \eqref{item:EmpEigenbasis}:
Note that $\range(F(V)) \subset \vecspan V$. Therefore, by \Cref{lem:BasicFiniteRank}\eqref{item:HUeigvals}, $F(V)$ has at most $m$ non-zero eigenvalues.
Let $\{u_i\}_{i=1}^k$ be orthonormal eigenvectors for the $k$ (where $k \leq m$) non-zero eigenvalues $(\mu_i)_{i=1}^k$ of $F(V)$ (orthogonality provided by the fact that $F(V)$ is self-adjoint).
Then $F(V)=\sum_{i=1}^k u_i\mu_i\,\left<u_i,\cdot\right>$. If $k<m$ there would exist some $z \in U \setminus \{0\}$ with $z \perp u_i$ for $i=1,\ldots,k$ and thus $\langle z, F(V) z\rangle=0$, which would imply $z \perp v_i$ for all $i=1,\ldots,m$. This is a contradiction and therefore $k=m$.
\end{proof}

\begin{lemma}[Perturbation of span operators]\label{lem:span-perturb}
Let $V=\{v_i\}_{i=1}^N \subset \mathcal{H}$ and $W=\{w_i\}_{i=1}^N \subset \mathcal{H}$. Then
\begin{equation*}
    F(W) = F(V) + F(W-V,W-V) + S(V,W-V),
\end{equation*}
where $W-V = \{w_i-v_i\}_{i=1}^N$.
In particular,
\begin{equation}
\label{eq:span-perturb}
    \|F(V)-F(W)\|_{\HS(\mc{H})} \leq \max_{i=1,\ldots, N} \|v_i - w_i\|^2 + 2\max_{i=1,\ldots, N} \|v_i - w_i\|\max_{i=1,\ldots, N} \| v_i\|.
\end{equation}
\end{lemma}

\begin{proof}
With $Z=W-V$ we get
\begin{align*}
F(W) & = F(V+Z) \\
 & = \frac1N \sum_{i=1}^N (v_i+z_i)\,\left<v_i+z_i,\cdot \right> \\
 & = F(V) + F(Z) + F(Z,V)+F(V,Z).
\end{align*}
The bound \eqref{eq:span-perturb} then follows by applying \Cref{lem:SpanBasic}\eqref{item:EmpNormBound}.
\end{proof}

Finally, we state and prove the main theorem of this section.
\begin{theorem}[Recovery of tangent space from samples]\label{thm:RecoverTanSpace} \hspace{-5pt}
Let $\theta \in \mfold$ and $\{\eta_i\}_{i=1}^N \subset \domexp_\theta$.
For $t \in (0,1]$, $i=1,\ldots,N$, let $\theta_i^{t} \assign \exp_{\theta}(t\cdot \eta_i) \in \mfold$, $\lambda_{\theta_i^t} \assign \emb(\theta_i^{t})$, and $\lambda_{\theta}=\emb(\theta)$.
Furthermore, let $V=\{v_i\}_{i=1}^N$ with $v_i \assign B_{\theta}\eta_i$ and $W^t=\{w_i^t\}_{i=1}^N$ with $w_i^{t} \assign \frac{1}{t}(T_i^t-\id)$, where $T_i^t$ is the optimal transport map from $\lambda_\theta$ to $\lambda_{\theta_i^t}$.

Let $F(V)$ and $F(W^t)$ be the corresponding finite-rank self-adjoint operators as specified in \Cref{def:SpanOperators} on the Hilbert space $\mc{H} \assign \LL^2(\lambda_\theta;\R^d)$.

\begin{enumerate}[(i)]
\item \textbf{Eigenvalues:}
Let $(\mu_i)_{i=1}^m$ and $(\nu^t_i)_{i=1}^{n^t}$ be the non-zero eigenvalues of $F(V)$ and $F(W^t)$, sorted in decreasing order (and we think of both lists as being followed by infinitely many zero eigenvalues). Then $m=\dim \vecspan V \leq k = \dim T_\theta \mfold$, and $n^t \leq N$.
There is a constant $C \in (0,\infty)$, not depending on $\theta$ or $(\eta_i)_i$, such that for $t \in (0,1/C)$,
\begin{align}
	|\mu_i - \nu_i^t| & \leq C\cdot\sqrt{t} & & \tn{for } i=1,\ldots,m, \nonumber \\
	0 \leq \nu_i^t &  \leq C\cdot\sqrt{t} & & \tn{for } i=m+1,\ldots,n^t,
	\label{eq:TanSpaceRecoveryEigvals}
\end{align}
In particular, if $t$ is sufficiently small, such that $L_{t}\assign \mu_m - 2C\cdot \sqrt{t}>0$, the spectral gap of length $\mu_m$ for $F(V)$ translates to a gap $\nu_m^t-\nu_{m+1}^t \geq L_{t}$ for $F(W^{t})$.

\item \textbf{Eigenvectors:} Let $\mathcal{V}_m$ be the span of the top $m$ (orthonormal) eigenvectors of $F(V)$, and let $\mathcal{W}^{t}_m$ be the span of the top $m$ (orthonormal) eigenvectors of $F(W^{t})$. Then, for the same $C$ as above, for $t \in (0,1/C)$,
\begin{equation}
    d(\mathcal{W}^{t}_m,\mathcal{V}_m)\leq \frac{2C\cdot \sqrt{t}}{\mu_m},
	\label{eq:TanSpaceRecoveryEigvecs}
\end{equation}
where $d$ denotes the subspace distance of \Cref{def:DSubspace}.
\end{enumerate}
\end{theorem}
\begin{proof}
In \Cref{asp:Main} we fixed $\dim \mfold=\dim T_\theta \mfold=k$, and by linearity of $B_\theta : T_\theta \mfold \to \LL^2(\lambda_\theta;\R^d)$ one has
$$\dim \vecspan V = \dim \vecspan B_\theta \{\eta_i\}_{i=1}^N \leq \dim B_\theta T_\theta \mfold =k$$
where the last equality is implied by \eqref{eq:VelEquivalence}.
Hence, by \Cref{lem:SpanBasic}\eqref{item:EmpEigenbasis} one has $m=\dim \vecspan V \leq k$.
Similarly, $\dim \vecspan W \leq N$ yields $n^t \leq N$.
Both $F(V)$ and $F(W^t)$ are therefore self-adjoint, finite rank, and thus by \Cref{thm:Weyl}
\begin{align}
\label{eq:TanSpaceRecoveryPreVal}
|\mu_i-\nu_i^t| \leq \|F(V)-F(W^t)\|_{\HS(\mc{H})}.
\end{align}
From \eqref{eq:span-perturb}, together with \Cref{prop:wvconvergence} (where we use that $\|\eta_i\|_\theta \leq \diam \mfold$), and \eqref{eq:VelEquivalence}, we get that there is some $C>0$ such that for $t \in (0,1/C)$,
\begin{align}
\label{eq:TanSpaceRecoveryHSBound}
	\|F(V)-F(W^t)\|_{\HS(\mc{H})} & \leq C (t+ \sqrt{t}) \leq C' \cdot \omega(t) \tn{ for } t \in (0,1/C')
\end{align}
where $C'$ in the third expression has to be chosen accordingly (and we subsequently relabel it to $C$).
Combining \eqref{eq:TanSpaceRecoveryPreVal} with \eqref{eq:TanSpaceRecoveryHSBound}, and using that $\mu_i=0$ for $i>m$, yields \eqref{eq:TanSpaceRecoveryEigvals}.

For the spectral gap, notice that $\nu_m^t\geq \mu_m- C \cdot \sqrt{t}$ and $\nu^t_{m+1}\leq C\cdot \sqrt{t}$, which implies
\begin{equation*}
    \nu^t_{m}-\nu^t_{m+1} \geq \mu_m - 2C \cdot \sqrt{t}=L_{t}.
\end{equation*}

Since $F(V)$ and $F(W^t)$ are both finite rank, \Cref{thm:DK-finite-rank} (where we use $\mu_{m+1}=0$) and \eqref{eq:TanSpaceRecoveryHSBound} imply \eqref{eq:TanSpaceRecoveryEigvecs}.
\end{proof}

The theorem implies that if we can sample from $\Lambda$ in the close vicinity (say within a ball of radius $t$ in either $\W$ or $\W_\Lambda$, see \Cref{prop:WLambdaWComparison}) of a reference sample $\lambda \in \Lambda$, then by a spectral analysis of the optimal transport maps from $\lambda$ to the other samples (after subtracting the identity and rescaling, e.g.~to unit length), the span of these maps will approximate the tangent space of the submanifold $\Lambda$, and an orthonormal basis can be obtained by diagonalizing the corresponding empirical matrix $F(W^t)$. To obtain the full tangent space, of course the samples need to capture all potential directions of variation. As one sends $t$ to zero, the dimension of $\Lambda$ (or of the captured submanifold) will be revealed by a spectral gap.

\begin{remark}[Approximation of samples $\lambda \in \Lambda$]\label{rem:empirical-measures}
In practice, the measures $\lambda_\theta$ and $\lambda_{\theta_i^t}$ may also not be available exactly, but only as approximations, e.g.~due to empirical sampling, or due to numerical discretization.
Let $\{\lambda_{\theta,n}\}_n$ and $\{\lambda_{\theta_i^t,n}\}_n$ be two sequences of empirical approximations of $\lambda_\theta$ and $\lambda_{\theta_i^t}$. Then, under suitable regularity conditions, the optimal transport map between the latter can be estimated via barycentric projection of the (entropic) transport plans between the approximate measures by barycentric projection, see for instance \cite{deb2021rates,pooladian2021}. This estimation error can in principle be incorporated into \Cref{thm:RecoverTanSpace}.

However, note that a joint limit $t \to 0$ and $n \to \infty$ is delicate. For fixed $n$, as $t \to 0$, any bounded local deformation of $\lambda_{\theta,n}$ will eventually be an optimal transport deformation and thus one will eventually always observe $w^t_{i,n}=v_{i,n}$, even if $v_{i,n}$ was not originally a gradient vector field. This is illustrated in more detail in Section \ref{sec:Numerics}.
\end{remark}

\section{Additional examples}
\label{sec:ExampleAdditional}

In this section we discuss additional examples for Wasserstein submanifolds constructed as in Section \ref{sec:ExampleGeneral}.
In the first two examples (Sections \ref{sec:OneDimBase} and \ref{sec:OneDimParam}) the map \eqref{eq:EulerianVel} maps to gradient vector fields.
In Section \ref{sec:OneDimBase} this is ensured by considering the special case $\Omega \subset \R^d$ with $d=1$, where all vector fields are gradients. In this case the space $(\prob(\Omega),\W)$ is flat, but the submanifold $(\Lambda,\W_\Lambda)$ can be curved.
In Section \ref{sec:OneDimParam} the gradient property is ensured by considering the special case that $\mfold$ is one-dimensional, $m=1$. In this case the tangent direction $\eta \in T_\theta \mfold$ is merely a real number and one can prescribe $B_\theta$ by choosing some gradient vector field $v(\theta,\cdot)$, then setting $B_\theta \eta \assign v(\theta,\cdot) \cdot \eta$, and finally construct $\psi(\theta,\cdot)$ by integration of \eqref{eq:EulerianVel} with respect to $\theta$.
Finally, in Section \ref{sec:ExampleNonGradient} we consider the general case when the vector field \eqref{eq:EulerianVel} is not a gradient and show how $B_\theta \eta$ can be constructed in this case.
Explicit instances for all three constructions and numerical illustrations are given in Section \ref{sec:Numerics}.

\subsection{One-dimensional base space}
\label{sec:OneDimBase}

Assume $d=1$, let $\mfold$ be a parameter manifold in the sense of Assumption \ref{asp:Main}\eqref{item:mfold}.
Let $(\psi(\theta,\cdot))_{\theta \in \mfold}$ be a family of functions $\R \to \R$ which are strictly increasing in their second argument (and hence invertible).
It is then well-known \cite[Chapter 2]{santambrogio2015optimal} that $\psi(\theta,\cdot)$ is the optimal transport map between $\lambda$ and $\lambda_\theta = \psi(\theta,\cdot)_{\#} \lambda$ and since $\psi(\theta_2,\psi^{-1}(\theta_1,\cdot))$ is also increasing for $\theta_1,\theta_2 \in \mfold$, it will be the optimal transport map from $\lambda_{\theta_1}$ to $\lambda_{\theta_2}$ and therefore
\begin{align*}
\W(\lambda_{\theta_1},\lambda_{\theta_2}) & =\int_{\R} |\psi(\theta_2,\psi^{-1}(\theta_1,x))-x|^2 \diff \lambda_{\theta_1}(x) \\
= & \int_{\R} |\psi(\theta_2,x)-\psi(\theta_1,x)|^2 \diff \lambda(x)=\|\psi(\theta_2,\cdot)-\psi(\theta_1,\cdot)\|^2_{\LL^2(\lambda)}.
\end{align*}
That is $\Lambda$, with respect to the ambient metric $\W$, can be isometrically embedded into the flat space $\LL^2(\lambda)$ but it may be a curved surface in this space and thus the geodesic restriction $\W_\Lambda$ may be more complex.
Assume that $\nabla_\theta \psi$, $\partial_x \nabla_\theta \psi$, $\nabla_\theta \psi^{-1}$, and $\partial_x \psi^{-1}$ exist and are uniformly bounded on $\mfold \times \R$.
Following \eqref{eq:EulerianVel} we set
$$B_\theta \eta \assign \nabla_\theta \psi(\theta, \psi^{-1}(\theta,\cdot)) \cdot \eta.$$
Since $d=1$, any vector field is a gradient vector field. By the assumptions on the derivatives of $\psi$ it is easy to verify \eqref{eq:VelDerivatives}.
For verifying \eqref{eq:VelTimeSmooth} consider normal coordinates centered at $\theta$, such that $\gamma(t)=\exp_\theta(t \cdot \eta)=\theta + t \cdot \eta$.
We then need to ensure that the time-derivative of
$$v(t,x) \assign \nabla_\theta \psi(\theta + t \cdot \eta,\psi^{-1}(\theta + t\cdot \eta,x)) \cdot \eta)$$
is sufficiently regular, which follows from the chain rule and the above assumptions on the derivatives of $\psi$.
For \eqref{eq:VelEquivalence} and Assumption \ref{asp:Main}\eqref{item:consistent} see the discussion in Section \ref{sec:ExampleGeneral}.

\begin{remark}
\label{rem:W1uv}
We now examine \Cref{prop:wvconvergence} for this setting. Consider normal coordinates at $\theta$, such that $\exp_\theta(t \cdot \eta)=\theta + t \cdot \eta$. As mentioned above, the optimal transport map between $\lambda_{\theta}$ and $\lambda_{\theta + t \cdot \eta}$ is given by
$$T_t=\psi(\theta + t \cdot \eta, \psi^{-1}(\theta,\cdot))$$
and therefore convergence of $w_t = (T_t-\id)/t$ to $v=B_\theta \eta = \partial_t T_t|_{t=0}$ as $t \to 0$ follows from continuous differentiability of $\psi$ with respect to $\theta$. When $\psi$ is twice continuously differentiable with respect to $\theta$, this convergence will be at least linear in $t$.
\end{remark}

\subsection{One-dimensional parameter space}
\label{sec:OneDimParam}

Let $\mfold$ be a compact interval of $\R$, equipped with the Euclidean distance, i.e.~$m=1$. For simplicity we assume $0 \in \mfold$.
Let $\phi : \mfold \times \R^d \ni (\theta,x) \mapsto \phi(\theta,x) \in \R$ be such that $\nabla_x \phi$, $\nabla_x^2 \phi$ and $\partial_\theta \nabla_x \phi$ are bounded on $\mfold \times \R^d$.
Let $v \assign \nabla_x \phi : \mfold \times \R^d \to \R^d$ be the gradient of $\phi$ with respect to $x$ and then set $B_\theta \eta \assign v(\theta,\cdot) \cdot \eta$. We then construct the family of diffeomorphisms $\psi(\theta,\cdot)$ by integrating the Eulerian velocity field $B_\theta \eta$ with respect to $\theta$ as follows. Let $\psi : \mfold \times \R^d \to \R^d$ be given by the solution to the initial value problem
\begin{align}
\label{eq:OneDimParamPsiDef}
	\psi(0,\cdot) & = \id, & \partial_\theta \psi(\theta,\cdot) = v(\theta,\psi(\theta,\cdot))
\end{align}
for $\theta \in \mfold$. This ODE is well-posed by \Cref{lem:Flow} since $v=\nabla_x \phi$ is Lipschitz continuous in both arguments by boundedness of $\nabla_x^2 \phi$ and $\partial_\theta \nabla_x \phi$.
Similarly, assumption \eqref{eq:VelDerivatives} holds by assumptions on $\nabla_x \phi$ and $\nabla_x^2 \phi$. For a curve $\gamma(t)=\exp_\theta(t \cdot \eta)=\theta+t \cdot \eta$ one has $\partial_t [B_{\gamma(t)} \dot{\gamma}(t)](x)=\partial_\theta v(\gamma(t),x) \cdot \eta^2$, which implies that \eqref{eq:VelTimeSmooth} holds by assumption on $\partial_\theta \nabla_x \phi$. For \eqref{eq:VelEquivalence} see the discussion in Section \ref{sec:ExampleGeneral}.
Since $v$ is by construction the Eulerian velocity field associated with $\psi$ when varying $\theta$, Assumption \ref{asp:Main} \eqref{item:consistent} holds, as described in Section \ref{sec:ExampleGeneral}.

It seems difficult to generalize this construction to multi-dimensional parameter manifolds due to consistency conditions on the velocity fields during integration. When integrating $B_\theta \eta$ along several paths $\gamma_i$ in $\mfold$ between $0$ and the same endpoint $\theta \in \mfold$, the resulting $\psi_i(\theta,\cdot)$ must generate the same push-forward $\psi_i(\theta,\cdot)_{\#} \lambda$. We are not aware of a simple condition for ensuring this consistency.

\subsection{More general diffeomorphic deformations and gradient projection}
\label{sec:ExampleNonGradient}

We now discuss the situation when \eqref{eq:EulerianVel} does not yield gradient vector fields. It is easy to see that the curve $\lambda_{\gamma(t)}$ and the vector field $v(t,\cdot)\assign B_{\gamma(t)} \dot{\gamma}(t)$, as given by \eqref{eq:EulerianVel}, satisfy the continuity equation
$$\partial_t \lambda_{\gamma(t)} + \nabla(v(t,\cdot) \cdot \lambda_{\gamma(t)}) = 0$$
in a distributional sense (see \Cref{prop:EnergyBasic} for the definition and \Cref{prop:SubMan:ParMan:WLambda} for the arguments). Any other vector field $\tilde{v}(t,\cdot)$ that satisfies $\nabla(v(t,\cdot) \cdot \lambda_{\gamma(t)})=\nabla(\tilde{v}(t,\cdot) \cdot \lambda_{\gamma(t)})$ will implement the same \emph{macroscopic} curve in the sense that it results in the same measure curve $t \mapsto \lambda_{\gamma(t)}$, even though it may be \emph{microscopically} different in the sense that it moves individual mass particles differently.
When $v(t,\cdot)$ is not a \emph{gradient} vector field, then it will not be the displacement that optimal transport implements in the sense that it will not be minimizing for $\energy((t \mapsto \lambda_{\gamma(t)}))$ in \eqref{eq:Energy}.
Instead, for given $\lambda_{\gamma(t)}$ and $v(t,\cdot)$, the optimal velocity field $\tilde{v}(t,\cdot)$ will be given by the minimizer to
\begin{align*}
	\min_{\tilde{v} \in \LL^2(\lambda_{\gamma(t)})} \|\tilde{v}\|_{\LL^2(\lambda_{\gamma}(t))}^2  \quad \tn{subject to} \quad
	\nabla \cdot (\tilde{v} \cdot \lambda_{\gamma(t)}) = \nabla \cdot (v(t,\cdot) \cdot \lambda_{\gamma(t)}).
\end{align*}
By recalling the definition for the weak divergence and adding a corresponding Lagrange multiplier $\phi \in \CC^1(\Omega)$ this can be written as (the factor 2 is for convenience)
\begin{align*}
	\min_{\tilde{v} \in \LL^2(\lambda_{\gamma(t)})} \sup_{\phi \in \CC^1(\Omega)} \|\tilde{v}\|_{\LL^2(\lambda_{\gamma}(t))}^2  - 2\int_\Omega \langle \nabla \phi, \tilde{v}-v\rangle\, \diff \lambda_{\gamma(t)}.
\end{align*}
Relaxing now $\phi$ to $\HH^1(\spt \lambda_{\gamma(t)})$, applying a minimax theorem, using that the optimal $\tilde{v}$ for fixed $\phi$ is given by $\tilde{v}=\nabla \phi$, and finally by completing the square, one finds that this is equivalent to solving
\begin{align}
\label{eq:ProjGrad}
	\min_{\phi \in \HH^1(\spt \lambda_{\gamma(t)})} \| \nabla \phi - v(t,\cdot)\|^2_{\LL^2(\lambda_{\gamma(t)})}
\end{align}
and then setting $\tilde{v}=\nabla \phi$.
That is, $\tilde{v}$ is the $\LL^2(\lambda_{\gamma(t)})$-projection of $v(t,\cdot)$ to the subspace of gradient vector fields (see \cite[Section 3.3.2]{Ambrosio2013} for a related discussion). This projection is a linear map, and therefore the natural choice for $B_\theta$ is the composition of the Eulerian velocity map \eqref{eq:EulerianVel} with this projection. It remains to verify whether this composition satisfies the necessary regularity assumptions \eqref{eq:VelDerivatives} and \eqref{eq:VelTimeSmooth}.

In the following we will sketch how this can be done via regularity theory for elliptic PDEs \cite{EvansPDE,GilbargTrudingerEllipticPDEs}.
For this we assume that the family $(\psi(\theta,\cdot))_{\theta \in \mfold}$ has sufficient regularity with respect to $\theta$ and $x$ for all the following constructions. Such smooth families of diffeomorphisms can be constructed in various ways. For instance, by taking gradients (with respect to $x$) of a regular family of uniformly convex potentials $\phi(\theta,\cdot)$ (convexity ensures invertibility of the gradient), or by integrating smooth families of vector fields as described in \cite{YounesShape2010}, similar as in Section \ref{sec:OneDimParam}.
Let $l$ and $l_\theta$ denote the Lebesgue densities of $\lambda$ and $\lambda_\theta$, respectively.
Then \eqref{eq:ProjGrad} is formally the variational formulation for the elliptic PDE
\begin{align}
\label{eq:PDE}
	\nabla \cdot (l_{\gamma(t)} \nabla \phi) & = \nabla \cdot (v \cdot l_{\gamma(t)}) \quad \tn{on } \spt \lambda_{\gamma(t)}, &
	v_n \cdot \nabla \phi & = v_n \cdot v \quad \tn{on } \partial \spt \lambda_{\gamma(t)}
\end{align}
where $v_n$ is the outward unit normal vector field on $\partial \spt \lambda_{\gamma(t)}$, i.e.~we impose Neumann boundary conditions.
Under suitable assumptions (e.g.~$\spt \lambda_{\gamma(t)}$ being the closure of its interior, and $l_{\gamma(t)}$ being bounded away from zero and from above on $\spt \lambda_{\gamma(t)}$) minimizers of \eqref{eq:ProjGrad} and weak solutions of \eqref{eq:PDE} can be identified.
If, in addition, $\partial \spt \lambda_{\gamma(t)}$, $l_{\gamma(t)}$, and $v$ are more regular (e.g.~$\partial \spt \lambda_{\gamma(t)}$ being locally the graph of some $C^{k,\alpha}$ function and $l_{\gamma(t)}$ and $v$ possessing higher-order regularity) then the weak solution $\phi$ will also be a classical solution, it inherits higher-order regularity from the data, and the solution will continuously depend on data. By applying some suitable scheme to extend $\phi$ from $\spt \lambda_{\gamma(t)}$ to $\R^d$ one can therefore ensure \eqref{eq:VelDerivatives} at least locally at some given $\theta=\gamma(t)$.

Next, we sketch the uniform extension of \eqref{eq:VelDerivatives} to all of $\mfold$, and how to verify \eqref{eq:VelTimeSmooth}.
For simplicity, assume that $\psi$ satisfies $\psi(\theta,\Omega)=\Omega$ for $\Omega=\spt \lambda$, i.e.~each diffeomorphism is a rearrangement of the mass of $\lambda$ within $\Omega$ (this implies that $v_n \cdot v=0$ in \eqref{eq:PDE}, i.e.~the boundary conditions become homogeneous). Of course, more general constructions are conceivable.
With sufficient uniform regularity on $\psi$ and regularity of $l$ we obtain that the corresponding family of densities $l_\theta$ is uniformly bounded from below and above on $\Omega$, and has uniform $C^{k,\alpha}$ regularity on this set.
This yields \eqref{eq:VelDerivatives} uniformly on $\mfold$.
Finally, for \eqref{eq:VelTimeSmooth} one first needs to ensure that $\psi$ is sufficiently regular such that $v(t,\cdot)\assign B_{\gamma(t)} \dot{\gamma}(t)$ with the preliminary $B_\theta \eta$ as given by \eqref{eq:EulerianVel} satisfies \eqref{eq:VelDerivatives} and \eqref{eq:VelTimeSmooth} (see Section \ref{sec:ExampleGeneral} for details). Denote then by $\phi(t,\cdot)$ the solution to \eqref{eq:PDE} at time $t$ for velocity field $v(t,\cdot)$. For \eqref{eq:VelTimeSmooth} we must control $\partial_t \nabla \phi(t,\cdot)$.
Taking the time derivative of \eqref{eq:PDE} we find that $\partial_t \phi(t,\cdot)$ solves
\begin{align*}
	\nabla \cdot (l_{\gamma(t)} \nabla \partial_t \phi(t,\cdot)) & = \partial_t \nabla \cdot (v(t,\cdot) \cdot l_{\gamma(t)}) - \nabla \cdot ((\partial_t l_{\gamma(t)}) \nabla \phi(t,\cdot)) \quad \tn{on } \Omega
\end{align*}
with homogeneous Neumann boundary conditions, i.e.~\eqref{eq:VelTimeSmooth} follows from regularity theory for classical solutions in a similar way as \eqref{eq:VelDerivatives}.

\section{Numerical illustrations}
\label{sec:Numerics}

In this section we give explicit examples for the constructions described in Section \ref{sec:ExampleAdditional} and provide some numerical illustrations.\footnote{Code can be found at \url{https://github.com/bernhard-schmitzer/W2Manifolds/}.}

\subsection{One-dimensional base space}
\label{sec:NumericsOneDimBase}

\begin{figure}[hbt]
\centering
\includegraphics[width=\textwidth]{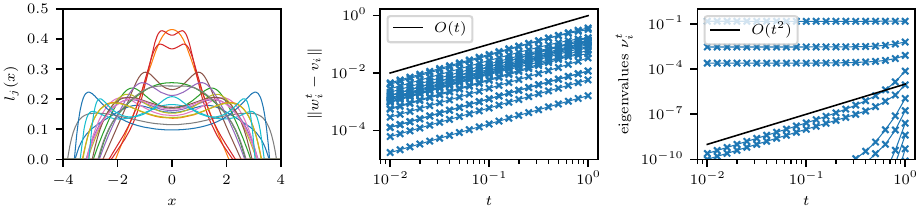}
\caption{Example for one-dimensional base space.
Left: Lebesgue densities $l_{j}$ for some sample measures $(\lambda_{j})_j$ from the manifold $\Lambda$.
Middle: For fixed $\theta \in \mfold$ and various tangent vectors $(\eta_i)_i$ in $T_\theta \mfold$ the difference between between $v_i \assign B_{\theta} \eta_i$ and $w_i^t \assign (T_i^t-\id)/t$ in $\LL^2(\lambda_\theta)$ is shown over $t$, where $T_i^t$ is the optimal transport map from $\lambda_\theta=\emb(\theta)$ to $\emb(\theta + t \cdot \eta_i)$. As expected, in this case we observe linear scaling, as indicated by the black line.
Right: Leading eigenvalues of the span operator $F(W^t)$ (\Cref{thm:RecoverTanSpace}) over $t$. As $t \to 0$ we recover three non-zero eigenvalues, corresponding to the dimension of $\mfold$. Residual eigenvalues decay as $O(t^2)$.
See text for full details.}
\label{fig:OneDimBase}
\end{figure}
Now we choose an explicit instance of the construction detailed in Section \ref{sec:OneDimBase}.
We consider the family of functions
\begin{align*}
	\psi(\theta,x) & \assign x + \sum_{i=1}^k a_i(\theta) \cdot \tanh((x-z_i(\theta))/w_i(\theta))
\end{align*}
where amplitudes $a_i$, offsets $z_i$, and widths $w_i$ are smooth functions of $\theta$. As long as the sum of the $a_i$ is bounded from below away from $-1$ and bounded from above, and the $w_i$ are positive and bounded away from zero, this is a family of increasing diffeomorphisms on $\R$. Concretely, we use $m=3$, $\theta=(\theta_1,\theta_2,\theta_3) \in \mfold = [0,1] \times [0,1] \times [0.2,1]$,
$k=2$ and set
\begin{align*}
	a_1(\theta) & = a_2(\theta) = 1, & a_3(\theta) & = \theta_1, \\
	z_1(\theta) & = -z_2(\theta) = \theta_2, & z_3(\theta) & =0 \\
	w_1(\theta) & = w_2(\theta) = \theta_3, & w_3(\theta) & = 0.3.
\end{align*}
As template $\lambda$ we choose the measure with the Lebesgue density $l(x) \assign \max\{0,0.75 \cdot (1-x^2)\}$.
As discussed, for given $\theta \in \mfold$ we set $\lambda_\theta \assign \psi(\theta,\cdot)_{\#} \lambda$, one obtains for the corresponding Lebesgue density $l_\theta(x)= l(\psi^{-1}(\theta,x)) \cdot \partial_x \psi^{-1}(\theta,x)$, and for the pairwise distances one has $\W(\lambda_{\theta_1},\lambda_{\theta_2})=\|\psi(\theta_1,\cdot)-\psi(\theta_2,\cdot)\|_{\LL^2(\lambda)}$.
We numerically approximate integrals of $\psi(\theta,\cdot)$ by evaluation on a uniform grid over $\spt \lambda = [-1,1]$, we approximate the inverse by inverting the piecewise linear interpolation of the grid values, and derivatives by finite differences.
Several example densities from $\Lambda$ are shown in Figure \ref{fig:OneDimBase}, left.

We perform a tangent space analysis on this manifold.
As reference parameter we choose $\theta \assign (0.5,0.5,0.6)$ and we sample $N\assign500$ tangent vectors $(\eta_i)_{i=1}^N$ uniformly from $\domexp_{\theta}=\mfold-\theta$.
Similar to \Cref{prop:wvconvergence}, let $v_i \assign B_{\theta} \eta_i$ and $w_i^t \assign (T_i^t-\id)/t$ for $t \in (0,1]$, where $T_i^t$ is the optimal transport map from $\emb(\theta)$ to $\emb(\theta + t \cdot \eta_i)$.
In this case it is given by $\psi(\theta + t \cdot \eta_i, \psi^{-1}(\theta,\cdot)$.

Figure \ref{fig:OneDimBase}, middle shows the deviation $\|v_i-w_i^t\|_{\LL^2(\emb(\theta))}$ for 20 samples and $t \in [10^{-2},1]$.
As expected, since in this case the transport maps $T_t^i$ are explicitly known and twice differentiable with respect to $t$, one observes linear convergence of $w_i^t$ to $v_i$ as $t\to 0$, cf.~\Cref{prop:wvconvergence} and \Cref{rem:W1uv}.
Figure \ref{fig:OneDimBase}, right shows the spectrum of the span operator $F(W^t)$ of \Cref{thm:RecoverTanSpace}. As expected, since $m=3$ and we picked tangent vectors $(\eta_i)_i$ that span (with probability one) the whole tangent space, three eigenvalues stabilize and remain non-zero as $t \to 0$.
The other eigenvalues tend to zero with an observed rate of $O(t^2)$. This is faster than the rate $O(\sqrt{t})$ appearing in \Cref{thm:RecoverTanSpace}. This is partially explained by the faster convergence of $w_i$ to $v_i$ (see Figure \ref{fig:OneDimBase}, middle), and we conjecture that the bound \eqref{eq:span-perturb} is not sharp in this experiment. However, with limited double floating point precision it is not possible to examine this numerically for small $t$.
The presence of more than three non-zero eigenvalues at larger $t$ confirms that the submanifold $\Lambda$ has extrinsic curvature within the flat space $(\probL(\Omega),\W)$ for $d=1$.

\subsection{One-dimensional parameter space}
\label{sec:NumericsOneDimParam}

\begin{figure}[hbt]
\centering
\includegraphics[width=\textwidth]{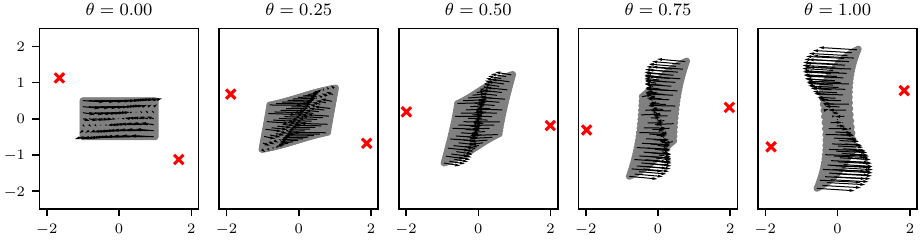}
\caption{Example for one-dimensional parameter space. At $\theta=0$, $\lambda_0=\lambda$ is the uniform probability measure on $[-1,1] \times [-0.5,0.5]$, discretized by a uniform Cartesian grid, shown in grey in the left panel. For several other $\theta$ the deformed point cloud is shown. Note that the point density is no longer uniform for $\theta \neq 0$, which is not explicitly encoded in the figure. Black arrows show a (subsampling) of the velocity field $v(\theta,\cdot)=\nabla_x \phi(\theta,\cdot)$. The red crosses indicate the positions $z_1(\theta)$, $z_2(\theta)$ used for the parametrization of the velocity field, \eqref{eq:NumOneDimParamZ}. See text for full details.}
\label{fig:OneDimParamDeformation}
\end{figure}

Next, we provide an instance of the construction described in Section \ref{sec:OneDimParam}.
We choose $d=2$, $\mfold=[-1,1]$,
$$\phi(\theta,x) \assign -\sum_{i=1}^k \exp\left(-|x-z_i(\theta)|^2/4\right)$$
with $k=2$ and
\begin{align}
\label{eq:NumOneDimParamZ}
	z_1(\theta) & \assign 2 \cdot \begin{pmatrix} \cos(\theta-\delta) \\ \sin(\theta-\delta) \end{pmatrix}, &
	z_2(\theta) & \assign -z_1(\theta), &
	\delta & \assign 0.6.
\end{align}
As template $\lambda$ we choose the uniform probability measure on $[-1,1] \times [-0.5,0.5]$.
We discretize this with a uniform Cartesian grid. Since this discretization will lead to some artefacts (see discussion of the results further down in this section, and in particular the example in Section \ref{sec:NumericsGradProj}) we repeat the experiment at different discretization scales to be able to detect these artefacts. Concretely, we discretize the rectangle $\spt \lambda$ with $50 \times 25$, $100 \times 50$, and $200 \times 100$ points, respectively. $\psi(\theta,\cdot)$ is then evaluated on these grid points by integrating $v(\theta,\cdot) \assign \nabla_x \phi(\theta,\cdot)$ as in \eqref{eq:OneDimParamPsiDef} with the scipy \texttt{solve\_ivp} function, using the explicit Runge--Kutta method of order 5. We obtain $\lambda_\theta$ numerically by a Lagrangian discretization, i.e.~by moving the points of the original grid according to $\psi(\theta,\cdot)$.
Some exemplary $\lambda_\theta$ and the velocity fields $v(\theta,\cdot)$ are shown in Figure \ref{fig:OneDimParamDeformation}.

\begin{figure}[hbt]
\centering
\includegraphics[]{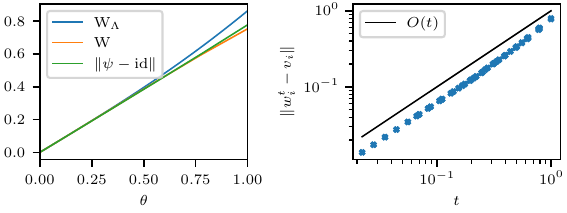}
\caption{Example for one-dimensional parameter space. Left: Different notions of pairwise distances on the manifold of Figure \ref{fig:OneDimParamDeformation}. Shown are $\W_\Lambda(\lambda_0,\lambda_\theta)$, $\W(\lambda_0,\lambda_\theta)$, and $\|\psi(\theta,\cdot)-\id\|_{\LL^2(\lambda_0)}$ for various $\theta$.
Right: Difference between $v_i \assign B_{\theta} \eta_i$ and $w_i^t \assign (T_i^t-\id)/t$ in $\LL^2(\lambda_0)$ over $t$ for various tangent vectors $\eta_i$, similar to Figure \ref{fig:OneDimBase}, middle. Again we observe linear scaling, as indicated by the black line. All results shown here are based on a discretization of $\lambda$ with $200 \times 100$ points.}
\label{fig:OneDimParamSummary}
\end{figure}

\begin{figure}[hbt]
\centering
\includegraphics[]{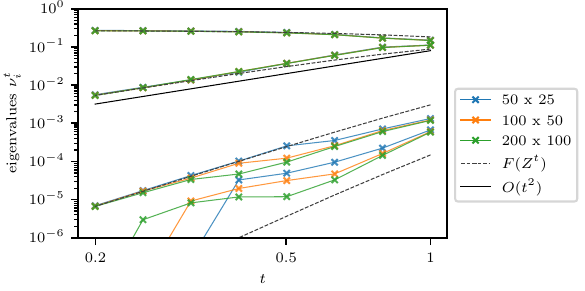}
\caption{Example for one-dimensional parameter space. Spectrum of the span operator $F(W^t)$ for different $t$ and different grid resolutions ($50 \times 25$, $100 \times 50$, $200 \times 100$). For comparison the spectrum of the operator $F(Z^t)$ is shown where $Z^t$ is constructed from the finite differences $\psi(t \cdot \eta_i,\cdot)-\id$. For visual clarity only the four largest eigenvalues are shown. For small $t$ the spectra of $F(W^t)$ and $F(Z^t)$ agree, at some points the precise transition depends on the grid discretization scale. Non-dominant eigenvalues decay like $O(t^2)$. See text for full details.}
\label{fig:OneDimParamSpectrum}
\end{figure}

Again, we perform a tangent space analysis on this example. As base point we choose $0 \in \mfold$. For $\theta \in [0,1]$ we then compute $\W(\lambda_0,\lambda_\theta)$, $\W_{\Lambda}(\lambda_0,\lambda_\theta)$ and $\|\psi(\theta,\cdot)-\id\|_{\LL^2(\lambda_0)}$.
We approximate the former numerically by solving the optimal transport problem between the point clouds, using entropic regularization and the multi-scale algorithm of \cite{SchmitzerScaling2019}.\footnote{Code available at \url{https://bernhard-schmitzer.github.io/MultiScaleOT/}} The final entropic blur parameter is chosen substantially smaller than the point discretization scale, such that the influence of entropic blurring is practically negligible. As usual, transport maps were then approximated by barycentric projection \cite{deb2021rates,pooladian2021}.
For one-dimensional parameters $\W_{\Lambda}(\lambda_0,\lambda_\theta)$ can be computed by integrating the path lengths of the particles along $\psi(\theta,\cdot)$ in the spirit of \Cref{prop:SubMan:Pullback:energy}. Finally, $\|\psi(\theta,\cdot)-\id\|_{\LL^2(\lambda_0)}$ simply measures the displacement of all particles by $\psi(\theta,\cdot)$ along straight lines. Compared to $\W_{\Lambda}(\lambda_0,\lambda_\theta)$ it does not measure the length of the paths of the continuous deformation from $\id=\psi(0,\cdot)$ to $\psi(\theta,\cdot)$ as $\theta$ is varied, and compared to $\W_{\Lambda}(\lambda_0,\lambda_\theta)$ the displacement implied by $\psi(\theta,\cdot)$ is not necessarily the optimal transport map. We therefore expect the inequality
\begin{equation}
\label{eq:NumOneDimParamHierarchy}
\W(\lambda_0,\lambda_\theta) \leq \|\psi(\theta,\cdot)-\id\|_{\LL^2(\lambda_0)} \leq \W_{\Lambda}(\lambda_0,\lambda_\theta).
\end{equation}
For $\theta \to 0$ we expect that all three behave similarly based on the results of Section \ref{sec:LocalLin}.
This is numerically confirmed in Figure \ref{fig:OneDimParamSummary}, left.
We observe in the figure that the inequalities in \eqref{eq:NumOneDimParamHierarchy} are strict, indicating that the line $\mfold=[-1,1]$ embedded into $(\probL(\Omega),\W)$ has non-zero extrinsic curvature.

Next, similar to the previous example, for base point $0 \in \mfold$, we pick 10 $(\eta_i)_i \in \domexp_0 = \mfold = [-1,1]$ on a uniform grid and set $v_i \assign B_{\theta} \eta_i$ and $w_i^t \assign (T_i^t-\id)/t$, as in the previous section, except that now the optimal transport map is not known in closed form but is approximated numerically.
Figure \ref{fig:OneDimParamSummary}, right shows that the discrepancy $\|v_i-w_i^t\|_{\LL^2(\lambda_0)}$ converges to 0 as $t \to 0$ as expected (\Cref{prop:wvconvergence}), approximately as $O(t)$ (see also \Cref{rem:W1uv}).

Finally, the spectrum of the span operator $F(W^t)$ is shown in Figure \ref{fig:OneDimParamSpectrum}. We observe that only one eigenvalue remains non-zero as $t \to 0$, consistent with the fact that $\mfold$ is one-dimensional. Interestingly, similar as in Section \ref{sec:NumericsOneDimBase} the non-dominant eigenvalues decay as $O(t^2)$. For comparison, we also show the spectrum of the span operator $F(Z^t)$ where $Z^t=\{\tfrac1t [\psi(t \cdot \eta_i,\cdot)-\id]\}_{i}$ are the displacements given by the deformation $\psi$. The spectra of $F(W^t)$ and $F(Z^t)$ differ for large $t$, in particular the dominant eigenvalue of $F(W^t)$ is shifted downwards relative to $F(Z^t)$, which corresponds to the fact that $\psi(t \cdot \eta_i,\cdot)$ is not the optimal transport map between $\lambda_0$ and $\lambda_{t \cdot \eta_i}$ for $t>0$.
However, by construction $\partial_\theta \psi(0,\cdot)$ is a gradient, and thus asymptotically $\psi(\theta,\cdot)$ is close to the optimal transport map for small $\theta$ (cf.~Section \ref{sec:LocalLin}), which is why the spectra of $F(W^t)$ and $F(Z^t)$ become increasingly similar as $t \to 0$.

Figure \ref{fig:OneDimParamSpectrum} shows three approximations of the spectrum of $F(W^t)$, computed at three different grid resolutions.
While the spectra agree on the two dominant eigenvalues, one observes a subtle discrepancy on the third largest eigenvalue.
This can be explained as follows. For a discrete, finite point cloud $(x_i)_i$ in $\R^d$, for $\theta$ sufficiently close to (but distinct from) $0$, $\psi(\theta,\cdot)$ will eventually be the exact optimal transport map between a measure concentrated on $(x_i)_i$ and its push-forward that is concentrated on $(\psi(\theta,x_i))_i$, simply because, eventually for each $i$, the point closest to $x_i$ in the deformed point cloud will be $\psi(\theta,x_i)$. Thus, for sufficiently small $t$, $W^t$ and $Z^t$ and the corresponding spectra will coincide (modulo some minor noise form residual entropic blur). The scale $t$ when this agreement sets in depends on the grid resolution, as is confirmed by Figure \ref{fig:OneDimParamSpectrum}.
This effect will become more prominent in the next numerical example.

\subsection{Gradient projection}
\label{sec:NumericsGradProj}

\begin{figure}
\centering
\includegraphics[width=\textwidth]{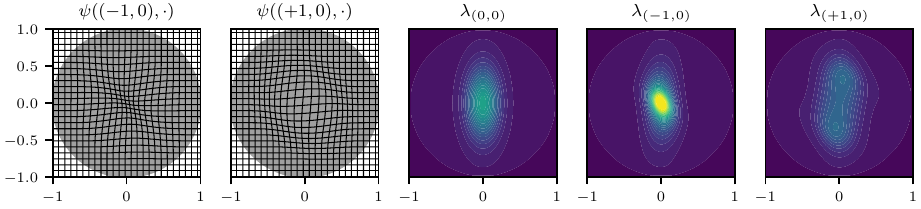}
\caption{Example for gradient projection. Deformations $\psi$ and densities of various $\lambda_\theta$ on the unit disk. Rotation is not shown. Numerically, the push-foward was implemented by Lagrangian discretization, i.e.~changing positions of a weighted point cloud. This figure shows a kernel density estimate of the density on a fixed Eulerian reference grid. See text for full details.}
\label{fig:FullDeformations}
\end{figure}

In this subsection we analyze a concrete instance of the construction described in Section \ref{sec:ExampleNonGradient}.
This verifies the existence of submanifolds in the sense of Assumption \ref{asp:Main} for $d>1$ and $m>1$ that are not flat, and it will illustrate the role of the gradient projection step in the construction of the velocity field operator $B_\theta$.
The example was designed such that its gradient projection problem is sufficiently regular, such that it exhibits non-trivial mass displacements that can be captured with reasonable numerical resolution, and such that it only requires a few simple explicit terms in the definition of $\psi$.

Let $d=2$, $\Omega \assign \{ x \in \R^d : |x| \leq 1\}$ be the closed unit disk. We use two parameters, $m=2$, $\mfold = [-1,1]^2$ and set
\begin{align*}
	\psi((\theta_1,\theta_2),x) & \assign R_{\theta_1 \cdot \pi/10} \left[ x+ 0.075 \cdot \theta_2 \cdot \nabla_x g\left(\sqrt{x^\top A x} \right) \right], \\
	\tn{with} \quad g(r) & \assign \begin{cases}
		\exp(-r^2/0.5^2) \cdot \exp(-1/(1-r^2)) & \tn{if } |r|<1, \\
		0 & \tn{else,}
	\end{cases} \\
	A & \assign R_{\pi/4} \begin{pmatrix} 1.8 & 0 \\ 0 & 1 \end{pmatrix} R_{-\pi/4}
\end{align*}
where $R_\alpha$ denotes the rotation matrix by angle $\alpha$ in $\R^2$. The second factor in the definition of $g$ is a well-known $\CC^\infty$ bump function with compact support on $[-1,1]$, the first factor was added to damp its derivatives near the boundary of the disk, which would otherwise require a more highly resolved grid to avoid discretization artefacts in the numerical push-forward of measures.
For $|\theta_2|$ sufficiently small, the part within the square brackets in the definition of $\psi(\theta,\cdot)$ is the gradient of a smooth uniformly convex potential (using $x=\nabla_x \tfrac12 |x|^2$) and thus a diffeomorphism. Using that $g$ is a bump function on the unit disk, we see that it maps the unit disk onto itself. The second step in the definition of $\psi$ is a rotation, which also is a smooth family of diffeomorphisms, that map the unit disk onto itself. For $\lambda$ we use the probability measure concentrated on the unit disk with Lebesgue density given by
\begin{align*}
l(x) & \assign \begin{cases}
	\mc{N} \cdot \left( 0.1 + \exp(-x^\top \Sigma^{-1} x/2) \right) & \tn{for } |x| \leq 1, \\
	0 & \tn{else,}
	\end{cases}
	\qquad \tn{with} \qquad \Sigma= \begin{pmatrix}
	0.02 & 0 \\ 0 & 0.1
	\end{pmatrix}.
\end{align*}
where $\mc{N}$ is a normalization constant.
By varying $\theta_2$ the Gaussian bump in $\lambda_0=\lambda$ can be made more concentrated or spread out into two modes.
Example of $\psi$, $\lambda(\theta,\cdot)$, and $\lambda_\theta$ for some values of $\theta$ are visualized in Figure \ref{fig:FullDeformations}.
Numerically $\lambda$ is again discretized by a uniform Cartesian grid on $[-1,1]^2$ which is subsequently restricted to the unit disk. To better understand the influence of discretization artefacts we repeat all experiments at various grid resolutions. Concretely, we use resolutions $51 \times 51$, $102 \times 102$, $204 \times 204$, $408 \times 408$ on $[-1,1]^2$.

\begin{figure}[hbt]
\centering
\includegraphics[width=\textwidth]{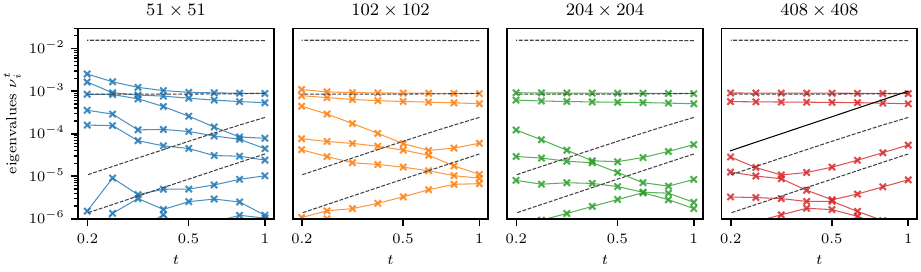}
\caption{Example for gradient projection. Spectrum of span operator $F(W^t)$ for different $t$ and grid resolutions (colored solid lines) and spectrum of $F(Z^t)$ for comparison (grey dashed lines). Similar to previous experiments, non-dominant eigenvalues decay quadratically in $t$ (indicated by the black line in the last panel). For small $t$ we observe discretization artefacts. See text for full details.}
\label{fig:ProjectionSpectrum}
\end{figure}

As before, we perform a tangent space analysis of the manifold, using $\theta=(0,0)$ as the base point such that $\lambda_{(0,0)}=\lambda$.
Figure \ref{fig:ProjectionSpectrum} shows the spectra of the span operator $F(W^t)$ and $F(Z^t)$, similar to Figure \ref{fig:OneDimParamSpectrum}.
We observe for large $t$ that $F(Z^t)$ has one eigenvalue that is substantially larger than any eigenvalue in $F(W^t)$.
The two largest eigenvalues in the spectrum of $F(Z^t)$ correspond approximately to the variation of $\psi$ with respect to $\theta_1$ and $\theta_2$ (the two velocity fields $\partial_{\theta_1} \psi$ and $\partial_{\theta_2} \psi$ are approximately orthogonal in $\LL^2(\lambda)$). This means, the largest eigenvalue corresponds approximately to the rotation of $\lambda$, which involves considerable movement of mass particles, and is far from an optimal mass rearrangement.
In the spectrum of $F(W^t)$, for large $t$, the largest eigenvalue corresponds approximately to the variation with respect to $\theta_2$ (which was constructed by using a gradient field), and the second largest eigenvalue corresponds approximately to the $\LL^2(\lambda)$-projection of the rotation to gradient fields, which has substantially lower movement cost than the original rotation, but induces the same change on the measure $\lambda$. As $t$ decreases, the two dominant eigenvalues remain stable, while all others decay. However, at some point several eigenvalues start to increase again, including the dominant eigenvalues. This corresponds to the transition of the spectrum of $F(W^t)$ to that of $F(Z^t)$ due to finite grid resolution where even the rotation displacement eventually becomes an optimal transport displacement (see discussion in Section \ref{sec:NumericsOneDimParam}). This interpretation is confirmed by the observation that the onset of the transition happens earlier on coarser grids.

\section{Conclusion and outlook}

In this article we have introduced a class of subsets $\Lambda$ of $\probL(\Omega)$ that is not flat but still allows bounds on the approximation error of linearized optimal transport in the spirit of finite-dimensional Riemannian geometry.
We have derived several corresponding extrinsic and intrinsic local linearization results, and results on recovering the local and global metric structure of $\Lambda$ from samples and the ambient metric $\W$. Some analytical and numerical examples were given as illustrations.

The obtained results suggest several directions for future work.

One such direction would be a more detailed study of the regularity assumptions needed to construct submanifolds of Wasserstein space that allow for strong local linearization results.
We have assumed that the local deformation velocity fields are $\CC^1$. However, optimal transport maps are known to be $\CC^1$ only under relatively strong assumptions on the involved measures, e.g.~requiring convex support \cite{Caffarelli-RegularityMappingsConvex-1992}, see also \cite[Section 12]{villani2009optimal}. Likewise, the projection to gradient vector fields discussed in Section \ref{sec:ExampleNonGradient} will only be $\CC^1$ under suitable regularity assumptions. Are these (at least approximately) satisfied in applications?
To what extent can our assumptions be relaxed?

In \eqref{eq:Pullback} we have formally considered the pull-back of the Riemannian tensor of $\W$ to $\mfold$ and shown that shortest paths with respect to this tensor exist, \Cref{prop:SubMan:Pullback:energy}.
For deriving our linearization results we have avoided establishing regularity properties of this tensor.
Our results in Sections \ref{sec:SubMan} to \ref{sec:TangentSpace} aim at recovering the submanifold $\Lambda$, the metric $\W_\Lambda$, and its tangent space when only samples of $\Lambda$ and the ambient metric $\W$ is available. This is an appropriate setting for many data analysis applications, and the results could not be obtained by a purely intrinsic analysis. Note that the approximation result of Section \ref{sec:GW} does not rely on the Riemannian structure of $(\Lambda,\W_\Lambda)$ but merely on its metric (see also the next point).
In future work, one could study the Levi-Civita connection and curvature tensor of $(\Lambda,\W_\Lambda)$, as well as the logarithmic and exponential maps, to verify that, when restricted to suitable finite-dimensional sets, Wasserstein space becomes a true Riemannian manifold. We expect that adapting the formal computations of \cite{LottWassersteinRiemannian2008} to our setting will be challenging because the class of `canonical vector fields' considered in \cite{LottWassersteinRiemannian2008} is in general not tangent to $\Lambda$.

The approximation result in the sense of Gromov--Wasserstein, Section \ref{sec:GW}, is based on discrete metric graphs and does not use any of the differentiable structure of the metric space $(\Lambda,\W_\Lambda)$. For instance, shortest paths in the approximate space will consist of discrete jumps along the graph and only ever consist of observed samples. Therefore, to obtain a faithful approximation, many samples and edges will be required. Presumably a more efficient approximation would be possible by exploiting the Riemannian structure more explicitly with other manifold learning methods, such as diffusion maps. Verifying that such techniques can be applied to our setting will probably be simplified by a more thorough intrinsic analysis as mentioned above.

\appendix

\section[Proof of the Benamou--Brenier formula]{Proof of the Benamou--Brenier formula, \protect{\Cref{prop:EnergyBasic}}}

Since $\Omega$ is convex, $(\prob(\Omega),\W)$ is a geodesic space \cite[Theorem 5.27]{santambrogio2015optimal} and therefore one has \cite[Box 5.2]{santambrogio2015optimal}
$$\W(\mu,\nu)^2 \hspace{-0.64pt} = \hspace{-0.64pt} \inf \hspace{-0.64pt} \left\{ \int_0^1 |\rho^\prime(t)|\,\dd t \,\middle|\, \rho \text{ is an absolutely continuous curve in $\W$ from $\mu$ to $\nu$} \right\}$$
where $|\rho^\prime(t)|$ is the metric derivative of $\rho$. In addition, the infimum can be restricted to Lipschitz continuous curves.
By \cite[Theorem 5.14]{santambrogio2015optimal}, for absolutely continuous curves $\rho$ there exists $v^* \in \CE(\rho)$ such that $\|v^*(t)\|_{\Lp{2}(\rho(t))}\leq |\rho^\prime|(t)$ for almost all $t$ and conversely any  $v\in \CE(\rho)$ satisfies $|\rho^\prime|(t)\leq \|v(t)\|_{\Lp{2}(\rho(t))}$ for almost all $t$ and so it follows that $v^*$ satisfies $\|v^*(t)\|_{\Lp{2}(\rho(t))}=|\rho^\prime|(t)$ and is a minimizer for $\energy(\rho)$.
Therefore $\W(\mu,\nu)$ can be written as above, and minimizing velocity fields exist.

It remains to show that $\energy$ is lower-semicontinuous.
By~\cite[Proposition 5.18]{santambrogio2015optimal} (where we choose $X=[0,1] \times \Omega$) we have that
\[ \cB: (\hat{\rho},\hat{m}) \mapsto \sup_{\substack{a\in\CC([0,1] \times \Omega;\bbR), b\in \CC([0,1] \times \Omega;\bbR^d) \\ a(x) + \frac14 |b(x)|^2 \leq 0 \, \forall x \in [0,1] \times \Omega}} \int_{[0,1] \times \Omega} a(x) \, \dd \hat{\rho} + \int_{[0,1] \times \Omega} \, \dd \langle b(x), \hat{m} \rangle \]
is lower semi-continuous on $\meas([0,1] \times \Omega)\times \meas([0,1] \times \Omega)^d$ with respect to weak* convergence, is finite only if $\hat{\rho} \geq 0$ and $\hat{m} \ll \hat{\rho}$, and satisfies $\cB(\hat{\rho},\hat{v} \cdot \hat{\rho}) =  \| \hat{v}\|_{\Lp{2}(\hat{\rho})}^2$ for measurable $\hat{v} : [0,1] \times \Omega \to \R^d$.

Let $\rho_n\in\Lip([0,1];(\prob(\Omega),\W))$ converge uniformly in the space of curves to $\rho$.
We want to show $\liminf_{n\to\infty} \energy(\rho_n)\geq \energy(\rho)$.
Consequently, we assume $\liminf_{n\to\infty}\energy(\rho_n)<+\infty$ (else the result is trivial), and by recourse to a subsequence we may further assume that there exists $C<\infty$ such that $\energy(\rho_n)\leq C^2$.
Let $v_n$ be the corresponding minimizing velocity fields. Further, let $(\hat{\rho}_n,\hat{m}_n) \in \meas([0,1]\times\Omega)^{1+d}$ be given by
\begin{align*}
& \int_{[0,1]\times \Omega} \phi\, \dd \hat{\rho}_n
  = \int_0^1 \int_\Omega \phi(t,\cdot) \dd \rho_n(t) \, \dd t
  \qquad \forall \phi\in\CC([0,1]\times\Omega;\bbR), \\
& \int_{[0,1]\times \Omega} \, \dd \langle \phi,\hat{m}_n \rangle
  = \int_0^1 \int_\Omega \langle \phi(t,\cdot) , v_n(t) \rangle \, \dd \rho_n(t) \, \dd t
  \qquad \forall \phi\in\CC([0,1]\times\Omega;\bbR^d).
\end{align*}
Then clearly, $\energy(\rho_n)=\cB(\hat{\rho}_n,\hat{m}_n)$ and $\hat{\rho}_n$ converge weak* to $\hat{\rho}$, characterized by
\begin{align*}
& \int_{[0,1]\times \Omega} \phi\, \dd \hat{\rho}
  = \int_0^1 \int_\Omega \phi(t,\cdot) \dd \rho(t) \, \dd t
  \qquad \forall \phi\in\CC([0,1]\times\Omega;\bbR).
\end{align*}
Further, we have
\[ \|\hat{m}_n\|_{\TV}\leq \int_0^1 \| v_n(t)\|_{\Lp{2}(\rho_n(t))} \, \dd t \leq \sqrt{\energy(\rho_n)} \leq C, \]
and so by Prokhorov's theorem there exists $\hat{m}\in\meas([0,1]\times \Omega)^d$ such that, after selecting a suitable subsequence, $\hat{m}_n\weakstarto \hat{m}$. By lower-semicontinuity we now have $\cB(\hat{\rho},\hat{m}) \leq \liminf_n \cB(\hat{\rho}_n,\hat{m}_n)\allowbreak \leq C$, which implies that $\hat{m} \ll \hat{\rho}$ and that the density $\hat{v}\assign\RadNik{\hat{m}}{\hat{\rho}}$ is in $\Lp{2}(\hat{\rho})$. Setting now $v(t) \assign \hat{v}(t,\cdot)$, one finds that $v \in \CE(\rho)$, since
\begin{align*}
& \int_0^1 \int_\Omega \partial_t \phi(t,\cdot) \, \dd \rho(t) \, \dd t
+ \int_0^1 \int_\Omega \, \langle \nabla \phi(t,\cdot),v(t) \rangle \, \dd \rho(t)\, \dd t \\
= {} & \int_{[0,1] \times \Omega} \partial_t \phi \, \dd \hat{\rho}
+ \int_{[0,1] \times \Omega} \,\dd\langle \nabla \phi,\hat{m} \rangle
\leftarrow
 \int_{[0,1] \times \Omega} \partial_t \phi \, \dd \hat{\rho}_n
+ \int_{[0,1] \times \Omega} \,\dd\langle \nabla \phi,\hat{m}_n \rangle \\
= {} & \int_0^1 \int_\Omega \partial_t \phi(t,\cdot) \, \dd \rho_n(t) \, \dd t
+ \int_0^1 \int_\Omega \, \langle \nabla \phi(t,\cdot),v_n(t) \rangle \, \dd \rho_n(t)\, \dd t \\
= {} & \int_\Omega \phi(1,\cdot) \, \dd \rho_n(1) - \int_\Omega \phi(0,\cdot) \, \dd \rho_n(0) 
\rightarrow \int_\Omega \phi(1,\cdot) \, \dd \rho(1) - \int_\Omega \phi(0,\cdot) \, \dd \rho(0).
\end{align*}
And therefore, $\energy(\rho) \leq \cB(\hat{\rho},\hat{m})$.

\section*{Acknowledgements}

KH was partially supported by a Research Enhancement Program grant from the College of Science at the University of Texas at Arlington. Research was sponsored by the Army Research Office and was accomplished under Grant
Number W911NF-23-1-0213. The views and conclusions contained in this document are those of the authors and
should not be interpreted as representing the official policies, either expressed or implied, of the Army Research
Office or the U.S. Government. The U.S. Government is authorized to reproduce and distribute reprints for
Government purposes notwithstanding any copyright notation herein.

CM is supported by NSF awards DMS-2306064 and DMS-2410140 and by a seed grant from the School of Data Science and Society at UNC.

BS was supported by the Emmy Noether Programme of the DFG (project number 403056140) and the DFG Collaborative Research Center 1456, ``Mathematics of Experiment'', project A03.

MT was supported by NoMADS, a European Union Horizon 2020 research and innovation programme grant (project number 777826), the Leverhulme Trust Research through the Project Award ``Robust Learning: Uncertainty Quantification, Sensitivity and Stability’’ (grant agreement RPG-2024-051), and the EPSRC Mathematical and Foundations of Artificial Intelligence Probabilistic AI Hub (grant agreement EP/Y028783/1).

KH, BS, and MT were supported by the Fields Institute for Research in Mathematical Sciences to attend the Focus Program on Data Science, Approximation Theory, and Harmonic Analysis in June 2022 where the initial conception of this project was done. The authors thank the Fields Institute for their hospitality.

The authors thank Amy Hamm Design for the production of Figure \ref{fig:GH:Proof:GeoConstruction}.

The authors thank the anonymous referees for their suggestions for improving the manuscript.

\bibliographystyle{siamplain}
\bibliography{references}
\end{document}